\documentclass{article}
\usepackage{array}
\usepackage{amsmath,amsfonts,bbold,bm}
\usepackage{multirow}
\usepackage{amssymb}
\usepackage{amsthm}
\usepackage{tablefootnote}
\usepackage{physics}
\usepackage[numbers]{natbib}
\usepackage{authblk}

\usepackage{subfigure}

\usepackage[english]{babel}
\usepackage[table]{xcolor}    
\usepackage{booktabs}
\usepackage{nicematrix}
\usepackage{amsthm}

\usepackage[letterpaper,top=2cm,bottom=2cm,left=3cm,right=3cm,marginparwidth=1.75cm]{geometry}

\usepackage[colorlinks=true, allcolors=blue]{hyperref}
\usepackage{bbm}
\hypersetup{citecolor=teal, linkcolor=blue}

\usepackage{graphicx}
\usepackage[stable]{footmisc}
\usepackage{setspace}
\usepackage[toc,page]{appendix}

\usepackage{tcolorbox}
\usepackage{pifont}
\definecolor{mydarkgreen}{RGB}{39,130,67}
\definecolor{mydarkred}{RGB}{192,47,25}
\newcommand{\green}{\color{mydarkgreen}}
\newcommand{\red}{\color{mydarkred}}
\newcommand{\cmark}{{\green\ding{51}}}
\newcommand{\xmark}{{\red\ding{55}}}

\usepackage{titlesec}
\titleformat{\paragraph}[runin]
{\normalfont\bfseries}{}{1em}{}[:]
\titlespacing*{\paragraph}{0pt}{6pt}{3pt}

\usepackage{tablefootnote}
\usepackage[capitalize,noabbrev]{cleveref}

\theoremstyle{plain}
\newtheorem{theorem}{Theorem}[section]
\newtheorem{proposition}[theorem]{Proposition}
\newtheorem{lemma}[theorem]{Lemma}
\newtheorem{corollary}[theorem]{Corollary}
\theoremstyle{definition}

\theoremstyle{remark}
\newtheorem{remark}[theorem]{Remark}

\RequirePackage{algorithm}
\RequirePackage{algorithmic}

\title{Deep Stochastic Mechanics}
\author[1]{Elena Orlova}
\author[4]{Aleksei Ustimenko $^\dag$}
\author[1]{Ruoxi Jiang}
\author[2]{Peter Y. Lu}
\author[1,3]{Rebecca Willett}
\affil[1]{Department of Computer Science, The University of Chicago, Chicago, US}
\affil[2]{Department of Physics, The University of Chicago, Chicago, US}
\affil[3]{Department of Statistics, The University of Chicago, Chicago, US}
\affil[4]{ShareChat, London, UK}
\date{}                   
\begin{document}
\maketitle
\def\thefootnote{$\dag$}\footnotetext{The corresponding author: research@aleksei.uk}\def\thefootnote{\arabic{footnote}}

\begin{abstract}
This paper introduces a novel deep-learning-based approach for numerical simulation of a time-evolving Schr\"odinger equation inspired by stochastic mechanics and generative diffusion models. Unlike existing approaches, which exhibit computational complexity that scales exponentially in the problem dimension, our method allows us to adapt to the latent low-dimensional structure of the wave function by sampling from the Markovian diffusion. Depending on the latent dimension, our method may have far lower computational complexity in higher dimensions. Moreover, we propose novel equations for stochastic quantum mechanics, resulting in quadratic computational complexity with respect to the number of dimensions. Numerical simulations verify our theoretical findings and show a significant advantage of our method compared to other deep-learning-based approaches used for quantum mechanics.
\end{abstract}

\let\clearpage\relax

\section{Introduction}\label{sec1}
Mathematical models for many problems in nature appear in the form of partial differential equations (PDEs) in high dimensions. 
Given access to precise solutions of the many-electron time-dependent Schr\"odinger equation (TDSE),
a vast body of scientific problems could be addressed, including in quantum chemistry \citep{cances2003computational, nakatsuji2012discovery}, drug discovery \citep{ganesan2017molecular, heifetz2020quantum}, condensed matter physics \citep{boghosian1998quantum, liu2013bilinear}, and quantum computing \citep{grover2001Schrodinger, papageorgiou2013measures}. However, solving high-dimensional PDEs and the Schr\"odinger equation, in particular, are notoriously difficult problems in scientific computing due to the well-known curse of dimensionality: the computational complexity grows exponentially as a function of the dimensionality of the problem \citep{bellman2010dynamic}. Traditional numerical solvers have been limited to dealing with problems in rather low dimensions since they rely on a grid.

Deep learning is a promising way to avoid the curse of dimensionality \citep{poggio2017and, madala2023cnns}. However, no known deep learning approach avoids it in the context of the TDSE \citep{manzhos2020machine}. Although generic deep learning approaches have been applied to solving the TDSE
\citep{e2017deep, han2018solving, raissi2019physics, weinan2021algorithms}, this paper shows that it is possible to get performance improvements by developing an approach specific to the TDSE by incorporating quantum physical structure into the deep learning algorithm itself. 

We propose a method that relies on a stochastic interpretation of quantum mechanics \citep{NelsonOG, guerra1995introduction, nelson2005mystery} and is inspired by the success of deep diffusion models that can model complex multi-dimensional distributions effectively \citep{yang2022diffusion};
we call it \textit{Deep Stochastic Mechanics (DSM)}. Our approach is not limited to only the linear Schr\"odinger equation but can be adapted to Klein-Gordon, Dirac equations \citep{serva1988relativistic, lindgren2019quantum}, and to the non-linear Schr\"odinger equations of condensed matter physics, e.g., by using mean-field stochastic differential equations (SDEs) \citep{eriksen2020mean}, or McKean-Vlasov SDEs \citep{dos2022simulation}.

\subsection{Problem Formulation}
The Schr\"odinger equation, a governing equation in quantum mechanics, predicts the future behavior of a dynamic system for $0 \le t \le T$ and $\forall x\in \mathcal{M}$:
\begin{align}\label{eq:Schrodinger}
    i \hbar \partial_{t} \psi (x, t) &= \mathcal{H} \psi(x, t),\\
    \psi(x, 0) &= \psi_{0}(x),
\end{align}
where $\psi:\mathcal{M}\times [0, T]\rightarrow \mathbb{C}$ 
is a wave function defined over a manifold $\mathcal{M}$, and $\mathcal{H}$ is a self-adjoint operator acting on a Hilbert space of wave functions. For simplicity of future derivations, we consider a case of a spinless particle\footnote{A multi-particle case is covered by considering $d = 3 n$, where $n$ -- the number of particles.} in $\mathcal{M}=\mathbb{R}^d$
moving in a smooth potential 
$V:\mathbb{R}^d\times [0, T] \rightarrow \mathbb{R}_{+}$.
In this case, $\mathcal{H} = -\frac{\hbar^2}{2}\mathrm{Tr}(m^{-1}\nabla^2) + V,$ where $m \in \mathbb{R}^d \otimes \mathbb{R}^d$ is a mass tensor. 
The probability density of finding a particle at position $x$ is $|\psi (x, t)|^2$. A notation list is given in Appendix \ref{sec:app_notation}.

Given initial conditions in the form of samples drawn from density $\psi_0(x)$, we wish to draw samples from $|\psi (x, t)|^2$ for $t \in (0, T]$
using a neural-network-based approach that can adapt to latent low-dimensional structures in the system and sidestep the curse of dimensionality. Rather than explicitly estimating $\psi(x,t)$ and sampling from the corresponding density, we devise a strategy that directly samples from an approximation of  $|\psi(x,t)|^2$, concentrating computation in high-density regions. When regions where the density $|\psi (x, t)|^2$ lie in a latent low-dimensional space, our sampling strategy concentrates computation in that space, leading to the favorable scaling properties of our approach.

\section{Related Work}\label{sec: related_work}

Physics-Informed Neural Networks (PINNs) \citep{raissi2019physics} are general-purpose tools that are widely studied for their ability to solve PDEs and can be applied to solve \cref{eq:Schrodinger}. However, this method is prone to the same issues as classical numerical algorithms since it relies on a collection of collocation points uniformly sampled over the domain ${\cal M} \subseteq {\mathbb R}^d$. In the remainder of the paper, we refer to this as a `grid' for simplicity of exposition. Another recent paper by \citet{bruna2022neural} introduces Neural Galerkin schemes based on deep learning, which leverage active learning to generate training data samples for numerically solving real-valued PDEs. Unlike collocation-points-based methods, this approach allows theoretically adaptive data collection guided by the dynamics of the equations if we could sample from the wave function effectively.

Another family of approaches including DeepWF \citep{han2019solving}, FermiNet \citep{pfau2020ferminet}, and PauliNet \citep{hermann2020deep} reformulates the problem \labelcref{eq:Schrodinger} as maximization of an energy functional that depends on the solution of the stationary Schr\"odinger equation. This approach sidesteps the curse of dimensionality but cannot be applied to the time-dependent wave function setting considered in this paper.

The only thing that one can experimentally obtain is samples from the quantum mechanics density. So, it makes sense to focus on obtaining samples from the density rather than attempting to solve the Schr\"odinger equation; these samples can be used to predict the system's behavior without conducting real-world experiments. 
Based on this observation, there are a variety of quantum Monte Carlo (MC) methods \citep{barker1979quantum, Corney_2004, austin2012quantum}, which rely on estimating expectations of observables rather than the wave function itself, resulting in improved computational efficiency. However, these methods still encounter the curse of dimensionality due to recovering the full-density operator. The density operator in atomic simulations is concentrated on a lower dimensional manifold of such operators \citep{eriksen2020mean}, suggesting that methods that adapt to this manifold can be more effective than high-dimensional grid-based methods. Deep learning has the ability to adapt to this structure. Numerous works explore the time-dependent Variational Monte Carlo (t-VMC) schemes \citep{carleo2017unitary, carleo2017solving, schmitt2020quantum, yao2021adaptive} for simulating many-body quantum systems. Their applicability is often tailored to a specific problem setting as these methods require significant prior knowledge to choose a good variational ansatz function. As highlighted by \citet{sinibaldi2023unbiasing}, t-VMC methods may encounter challenges related to systematic statistical bias or exponential sample complexity, particularly when the wave function contains zeros.
 
As noted in \citet{schlick2010molecular}, knowledge of the density is unnecessary for sampling. We need a score function $\nabla \log \rho$ to be able to sample from it. The fast-growing field of generative modeling with diffusion processes demonstrates that for high-dimensional densities with low-dimensional manifold structure, it is incomparably more effective to learn a score function than the density itself \citep{ho2020denoising, yang2022diffusion}. 

For high-dimensional real-valued PDEs, there exist a variety of classic and deep learning-based approaches that rely on sampling from diffusion processes, e.g., \citet{cliffe2011multilevel, warin2018nesting, han2018solving, weinan2021algorithms}. Those works rely on the Feynman-Kac formula \citep{del2004feynman} to obtain an estimator for the solution to the PDE. However, for the Schr\"odinger equation, we need an analytical continuation of the Feynman-Kac formula on an imaginary time axis \citep{yan1994feynman} as it is a complex-valued equation. This requirement limits the applicability of this approach to our setting. BSDE methods studied by \citet{nusken2021solving, nusken2021interpolating} are closely related to our approach, but they are developed for the elliptic version of the Hamilton–Jacobi–Bellman (HJB) equation. We consider the hyperbolic HJB setting, for which the existing method cannot be applied. 

\section{Contributions}
We are inspired by works of \citet{NelsonOG, nelson2005mystery}, 
who has developed a stochastic interpretation of quantum mechanics, so-called stochastic mechanics, based on a Markovian diffusion. Instead of solving the Schr\"odinger equation\labelcref{eq:Schrodinger}, our method aims to learn the stochastic mechanical process's osmotic and current velocities equivalent to classical quantum mechanics. Our formulation differs from the original one \citep{NelsonOG, guerra1995introduction, nelson2005mystery}, as we derive equivalent differential equations describing the velocities that do not require the computation of the Laplacian operator. 
Another difference is that our formulation interpolates anywhere between stochastic mechanics and deterministic Pilot-wave theory \citep{PhysRev.85.166}. More details are given in \cref{sec:app_interpolation}.

We highlight the main contributions of this work as follows:

\begin{itemize}

    \item We propose to use a stochastic formulation of quantum mechanics \citep{NelsonOG, guerra1995introduction, nelson2005mystery} to create an efficient and theoretically sound computational framework for quantum mechanics simulation. We accomplish our result by using stochastic mechanics equations stemming from Nelson's formulation. In contrast to Nelson’s original expressions, which rely on second-order derivatives like the Lagrangian, our expressions rely solely on first-order derivatives -- specifically, the gradient of the divergence operator. This formulation, which is more amenable to neural network-based solvers, results in a reduction in the computational complexity of the loss evaluation from cubic to quadratic in dimension.
    \item We prove theoretically in \cref{sec:theory_main_part} that the proposed loss function upper bounds the $L_2$ distance between the approximate process and the `true' process that samples from the quantum density, which implies that if loss converges to zero, then the approximate process strongly converges to the `true' process. Our theoretical finding offers a simple mechanism for guaranteeing the accuracy of our predicted solution, even in settings in which no baseline methods are computationally tractable.
    \item We empirically estimate the performance of our method in various settings. Our approach shows a superior advantage to PINNs and t-VMC in terms of accuracy. We also conduct an experiment for non-interacting bosons where our method reveals linear convergence time in the dimension, operating easily in a higher-dimensional setting. Another interacting bosons experiment highlights the favorable scaling properties of our approach in terms of memory and computing time compared to a grid-based numerical solver. While our theoretical analysis establishes an $\mathcal{O}(d^2)$ bound on the algorithmic complexity, we observe an empirical scaling closer to $\mathcal{O}(d)$ for the memory and compute requirements as the problem dimension $d$ increases due to parallelization in modern machine learning frameworks.

\end{itemize}

\cref{table:comparison} compares properties of methods for solving \cref{eq:Schrodinger}. For numerical solvers, the number of grid points scales as $\mathcal{O}(N^{\frac{d}{2}+1})$ as $N$ is the number of discretization points in time, and $\sqrt{N}$ is the number of discretization points in each spatial dimension. We assume a numerical solver aims for a precision $\varepsilon = \mathcal{O}(\frac{1}{\sqrt{N}})$. In the context of neural networks, the iteration complexity is dominated by loss evaluation. For PINNs, $N_f$ denotes the number of collocation points used to enforce physics-informed constraints in the spatio-temporal domain for $d=1$. The original PINN formulation faces an exponential growth in the number of collocation points with respect to the problem dimension, $\mathcal{O}(N_f^d)$, posing a significant challenge in higher dimensions. Subsampling $\mathcal{O}(d)$ collocation points in a non-adaptive way leads to poor performance for high-dimensional problems.

For both t-VMC and FermiNet, $H_d$ denotes the number of MC iterations required to draw a single sample. The t-VMC approach requires calculating a matrix inverse, which generally exhibits a cubic computational complexity of $\mathcal{O}(d^3)$ and may suffer from numerical instabilities. Similarly, the FermiNet method, which is used for solving the time-independent Schr\"odinger equation to find ground states, necessitates estimating matrix determinants, an operation that also scales as $\mathcal{O}(d^3)$. We note that for our  DSM approach, $N$ is independent of $d$. We focus on lower bounds on iteration complexity and known bounds for the convergence of non-convex stochastic gradient descent \citep{fehrman2019convergence} that scales polynomial with $\varepsilon^{-1}$.

\begin{table*}[ht!]
    \centering
    \small
    \caption{Comparison of different approaches for simulating quantum mechanics. }\label{table:comparison} 
    \vskip 0.15in
    \scalebox{0.95}{
    \begin{tabular}{ c c c c c c c c}
    \hline
    Method & Domain & \begin{tabular}{@{}c@{}} Time \\ Evolving \end{tabular} & Adaptive  & \begin{tabular}{@{}c@{}} Iteration \\ complexity \end{tabular}   & \begin{tabular}{@{}c@{}} Overall \\ complexity \end{tabular} \\
    \hline
    PINN \citep{raissi2019physics} & Compact & \cmark & \xmark  & $\mathcal{O}(N_{f}^d)$ & $\geq \mathcal{O}(N_{f}^d \mathrm{poly}(\varepsilon^{-1}))$ 
    \\    [0.5ex]
    FermiNet \citep{pfau2020ferminet}
    & $\mathbb{R}^d$ & \xmark & \cmark & $\mathcal{O}(H_d d^3)$  & $\geq \mathcal{O}(H_{d} d^3 \mathrm{poly}(\varepsilon^{-1}))$  
    \\ [0.5ex]
    t-VMC & $\mathbb{R}^d$ & \cmark & \cmark  & $\mathcal{O}(H_d d^3)$ & $\ge \mathcal{O}(H_d d^3 \mathrm{poly}(\varepsilon^{-1}))$ 
    \\    [0.5ex]
    Num. solver & Compact & \cmark & \xmark  & N/A & $\displaystyle \mathcal{O}(d\varepsilon^{-d-2})$ 
    \\ [1ex] 
    \textbf{DSM (Ours)} & $\mathbb{R}^d$ & \cmark & \cmark & $\mathcal{O}(Nd^2)$ & $\geq \mathcal{O}(N d^2\mathrm{poly}(\varepsilon^{-1}))$  
    \\
    \hline
    \end{tabular}
    }
    \vskip -0.1in
\end{table*}

\section{Deep Stochastic Mechanics}\label{sec:SDE_mechanics}

here is a family of diffusion processes that are equivalent to \cref{eq:Schrodinger} in a sense that all time-marginals of any such process coincide with $|\psi (x, t)|^2$; we refer to Appendix \ref{sec:stochastic_processes} for derivation. Assuming $\psi (x, t) = \sqrt{\rho(x, t)}e^{iS(x, t)}$, we define:
\begin{equation}\label{eq:u_and_v}
    \begin{aligned}
        v(x,t) &= \frac{\hbar}{m}\nabla S(x, t),\\
        u(x,t) &= \frac{\hbar}{2m}  \nabla \log \rho(x, t).
    \end{aligned}
\end{equation}
Our method relies on the following stochastic process with $\nu \ge 0$ \footnote{$\nu=0$ is allowed if and only if $\psi_{0}$ is sufficiently regular, e.g., $|\psi_{0}|^2 > 0$ everywhere.}, which corresponds to sampling from
$\rho = \big|\psi (x, t)\big|^2$ \citep{NelsonOG}:
\begin{equation}\label{eq:main_sde}
\begin{aligned}
    \mathrm{d}Y(t) &= (v(Y(t), t)+\nu u(Y(t), t))\mathrm{d}t + \sqrt{\frac{\nu \hbar}{m} }\mathrm{d}\overset{\rightarrow}{W},\\
    Y(0) &\sim \big|\psi_{0}\big|^2,
\end{aligned}
\end{equation}
where $u$ is an osmotic velocity, $v$ is a current velocity and $\overset{\rightarrow}{W}$ is a standard (forward) Wiener process. Process $Y(t)$ is called the Nelsonian process. Since we don't know the true $u, v$, we instead aim at approximating them with the process defined using neural network approximations $v_{\theta}, u_{\theta}$:
\begin{equation}\label{eq:approx_sde}
\begin{aligned}
    \mathrm{d}X(t) &= (v_{\theta}(X(t), t)+\nu u_{\theta}(X(t), t))\mathrm{d}t + \sqrt{\frac{\nu \hbar}{m} }\mathrm{d}\overset{\rightarrow}{W}, \\
    X(0) &\sim \big|\psi_{0}\big|^2.
\end{aligned}
\end{equation}
Any numerical integrator can be used to obtain samples from the diffusion process. The simplest one is the Euler--Maruyama integrator \citep{kloeden1992stochastic}:
\begin{align}\label{eq:main_sde_integrator}
    X_{i+1} &= X_{i} + (v_{\theta}(X_{i}, t_{i})+\nu u_{\theta}(X_{i}, t_{i}))\epsilon + \mathcal{N}\big(0, \frac{\nu \hbar}{m}  \epsilon I_{d}\big),
\end{align}
where $\epsilon > 0$ denotes a step size, $0 \le i < \frac{T}{\epsilon}$, 
and $\mathcal{N}(0, I_{d})$ is a Gaussian distribution. We consider this integrator in our work. Switching to higher-order integrators, e.g., the Runge-Kutta family of integrators \citep{kloeden1992stochastic}, can potentially enhance efficiency and stability when $\epsilon$ is larger.

The diffusion process from \cref{eq:main_sde} achieves sampling from $\rho = \big|\psi (x, t)\big|^2$ for each $t\in[0, T]$ for known $u$ and $v$. Assume that $\psi_{0}(x) = \sqrt{\rho_{0}(x)}e^{i S_{0}(x)}$. Our approach relies on the following equations for the velocities:
\begin{subequations}
\label{eq:nav}
\begin{align}\label{eq:nav-s1}
    \partial_{t} v &= -\frac{1}{m} \nabla V + \langle u, \nabla\rangle u - \langle v, \nabla\rangle v + \frac{\hbar}{2m} \nabla\langle \nabla, u\rangle,
\end{align}
\vskip -0.25in
\begin{align}\label{eq:nav-s2}
    \partial_{t} u &=  - \nabla \langle v, u\rangle - \frac{\hbar}{2m} \nabla \langle \nabla, v\rangle,\
\end{align}
\vskip -0.25in
\begin{align}\label{eq:nav-s3}
    v_{0}(x) &= \frac{\hbar}{m}\nabla S_0(x), ~ u_{0}(x) = \frac{\hbar}{2m}  \nabla \log \rho_0(x).
\end{align}
\end{subequations}
These equations are derived in Appendix \ref{app:stochastic_mechanics} and are equivalent to the Schr\"odinger equation.
As mentioned, our equations differ from the canonical ones developed in \citet{NelsonOG, guerra1995introduction}. In particular, the original formulation from \cref{eq:nelson_eq}, which we call the \textit{Nelsonian version}, includes the Laplacian of $u$; in contrast, \textit{our version} in \labelcref{eq:nav-s1} uses the gradient of the divergence operator. These versions are equivalent in our setting, but our version has significant computational advantages, as we describe later in \Cref{prop:alg}.

\subsection{Learning Drifts}
This section describes how we learn the velocities $u_{\theta}(X, t)$ and $v_{\theta}(X, t)$,  parameterized by neural networks with parameters $\theta$. We propose to use a combination of three losses: two of them come from the Navier-Stokes-like equations \labelcref{eq:nav-s1}, \labelcref{eq:nav-s2}, and the third one enforces the initial conditions \labelcref{eq:nav-s3}. We define non-linear differential operators that appear in \cref{eq:nav-s1}, \labelcref{eq:nav-s2}: 

\begin{equation} \label{eq:diff_operators}
\begin{aligned}
    \mathcal{D}_{u}[v, u, x, t] =  - \nabla \langle v(x, t), u(x, t)\rangle - \frac{\hbar}{2m} \nabla \langle\nabla, v(x, t)\rangle,
\end{aligned}
\end{equation}

\begin{equation} \label{eq:diff_operators2}
\begin{aligned}
    \mathcal{D}_{v}[v, u, x, t] =  - \frac{1}{m} \nabla V (x, t) + \frac{1}{2}\nabla \| u(x, t)\|^2 -\frac{1}{2}\nabla \| v(x, t)\|^2 + \frac{\hbar}{2m} \nabla\langle \nabla, u(x, t)\rangle.
\end{aligned}
\end{equation}

We aim to minimize the following losses:
\begin{equation} \label{eq:L_u}
\begin{aligned}
    L_1 (v_{\theta}, u_{\theta}) = \int_{0}^T\mathbb{E}^{X}\big\| \partial_{t} {  u_\theta}(X(t), t) -\mathcal{D}_{u}[v_{\theta}, u_{\theta}, X(t), t]\big\|^2\mathrm{d}t,
\end{aligned}
\end{equation}
\begin{equation}\label{eq:L_v}
\begin{aligned}
    L_2 (v_{\theta}, u_{\theta}) =   \int_{0}^T\mathbb{E}^{X} \big\| \partial_{t} v_{\theta} (X(t),t) -  \mathcal{D}_{v}[v_{\theta}, u_{\theta}, X(t), t] \big\|^2\mathrm{d}t ,
    \end{aligned}
\end{equation}

\begin{equation}\label{eq:L_0}
\begin{aligned}
    L_{3}(v_{\theta}, u_{\theta}) = \mathbb{E}^{X} \|u_{\theta}(X(0), 0) - u_{0}(X(0))\|^2, 
\end{aligned}
\end{equation}

\begin{equation}\label{eq:L_0_v}
\begin{aligned}
    L_{4}(v_{\theta}, u_{\theta}) = \mathbb{E}^{X} \|v_{\theta}(X(0), 0) - v_{0}(X(0))\|^2,
\end{aligned}
\end{equation}

where $u_0,v_0$ are defined in \cref{eq:nav-s3}.
Finally, we define a combined loss using a weighted sum with $w_{i} > 0$:
\begin{equation}\label{eq:L_full}
    \mathcal{L}(\theta) = \sum_{i=1}^{4} w_i L_{i} (v_{\theta}, u_{\theta}).
\end{equation}
The basic idea of our approach is to sample new trajectories using \cref{eq:main_sde_integrator} with $\nu=1$ for each iteration $\tau$. These trajectories are then used to compute stochastic estimates of the loss from \cref{eq:L_full}, and then we back-propagate gradients of the loss to update $\theta$.  
We re-use recently generated trajectories to reduce computational overhead as SDE integration cannot be paralleled. The training procedure is summarized in Algorithm \ref{alg:train} and \cref{fig:dsm_overview}; a more detailed version is given in \cref{sec:app_algo}.

\begin{algorithm}
\small
\caption{Training algorithm pseudocode}\label{alg:train}
\begin{algorithmic}
\vspace{2pt}
        \vspace{2pt}
\STATE {\bfseries Input} $\psi_{0}$ -- initial wave-function, $M$ -- epoch number, $B$ -- batch size, other parameters (optimizer parameters, physical constants, Euler--Maruyama parameters; see \cref{sec:app_algo})  
\STATE Initialize NNs $u_{\theta_0}$, $v_{\theta_0}$
\FOR{each iteration $0 \le \tau < $ $M$} 
        \STATE Sample $B$ trajectories 
        using $u_{\theta_\tau}, v_{\theta_\tau}$ via \cref{eq:main_sde_integrator} with $\nu=1$ 
        \STATE Estimate loss $\mathcal{L}(v_{\theta_{\tau}}, u_{\theta_{\tau}})$ from \cref{eq:L_full} over the sampled trajectories
        \vspace{2pt}
        \STATE Back-propagate gradients to get $\nabla_{\theta}\mathcal{L}(v_{\theta_{\tau}}, u_{\theta_{\tau}})$
        \STATE An optimizer step to get $\theta_{\tau+1}$        
\ENDFOR
\STATE {\bfseries output } $u_{\theta_{M}}, v_{\theta_{M}}$ 
\end{algorithmic}
\end{algorithm}

\begin{figure*}[t!]
\vskip 0.2in
    \centering
    \includegraphics[width=1.0\textwidth]{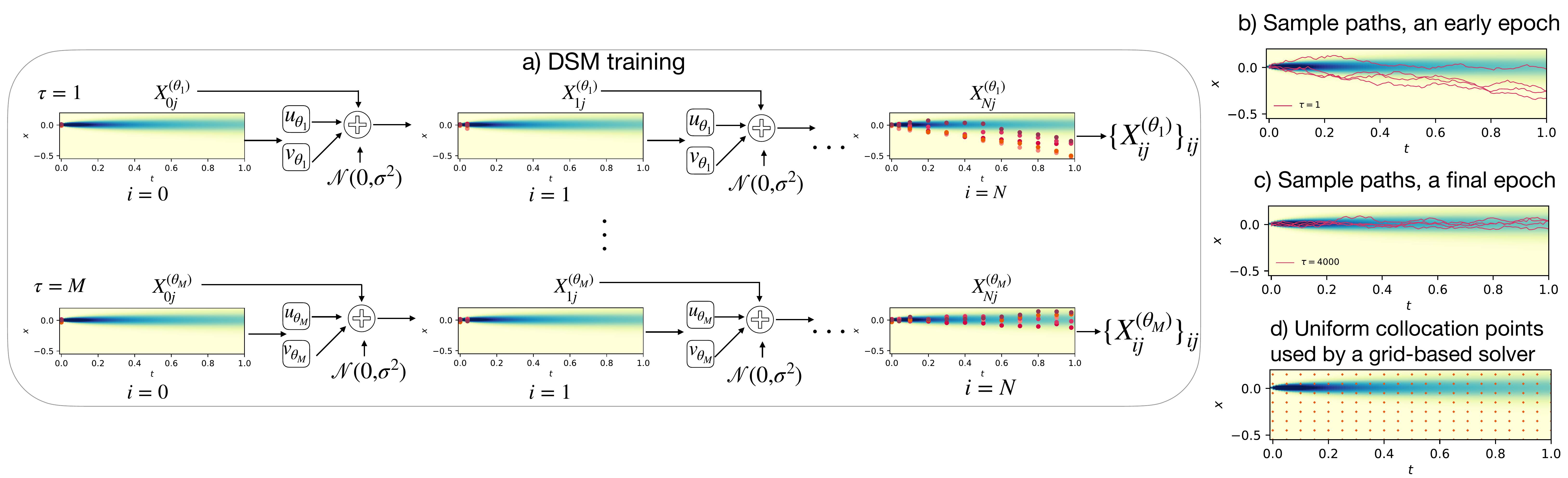} 
    \caption{An illustration of our approach. Blue regions in the plots correspond to higher-density regions. (a) DSM training scheme: at every epoch $\tau$, we generate $B$ full trajectories $\{ X_{ij}\}_{ij}$, $i=0, ..., N$, $j=1, ..., B$. Then, we update the weights of our NNs.  (b) An illustration of sampled trajectories at the early epoch. (c) An illustration of sampled trajectories at the final epoch.  (d) Collocation points for a grid-based solver where it should predict values of $\psi(x, t)$.  
    }
    \label{fig:dsm_overview}
    \vskip -0.2in
\end{figure*}

We use trained  $u_{\theta_{M}}, v_{\theta_{M}}$ to simulate the forward diffusion for $\nu \geq 0$ given $X_0 \sim \mathcal{N}(0, I_{d})$: 
\begin{equation}
    X_{i+1} = X_{i} + (v_{\theta_{M}}(X_{i}, t_{i})+\nu u_{\theta_{M}}(X_{i}, t_{i}))\epsilon + \mathcal{N}\big(0, \frac{\hbar}{m}\nu \epsilon I_{d}\big). 
\end{equation}

Appendix \ref{sec:applications_full} describes a wide variety of possible ways to apply our approach for estimating an arbitrary quantum observable, singular initial conditions like $\psi_{0} = \delta_{x_{0}}$, singular potentials, correct estimations of observable that involve measurement process, and recovering the wave function from the velocities $u,v$.

Although PINNs can be used to solve Equations \labelcref{eq:nav-s1}, \labelcref{eq:nav-s2}, that approach would suffer from having fixed sampled density (see \cref{sec:experiments}). Our method, much like PINNs, seeks to minimize the residuals of the PDEs from Equations (\ref{eq:nav-s1}) and (\ref{eq:nav-s2}). However,
we do so on the distribution generated by sampled trajectories $X(t)$, which in turn depends on current neural approximations $v_{\theta}, u_{\theta}$. This allows our method to focus only on high-density regions and alleviates the inherent curse of dimensionality that comes from reliance on a grid.

\subsection{Algorithmic Complexity}
Our formulation of stochastic mechanics with novel Equations \labelcref{eq:nav} is much more amenable to automatic differentiation tools than if we developed a neural diffusion approach based on the Nelsonian version. In particular, the original formulation uses the Laplacian operator $\Delta u$ that naively requires $\mathcal{O}(d^3)$ operations, which might become a major bottleneck for scaling them to many-particle systems. While a stochastic trace estimator \citep{hutchinson} may seem an option to reduce the computational complexity of Laplacian calculation to $\mathcal{O}(d^2)$, it introduces a noise of an amplitude $\mathcal{O}(\sqrt{d})$. Consequently, a larger batch size (as $\mathcal{O}(d)$) is necessary to offset this noise resulting in still a cubic complexity.

\begin{remark}\label{prop:alg}
    The algorithmic complexity w.r.t. $d$ of computing differential operators from Equations (\ref{eq:diff_operators}), (\ref{eq:diff_operators2}) for velocities $u, v$ is $\mathcal{O}(d^2)$. \footnote{Estimation of a term $\nabla V (x, t)$ might have different computational complexity from $\mathcal{O}(d)$, $\mathcal{O}(d^2)$, or even higher depending on a particle interaction type.}
\end{remark}

This remark is proved in Appendix \ref{sec:app_complexity}. This trick with the gradient of divergence can be used as we rely on the fact that the velocities $u, v$ are full gradients, which is not the case for the wave function $\psi(x, t)$ itself.

{We expect that one of the factors of $d$ associated with evaluating a $d$-dimensional function gets parallelized over in modern machine learning frameworks, so we can see a linear scaling even though we are using an $\mathcal{O}(d^2)$ method.  We will see such behavior in our experiments.

\subsection{Theoretical Guarantees}\label{sec:theory_main_part}
To further justify the effectiveness of our loss function, we prove the following theorem in Appendix \ref{app:theory}:
\begin{theorem} (Strong Convergence Bound) We have the following bound between processes $Y$ (the Nelsonian process that samples from $|\psi|^2$) and $X$ (the neural approximation with $v_{\theta}, u_{\theta}$):
\begin{align}
\sup_{t\le T} \mathbb{E}\|X(t) - Y(t)\|^2 \le C_{T} \mathcal{L}(v_{\theta}, u_{\theta}),
\end{align}
where constant $C_{T}$ is defined explicitly in \ref{thm:main_theorem}. 
\end{theorem}

This theorem means optimizing the loss leads to a strong convergence of the neural process $X$ to the Nelsonian process $Y$, and that the loss value directly translates into an improvement of $L_{2}$ error between the processes. The constant $C$ depends on a horizon $T$ and Lipshitz constants of $u, v, u_{\theta}, v_{\theta}$. It also hints that we have a `low-dimensional' structure when Lipshitz constants of $u, v, u_{\theta}, v_{\theta}$ are $\ll d$, which is the case of low-energy regimes (as large Lipshitz smoothness constant implies large value of the Laplacian and, hence, energy) and with the proper selection of a neural architecture \citep{aziznejad2020deep}.

\section{Experiments}\label{sec:experiments}
\paragraph{Experimental setup} As a baseline, we use an analytical or numerical solution. We compare our method's (DSM) performance with PINNs and t-VMC. In the case of non-interacting particles, the models are feed-forward neural networks with one hidden layer and a hyperbolic tangent ($\tanh$) activation function. We use a similar architecture with residual connection blocks and a $\tanh$ activation function when studying interacting particles. Further details on numerical solvers, architecture, training procedures, hyperparameters of our approach, PINNs, and t-VMC can be found in \cref{sec:app_training}. Additional experiment results are given in  \cref{sec:app_res}. The code of our experiments can be found on GitHub \footnote{\url{https://github.com/elena-orlova/deep-stochastic-mechanics}}. We only consider bosonic systems, leaving fermionic systems for further research.

\paragraph{Evaluation metrics} 
We estimate errors between true and predicted values of the mean and the variance of a coordinate $X_i$ at time $i = 1, \dots, T$ as the relative $L_2$-norm, namely $\mathcal{E}_m(X_i)$ and $\mathcal{E}_v(X_i)$. 
The standard deviation (confidence intervals) of the observables are indicated in the results. 
True $v$ and $u$ values are estimated numerically with the finite difference method. Our trained $u_\theta$ and $v_\theta$ should output these values. We measure errors $\mathcal{E} (u)$ and  $\mathcal{E} (v)$ as the $L_2$-norm between the true and predicted values in $L_2(\mathbb{R}^d\times [0, T], \mu)$ with $\mu(\mathrm{d}x, \mathrm{d}t) = |\psi(x, t)|^2 \mathrm{d}x\mathrm{d}t$.

\subsection{Non-interacting Case: Harmonic Oscillator} \label{sec:harm_osc}
We consider a harmonic oscillator model with 
$x\in\mathbb{R}^{1}$, 
$V (x) = \frac{1}{2}m\omega^2(x - 0.1)^2$,   
$t\in [0, 1]$ and where $m=1$ and $\omega=1$. The initial wave function is given as $\psi(x, 0) \propto e^{-x^2/(4\sigma^2)}$. Then ${  u}_0(x) = -\frac{\hbar x}{2 m \sigma^2} $, 
${  v}_0(x) \equiv 0$. $X(0)$ comes from $X(0) \sim \mathcal{N}(0, \sigma^2),$ where $\sigma^2 = 0.1$.

We use the numerical solution as the ground truth. Our approach is compared with a PINN. The PINN input data consists of $N_0 = 1000$ points sampled for estimating $\psi(x, 0)$, $N_b = 300$ points for enforcing the boundary conditions (we assume zero boundary conditions), and $N_f = 60000$ collocation points to enforce the corresponding equation inside the solution domain, all points sampled uniformly for $x \in [-2, 2]$ and $t \in [0, 1]$.

\cref{fig:boson_interact_2d_both_new}(a) summarizes the results of our experiment. The left panel of the figure illustrates the evolution of the density $|\psi (x, t)|^2$ over time for different methods. It is evident that our approach accurately captures the density evolution, while the PINN model initially aligns with the ground truth but deviates from it over time. Sampling collocation points uniformly when density is concentrated in a small region explains why PINN struggles to learn the dynamics of \cref{eq:Schrodinger}; we illustrate this effect in Figure \ref{fig:dsm_overview} (d). The right panel demonstrates observables of the system, the averaged mean of $X_i$, and the averaged variance of $X_i$.  Our approach consistently follows the corresponding distribution of $X_i$. On the contrary, the predictions of the PINN model only match the distribution at the initial time steps but fail to accurately represent it as time elapses. Table \ref{table:osc_1d_comparison} shows the error rates for our method and PINNs. In particular, our method performs better in terms of all error rates than the PINN. These findings emphasize the better performance of the proposed method in capturing the dynamics of the Schr\"odinger equation compared to the PINN model.

\begin{figure}[!t]
 \vskip 0.1in
    \centering
    {\includegraphics[width=0.98\textwidth]{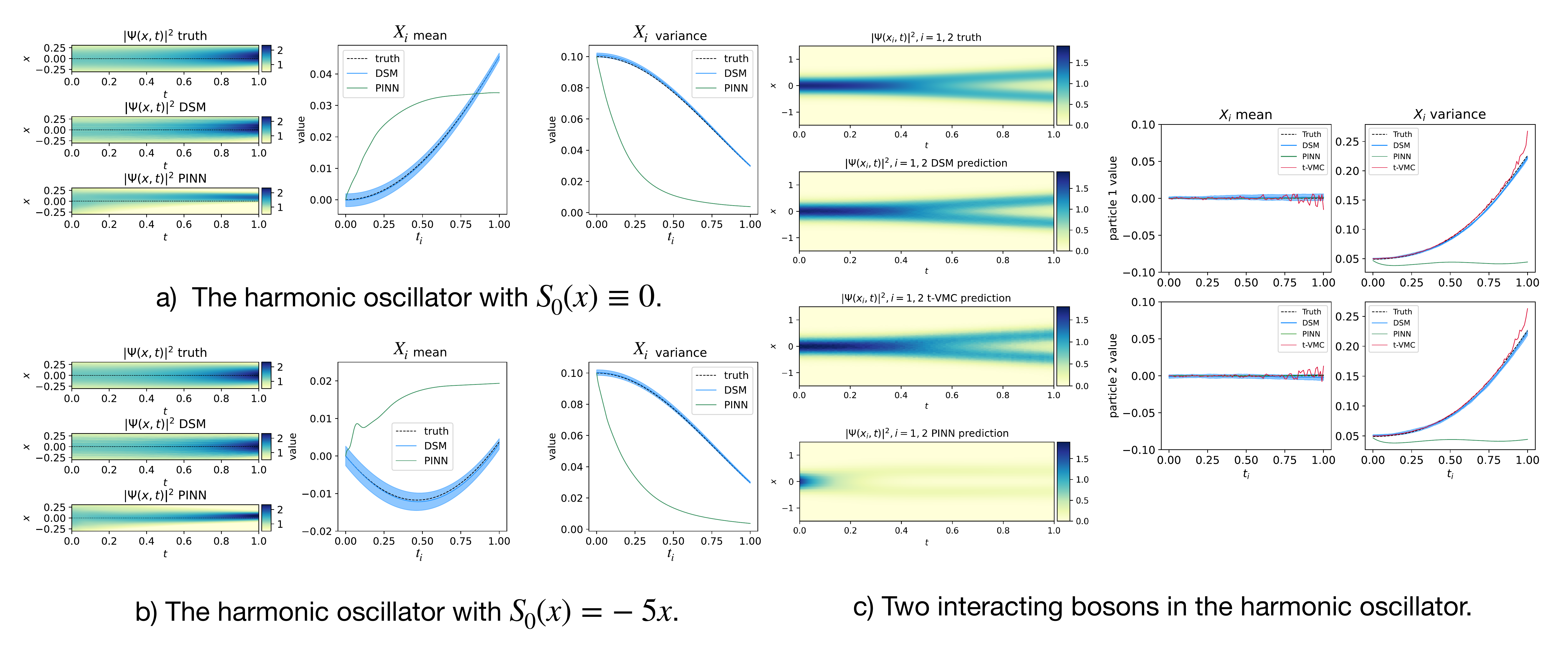}} 
    \caption{Simulation results of PINN and our DSM method: (a) and (b) correspond to a particle in the harmonic oscillator with different initial phases; (c) corresponds to two interacting bosons in the harmonic oscillator. The left panel of these figures corresponds to the density $|\psi(x,t)|^2$ of the ground truth solution, our approach (DSM), PINN, and t-VMC. The right panel presents statistics, including the particle's mean position and variance.}
    \label{fig:boson_interact_2d_both_new}
    \vskip -0.1in
\end{figure}

We also consider a non-zero initial phase $S_0(x) = -5x$. It corresponds to the initial impulse of a particle.
Then ${v}_0(x) \equiv -\frac{5\hbar}{m}$. The PINN inputs are $N_0 = 3000$, $N_b = 300$ points, and  $N_f = 80000$ collocation points.
\cref{fig:boson_interact_2d_both_new} (b) and Table \ref{table:osc_1d_comparison} present the results of our experiment. Our method consistently follows the corresponding ground truth, while the PINN model fails to do so. It indicates the ability of our method to accurately model the behavior of the quantum system. 

In addition, we consider an oscillator model with three non-interacting particles, which can be seen as a 3d system. The results are given in \cref{table:osc_1d_comparison} and \cref{sec:app_res_3d}.

\addtocounter{footnote}{-1} 
\begin{table*}[ht!]
\caption{Results for different harmonic oscillator settings. 
In the 3d setting, the reported errors are averaged across all dimensions. The \textbf{best} results are in bold. 
\footnotemark{}}\label{table:osc_1d_comparison}
\vskip 0.15in
\centering
\scalebox{0.80}{
\begin{tabular}{c c c c c c} 
 \hline
 Setting & Model & $\mathcal{E}_m(X_i)$ $\downarrow$ & $\mathcal{E}_v(X_i)$ $\downarrow$ & $\mathcal{E}(v)$ $\downarrow$ & $\mathcal{E}(u)$ $\downarrow$ \\ [0.5ex] 
 \hline
\multirow{2}{*}{\begin{tabular}{@{}c@{}} $d=1$, \\ $S_0(x)\equiv 0$ \end{tabular}} & PINN & 0.877 $\pm$ 0.263 &  0.766 $\pm$ 0.110 & 24.153 $\pm$ 3.082  & 4.432 $\pm$ 1.000 \\ [0.3ex] 
 & DSM &  $\bf 0.079 \pm 0.007$  &  $\bf 0.019 \pm 0.005$ & $\bf 1.7 \times 10^{-4} \pm 4.9 \times 10^{-5}$  &  $\bf 2.7 \times 10^{-5} \pm  4.9 \times 10^{-6}$  \\ [0.3ex] 
 & Gaussian sampling & 0.355 $\pm$ 0.038 &  0.460 $\pm$ 0.039 & 8.478  $\pm$ 4.651 & 2.431 $\pm$	0.792 \\  [1ex] 
 \hline
 \multirow{2}{*}{\begin{tabular}{@{}c@{}} $d=1$, \\ $S_0(x)=-5x$ \end{tabular}} & PINN & 2.626 $\pm$	0.250  &  0.626 $\pm$ 0.100 & 234.926 $\pm$	57.666 & 65.526 $\pm$ 8.273 \\ [0.3ex] 
 & DSM & $ \bf 0.268 \pm	0.036$   &  $ \bf 0.013 \pm	0.008$ &  $  \bf 1.4 \times 10^{-5} \pm 5.5 \times 10^{-6}$   & $ \bf 2.5 \times 10^{-5} \pm 3.8 \times 10^{-6}$   \\[0.3ex] 
 & Gaussian sampling &  0.886 $\pm$ 0.137  &  0.078 $\pm$ 0.013 &  73.588 $\pm$	6.675 & 16.298 $\pm$	6.311 \\ [1ex] 
 \hline
 \multirow{2}{*}{\begin{tabular}{@{}c@{}} $d=3$, \\ $S_0(x)\equiv 0$ \end{tabular}} & DSM (Nelsonian) & $ \bf 0.080 \pm 0.015$ & $ \bf 0.016 \pm 0.007$ & $\bf 8.1 \times 10^{-5} \pm 2.8 \times 10^{-5}$ & $\bf 4.0 \times 10^{-5} \pm 2.2 \times 10^{-5}$ \\ [0.3ex] 
 &  DSM (Grad Div) & $\bf 0.075 \pm 0.004$ & $ \bf 0.015 \pm 0.004$ &   $ \bf 6.2 \times 10^{-5} \pm 2.2 \times 10^{-5} $ &  $ \bf 3.9 \times 10^{-5} \pm 2.9 \times 10^{-5}$ \\[0.3ex] 
 & Gaussian sampling   & 0.423 $\pm$	0.090  & 4.743 $\pm$	0.337 & 6.505  $\pm$ 3.179 & 3.207 $\pm$	0.911 \\ [1ex] 
 \hline
\multirow{2}{*}{\begin{tabular}{@{}c@{}}  $d=2$, \\ interacting \\ system \end{tabular}} 
& PINN & 0.258 $\pm$ 0.079 & 1.937 $\pm$	0.654 & 20.903  $\pm$ 7.676 & 10.210 $\pm$	3.303    \\[0.3ex] 
& DSM & $\bf 0.092 \pm	0.004$& $\bf 0.055 \pm	0.015$ &  $ \bf 7.6 \times 10^{-5}\pm  1.0 \times 10^{-5}$ & $\bf 6.6\times 10^{-5} \pm 2.8 \times 10^{-5}$  \\ [0.3ex] 
& t-VMC & 0.103 $\pm$ 0.007 & 0.109 $\pm$	0.023 & $2.9 \times 10^{-3} \pm 2.4 \times 10^{-4}$   &  $3.5 \times 10^{-4} \pm 0.8 \times 10^{-4}$   \\
\hline
\end{tabular}}
\vskip -0.1in
\end{table*}

\subsection{Naive Sampling}
To further evaluate our approach, we consider the following sampling scheme: it is possible to replace all measures in the expectations from \cref{eq:L_full} with a Gaussian noise $\mathcal{N}(0,1)$. Minimizing this loss perfectly would imply that the PDE is satisfied for all values $x, t$. 
\cref{table:osc_1d_comparison} shows worse quantitative results compared to our approach in the setting from \cref{sec:harm_osc}.
More detailed results, including the singular initial condition and 3d harmonic oscillator setting, are given in Appendix \ref{app:naive_sampling}.

\subsection{Interacting System}\label{sec:interact_main}
Next, we consider a system of two interacting bosons in a harmonic trap with a soft contact term $V(x_1,x_2) = \frac{1}{2}m\omega^2(x_1^2 + x_2^2) + \frac{g}{2}\frac{1}{\sqrt{2\pi\sigma^2}}e^{-(x_{1} - x_{2})^2/(2\sigma^2)}$ and initial condition $\psi_{0} \propto e^{-m\omega^2x^2/(2\hbar)}$. We use $\omega=1$, $T=1$, $\sigma^2=0.1$, and $N=1000$. The term $g$ controls interaction strength. When $g = 0$, there is no interaction, and $\psi_{0}$ is the ground state of the corresponding Hamiltonian $\mathcal{H}$. We use $g = 1$ in our simulations. 

\cref{fig:boson_interact_2d_both_new} (c) shows simulation results: our method follows the corresponding ground truth while PINN fails over time. As $t$ increases, the variance of $X_i$ for PINN either decreases or remains relatively constant, contrasting with the dynamics that exhibit more divergent behavior. We hypothesize that such discrepancy in the performance of PINN, particularly in matching statistics, is due to the design choice. Specifically, the output predictions, $\psi(x_i,t)$, made by PINNs are not constrained to adhere to physical meaningfulness, meaning $\int_{\mathbb{R}^d} \big|\psi(x, t)\big|^2\mathrm{d}x$ does not always equal 1, making uncontrolled statistics. 

As for the t-VMC baseline, the results are a good qualitative approximation to the ground truth. The t-VMC ansatz representation comprises Hermite polynomials with two-body interaction terms \citep{carleo2017unitary}, scaling quadratically with the number of basis functions. This representation inherently incorporates knowledge about the ground truth solution. However, even when using the same number of samples and time steps as our DSM approach, t-VMC does not achieve the same level of accuracy, and the t-VMC approach does not perform well beyond $d=3$ (see Appendix \ref{app:scaling_bosons_interact}). We anticipate the performance of t-VMC will further deteriorate for larger systems due to the absence of higher-order interactions in the chosen ansatz. We opted for this polynomial representation for scalability and because our experiments with neural network ansatzes \citep{schmitt2020quantum} did not yield satisfactory results for any $d$. Additional details are provided in Appendix \ref{sec:app_inter_nn_details}.

\addtocounter{footnote}{-1}
\stepcounter{footnote}\footnotetext{The difference between the mean errors of the DSM approach and other methods is statistically significant with a p-value
$< 0.001$ measured by the one-sided Welsh t-test.
Each model is trained and evaluated 10 times independently.
}

\subsubsection{DSM in Higher Dimensions}

To verify that our method can yield reasonable outputs for large many-body systems, we perform experiments on a 100 particle version of the interacting boson system. While ground truth is unavailable for a system of such a large scale, we perform a partial validation of our results by analyzing how the estimated densities change at $x=0$ as a function of the interaction strength $g$. Scaling our method to many particles is straightforward, as we only need to adjust the neural network input size and possibly other parameters, such as a hidden dimension size. The obtained results in \cref{fig:100_bosons_stat} suggest that the time evolution is at least qualitatively reasonable since the one-particle density decays more quickly with increasing interaction strength $g$. In particular, this value should be higher for overlapping particles (a stable system with a low $g$ value) and lower for moving apart particles (a system with a stronger interaction $g$). Furthermore, the low training loss of $10^{-2}$ order achieved by our model suggests that it is indeed representing a process consistent with Schr\"odinger equation, even for these large-scale systems. This experiment demonstrates our ability to scale the DSM approach to large interacting systems easily while providing partial validation of the results through the qualitative analysis of the one-particle density and its dependence on the interaction strength.

\begin{figure}[ht]
\centering
\vskip 0.1in
\includegraphics[width=0.4\textwidth]{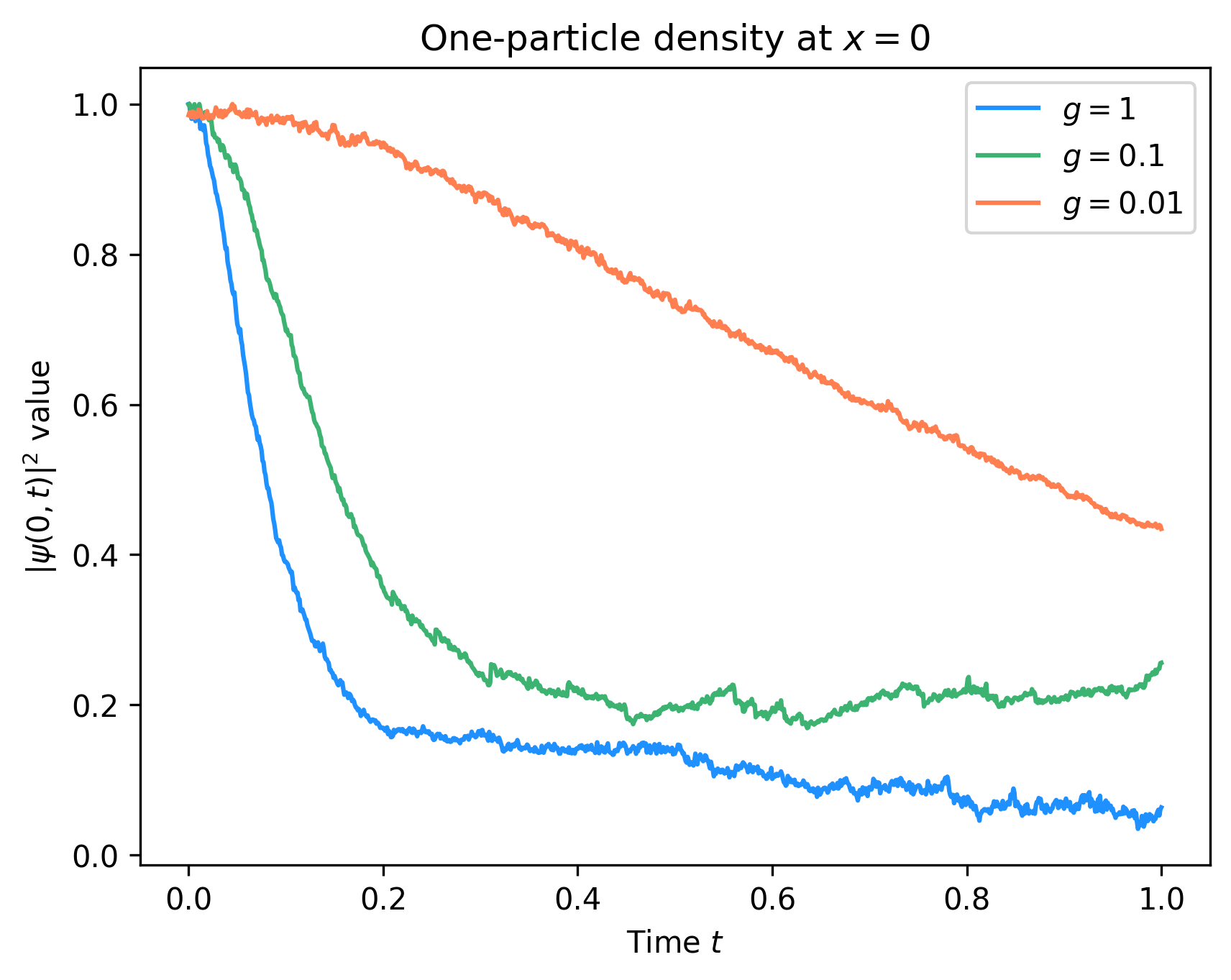}
    \caption{The one-particle density of a system of $100$ interacting bosons for varying interaction strength $g$. 
    For a weaker interaction, the one-particle density is higher, indicating a more stable particle configuration. Conversely, for a stronger interaction, this value decreases, suggesting a more dispersed particle behavior. 
    }
    \label{fig:100_bosons_stat}
    \vskip -0.1in
\end{figure}


\subsection{Computational and Memory Complexity}\label{sec:res_complexity}

\subsubsection{Non-Interacting System} \label{sec:res_complexity_noninter}
We measure training time per epoch and total train time for two versions of the DSM algorithm for $d = 1, 3, 5, 7, 9$: the Nelsonian one and our version. 
The experiments are conducted using the harmonic oscillator model with $S_0(x)\equiv 0$ from \cref{sec:harm_osc}. The results are averaged across 30 runs.
In this setting, the Hamiltonian is separable in the dimensions, and the problem has a linear scaling in $d$. However, given no prior knowledge about that, traditional numerical solvers and PINNs would suffer from exponential growth in data when tackling this task. 
Our method does not rely on a grid in $x$, and avoids computing the Laplacian in the loss function. That establishes lower bounds on the computational complexity of our method, and this bound is sharp for this particular problem. The advantageous behavior of our method is observed without any reliance on prior knowledge about the problem's nature.

\paragraph{Time per epoch} The left panel of Figure \ref{fig:complexity} illustrates the scaling of time per iteration for both the Nelsonian formulation and our proposed approach. The time complexity exhibits a quadratic scaling trend for the Nelsonian version, while our method achieves a more favorable linear scaling behavior with respect to the problem dimension. These empirical observations substantiate our analytical complexity analysis.  

\paragraph{Total training time} The right panel of Figure \ref{fig:complexity} demonstrates the total training time of our version versus the problem dimension. We train our models until the training loss reaches a threshold of $2.5 \times 10^{-5}$. We observe that the total training time exhibits a linear scaling trend as the dimensionality $d$ increases.  The performance errors are presented in Appendix \ref{app:scaling_noninteract}.

\begin{figure}[ht]
\centering
\includegraphics[width=0.70\textwidth]{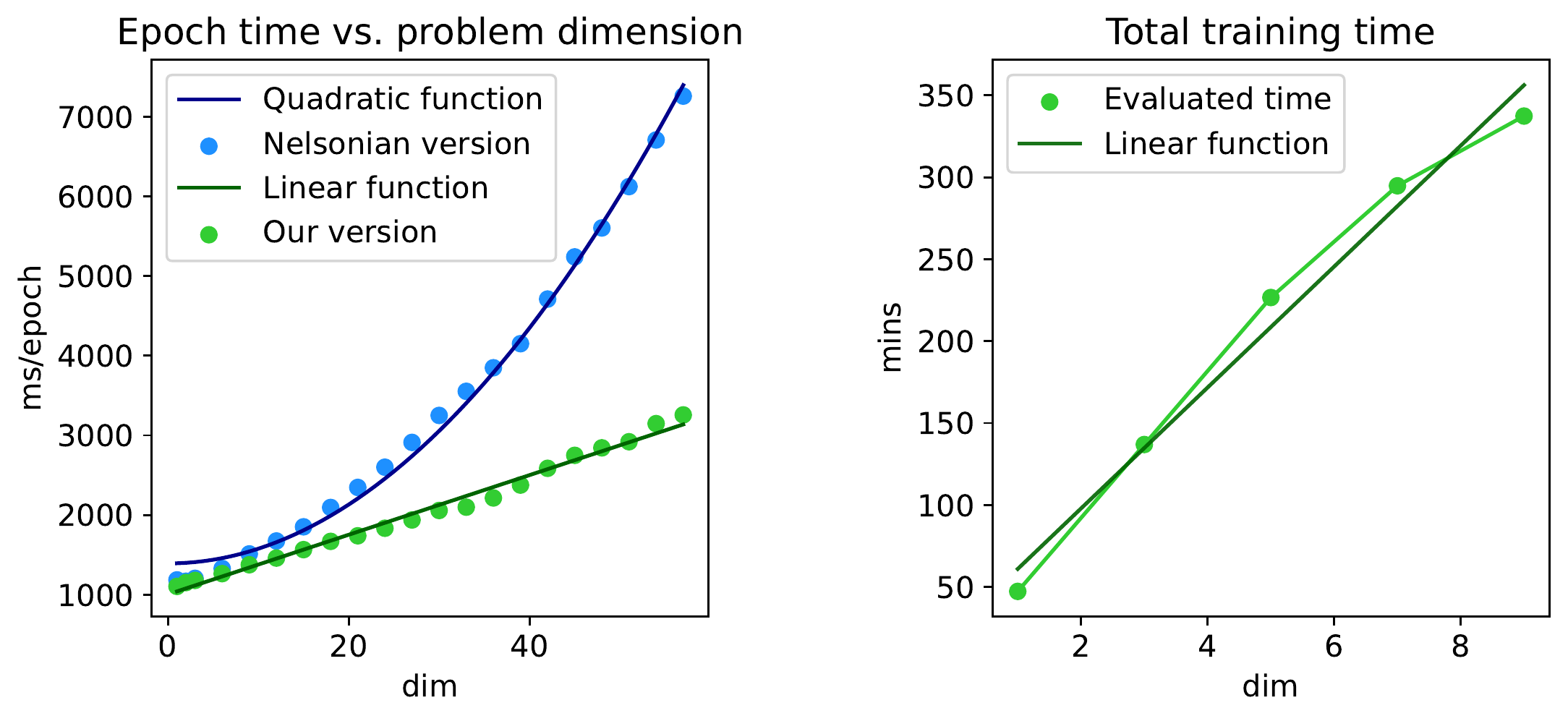}
    \caption{Empirical complexity evaluation of our method for the non-interacting system.}
    \label{fig:complexity}
\end{figure}

\subsubsection{Interacting System}\label{sec:scaling_bosons_interact}

We study the scaling capabilities of our DSM approach in the setting from \cref{sec:interact_main}, comparing the performance of our algorithm with a numerical solver based on the Crank–Nicolson method. \cref{table:num_solver_scaling} reports time and memory usage of the numerical solver. \cref{table:dsm_scaling} shows training time, time per epoch, and memory usage for our method. More details and illustrations of obtained solutions are given in \cref{app:scaling_bosons_interact}.

\paragraph{Memory} DSM memory usage and time per epoch grow linearly in $d$ (according to our theory and evident in our numerical results) in contrast to the Crank-Nikolson solver, whose memory usage grows exponentially since discretization matrices are of  $N^d \times N^d$ size. As a consequence, we are unable to execute the Crank-Nicolson method for $d > 4$ on our computational system due to memory constraints. The results show that our method is far more memory efficient for larger $d$. 

\paragraph{Compute time} While the total compute times of our DSM method, including training, are longer than those of the Crank-Nicolson solver for smaller values of $d$, the scaling trends suggest a computational advantage as $d$ increases. In general, DSM is expected to scale quadratically with the problem dimension as there are pairwise interactions in our potential function.

\begin{table}[ht!]
\caption{Time (s) to get a solution and memory usage (Gb) of the Crank-Nicolson method for different problem dimensions (interacting bosons).} \label{table:num_solver_scaling}
\vskip 0.1in
\centering
\scalebox{0.95}{
\begin{tabular}{c c c c } 
 \hline
 ~ & $d=2$ & $d=3$ & $d=4$ \\ [0.5ex] 
 \hline
Time &  0.75 & 35.61 & 2363 \\ 
Memory usage &   7.4 & 10.6  & 214 \\[1ex] 
 \hline
\end{tabular}
}
\end{table}

\begin{table}[ht!]
\caption{Training time (s), time per epoch (s/epoch), and memory usage (Gb) of our method for different problem dimensions (interacting bosons).} \label{table:dsm_scaling}
\vskip 0.1in
\centering
\scalebox{0.95}{
\begin{tabular}{c c c c c} 
 \hline
 ~ & $d=2$ & $d=3$ & $d=4$ & $d=5$ \\ [0.5ex] 
 \hline
Training time  &  1770 & 3618 & 5850 & 9240 \\ 
Time per epoch &  0.52 & 1.09 & 1.16 & 1.24 \\
Memory usage &  17.0 & 22.5 & 28.0 & 33.5 \\[1ex] 
 \hline
\end{tabular}
}
\end{table}

\section{Discussion and Limitations}
This paper considers the simplest case of the linear spinless Schr\"odinger equation on a flat manifold $\mathbb{R}^d$ with a smooth potential. For many practical setups, such as quantum chemistry, quantum computing, or condensed matter physics, our approach should be modified, e.g., by adding a spin component or by considering some approximation and, therefore, requires additional validations that are beyond the scope of this work. We have shown evidence of adaptation of our method to one kind of low-dimensional structure, but this paper does not explore a broader range of systems with low latent dimension.

\section{Conclusion}\label{sec:conclusion}
We develop a new algorithm for simulating quantum mechanics that addresses the curse of dimensionality by leveraging the latent low-dimensional structure of the system. This approach is based on a modification of the stochastic mechanics theory that establishes a correspondence between the Schr\"odinger equation and a diffusion process. We learn the drifts of this diffusion process using deep learning to sample from the corresponding quantum density. We believe that our approach has the potential to bring to quantum mechanics simulation the same progress that deep learning has enabled in artificial intelligence.
We provide future work discussion in Appendix \ref{app:future_work}.

\section*{Acknowledgements}

The authors gratefully acknowledge the support of DOE DE-SC0022232, NSF DMS-2023109, NSF PHY2317138, NSF 2209892, and the University of Chicago Data Science Institute. Peter Y. Lu gratefully acknowledges the support of the Eric and Wendy Schmidt AI in Science Postdoctoral Fellowship, a Schmidt Futures program.


\bibliography{references}
\newpage
\appendix
\section{Notation}\label{sec:app_notation}
\begin{itemize}
    \item $\langle a, b\rangle  = \sum_{i=1}^{d} a_{i}b_{i}$ for $a,b \in \mathbb{R}^d$ -- a scalar product.
    \item $\|a\| = \sqrt{\langle a, a\rangle}$ for $a \in \mathbb{R}^d$ -- a norm.
    \item $\mathrm{Tr}(A) = \sum_{i=1}^{d} a_{ii}$ for a matrix $A = \big[a_{ij}\big]_{i=1,j=1}^{d,d}$.
    \item $A(t), B(t), C(t), \ldots$ -- stochastic processes indexed by time $t\ge 0$.
    \item $A_{i}, B_{i}, C_{i}, \ldots$ -- approximations to those processes at a discrete time step $i$, $i = 1, \dots, N$, where $N$ is the number of discritization time points.
    \item $a, b, c$ -- other variables.
    \item $\mathbf{A}, \mathbf{B}, \mathbf{C}, \ldots$ -- quantum observables, e.g., $\mathbf{X}(t)$ -- result of quantum measurement of the coordinate of the particle at moment $t$.
    \item $\rho_{A}(x, t)$ -- a density probability of a process $A(t)$ at time $t$.
    \item $\psi(x, t)$ -- a wave function.
    \item $\psi_{0} = \psi(x, 0)$ -- an initial wave function.
    \item $\rho(x, t) = \big|\psi(x, t)\big|^2$ -- a quantum density.
    \item $\rho_{0}(x) = \rho(x, 0)$ -- an initial probability distribution.
    \item $\psi(x, t) = \sqrt{\rho(x, t)}e^{i S(x, t)}$, where $S(x, t)$ -- a single-valued representative of the phase of the wave function.
    \item $\nabla = \Big(\frac{\partial}{\partial x_{1}} \cdot, \ldots,\frac{\partial}{\partial x_{d}} \cdot\Big)$ -- the gradient operator. If $f:\mathbb{R}^d\rightarrow \mathbb{R}^m$, then $\nabla f(x) \in \mathbb{R}^{d\times m}$ is the Jacobian of $f$, in the case of $m=1$ we call it a gradient of $f$.
    \item $\nabla^2  = \Big[\frac{\partial^2}{\partial x_i \partial x_j}\Big]_{i=1,j=1}^{d,d}$ -- the Hessian operator.
    \item $\nabla^2 \cdot A = \Big[\frac{\partial^2}{\partial x_i \partial x_j}a_{ij}\Big]_{i=1,j=1}^{d,d}$ for $A = \big[a_{ij}(x)\big]_{i=1,j=1}^{d,d}$.
    \item $\langle \nabla, \cdot\rangle$ -- the divergence operator, e.g., for $f:\mathbb{R}^d \rightarrow \mathbb{R}^d$, we have $\langle \nabla, f(x)\rangle = \sum_{i=1}^{d} \frac{\partial}{\partial x_{i}}f_{i}(x)$.
    \item $\Delta = \mathrm{Tr}(\nabla^2)$ -- the Laplace operator.
    \item $m$ -- a mass tensor (or a scalar mass).
    \item $\hbar$ -- the reduced Planck's constant.
    \item $\partial_{y} = \frac{\partial}{\partial y}$ -- a short-hand notation for a partial derivative operator.
    \item $\big[A, B\big] = AB - BA$ -- a commutator of two operators. If one of the arguments is a scalar function, we consider a scalar function as a point-wise multiplication operator.
    \item $|z| = \sqrt{x^2 + y^2}$ for a complex number $z = x + iy \in \mathcal{C}, x,y \in \mathbb{R}$.
    \item $\mathcal{N}(\mu, C)$ -- a Gaussian distribution with mean $\mu \in \mathbb{R}^d$ and covariance $C\in \mathbb{R}^{d\times d}$.
    \item $A \sim \rho$ means that $A$ is a random variable with distribution $\rho$. We do not differentiate between "sample from" and "distributed as", but it is evident from context when we consider samples from distribution versus when we say that something has such distribution.
    \item $\delta_{x}$ -- delta-distribution concentrated at $x$. It is a generalized function corresponding to the "density" of a distribution with a singular support $\{x\}$. 
\end{itemize}

\section{DSM Algorithm}\label{sec:app_algo}

We present detailed algorithmic descriptions of our method: Algorithm \ref{alg:generate} for batch generation and Algorithm \ref{alg:train_orig} for model training. During inference, distributions of $X_i$ converge to $\rho=|\psi|^2$, thereby yielding the desired outcome. Furthermore, by solving \cref{eq:nav-s1} on points generated by the current best approximations of $u, v$, the method exhibits self-adaptation behavior. Specifically, it obtains its current belief where $X(t)$ is concentrated, updates its belief, and iterates accordingly. With each iteration, the method progressively focuses on the high-density regions of $\rho$, effectively exploiting the low-dimensional structure of the underlying solution.
\begin{algorithm}
\small
\caption{GenerateBatch($u, v, \rho_{0}, \nu, T, B, N$) -- sample trajectories}\label{alg:generate} 
\begin{algorithmic}
\STATE {\bfseries Physical hyperparams:}  $T$ -- time horizon, $\psi_{0}$ -- initial wave-function.
\STATE {\bfseries Hyperparams: } $\nu \ge 0$ -- diffusion constant, $B \ge 1$ -- batch size, $N \ge 1$ -- time grid size.
\vspace{2pt}
        \STATE $t_{i} =  i T/N$ for $0 \le i\le N$
        \STATE sample $X_{0j} \sim \big|\psi_{0}\big|^2$ for $1 \le j B$
        \vspace{2pt}
        \vspace{2pt}
        \FOR{$1 \le i \le N$}
            \STATE sample $\xi_{j}\sim \mathcal{N}(0, I_{d})$ for $1 \le j \le B$
            \STATE $X_{ij} = X_{(i-1) j} + \frac{T}{N}\big(v_{\theta}(X_{(i-1)j}, t_{i-1})+\nu u_{\theta}(X_{(i-1)j}, t_{i-1})\big) + \sqrt{\frac{\nu\hbar T}{mN}}\xi_{j}$ for $1 \le j \le B$
        \ENDFOR
        \STATE {\bfseries output } \Big\{$\big\{X_{ij}\big\}_{j=1}^{B}\Big\}_{i=0}^{N}$ 
\end{algorithmic}
\end{algorithm}

\begin{algorithm}
\small
\caption{A training algorithm}\label{alg:train_orig}
\begin{algorithmic}
\STATE {\bfseries Physical hyperparams:}  $m > 0 $ -- mass, $\hbar > 0$ -- reduced Planck constant,  $T$ -- a time horizon, $\psi_{0}:\mathbb{R}^d\rightarrow \mathbb{C}$ -- an initial wave function, $V:\mathbb{R}^d\times [0, T]\rightarrow \mathbb{R}$ -- potential.
\STATE {\bfseries Hyperparams: } $\eta > 0$ -- learning rate for backprop, $\nu > 0$ -- diffusion constant, $B \ge 1$ -- batch size, $M \ge 1$ -- optimization steps, $N \ge 1$ -- time grid size, $w_{u}, w_{v}, w_{0} > 0$ -- weights of losses.
\vspace{2pt}
        \vspace{2pt}
\STATE {\bfseries Instructions:}
        \STATE $t_{i} = i T / N$ for $0 \le i\le N$
        \vspace{2pt}
\FOR{$1 \le \tau \le M$}
        \STATE $X = \mathrm{GenerateBatch}(u_{\theta_{\tau-1}}, v_{\theta_{\tau-1}}, \psi_{0}, \nu, T, B, N)$ 
        \STATE define $L^{u}_{\tau}(\theta) = \frac{1}{(N+1)B}\sum_{i=0}^{N}\sum_{j=1}^{B}\big\| \partial_{t} {  u_\theta}(X_{ij}, t_{i}) - \mathcal{D}_{u}[u_{\theta}, v_{\theta}, X_{ij}, t_{i}]  \big\|^2$
        \STATE define $L^{v}_{\tau}(\theta) = \frac{1}{(N+1)B} \sum_{i=0}^{N}\sum_{j=1}^{B}\big\| \partial_{t} v_\theta(X_{ij}, t_{i}) -  \mathcal{D}_{v}[u_{\theta}, v_{\theta}, X_{ij}, t_{i}]  \big\|^2$
        \STATE define $L^{0}_{\tau}(\theta) = \frac{1}{B}\sum_{j=1}^{B}\Big(\big\|u_{\theta}(X_{0j}, t_{0}) - u_0(X_{0j})\big\|^2 + \big\|v_{\theta}(X_{0j}, t_0) - v_{0}(X_{0j}, t_{0})\big\|^2\Big)$
        \STATE define $\mathcal{L}_{\tau}(\theta) = w_{u} L^{u}_{\tau}(\theta) + w_{v} L^{v}_{\tau}(\theta) + w_{0} L^{0}_{\tau}(\theta)$
        \vspace{2pt}
        \STATE $\theta_{\tau} = \mathrm{OptimizationStep}(\theta_{\tau-1}, \nabla_{\theta} \mathcal{L}_{\tau}(\theta_{\tau-1}), \eta)$
\ENDFOR
\STATE {\bfseries output }$u_{\theta_{M}}, v_{\theta_{M}}$ 
\end{algorithmic}
\end{algorithm}

\section{Experiment Setup Details}\label{sec:app_training}

\subsection{Non-Interacting System}\label{sec:app_noninter_nn_details}
In our experiments, we set $m=1$, $\hbar = 10^{-2}$\footnote{The value of the reduced Plank constant depends on the metric system that we use and, thus, for our evaluations we are free to choose any value.}, $\sigma^2=10^{-1}$. For the harmonic oscillator model, $N = 1000$ and the batch size $B=100$; for the singular initial condition problem, $N=100$ and $B=100$.  
For evaluation, our method samples 10000 points per time step, and the observables are estimated from these samples; we run the model this way ten times. 

\subsubsection{A Numerical Solution} 

\paragraph{1d harmonic oscillator with $S_0 (x) \equiv 0$} To evaluate our method's performance, we use a numerical solver that integrates the corresponding differential equation given the initial condition. We use \textsc{SciPy} library \citep{virtanen2020scipy}. The solution domain is $x \in [-2, 2]$ and  $t \in [0, 1]$, where  $x$ is split into 566 points and $t$ into 1001 time steps. This solution can be repeated $d$ times for the $d$-dimensional harmonic oscillator problem. 

\paragraph{1d harmonic oscillator with $S_0(x) = -5x$} We use the same numerical solver as for the $S_0(x)\equiv 0$ case. The solution domain is $x \in [-2, 2]$ and  $t \in [0, 1]$, where $x$ is split into 2829 points and $t$ is split into 1001 time steps.

\subsubsection{Architecture and Training Details}

A basic NN architecture for our approach and the PINN is a feed-forward NN with one hidden layer with $\tanh$ activation functions. We represent the velocities $u$ and $v$ using this NN architecture with 200 neurons in the case of the singular initial condition. The training process takes about 7 mins. For $d=1$, a harmonic oscillator with zero initial phase problem, there are 200 neurons for our method and 400 for the PINN; for $d=3$ and more dimensions, we use 400 neurons. This rule holds for the experiments measuring total training time in Section \ref{sec:res_complexity}. In a 1d harmonic oscillator with a non-zero initial phase problem, we use 300 hidden neurons in our models. In the experiments devoted to measuring time per epoch (from Section \ref{sec:res_complexity}), the number of hidden neurons is fixed to 200 for all dimensions. We use the Adam optimizer \citep{kingma2014adam} with a learning rate $10^{-4}$. In our experiments, we set $w_{u}=1, w_{v}=1, w_{0} =1$. For PINN evaluation, the test sets are the same as the grid for the numerical solver. In our experiments, we usually use a single NVIDIA A40 GPU. For the results reported in Section \ref{sec:res_complexity}, we use an NVIDIA A100 GPU.

\subsubsection{On Optimization}

We use Adam optimizer \citep{kingma2014adam} in our experiments. Since the operators in \cref{eq:diff_operators} are not linear, we may not be able to claim convergence to the global optima of such methods as SGD or Adam in the Neural Tangent Kernel (NTK) \citep{jacot2018neural} limit. Such proof exists for PINNs in \citet{wang2022and} due to the linearity of the Schr\"odinger \cref{eq:Schrodinger}. It is possible that non-linearity in the loss \cref{eq:L_full} requires non-convex methods to achieve theoretical guarantees on convergence to the global optima \citep{raginsky2017nonconvex, muzellec2020dimensionfree}. Further research into NTK and non-linear PDEs is needed \citep{wang2022and}.

The only noise source in our loss \cref{eq:L_full} comes from trajectory sampling. This fact contrasts sharply with generative diffusion models relying on score matching  \citep{yang2022diffusion}. In these models, the loss has $\mathcal{O}(\epsilon^{-1})$ variance as it implicitly attempts to numerically estimate the stochastic differential $\frac{X(t+\epsilon)-X(t)}{\epsilon}$ which leads to $\frac{1}{\sqrt{\epsilon}}$ contribution from increments of the Wiener process. In our loss, the stochastic differentials are evaluated analytically in \cref{eq:diff_operators} avoiding such contributions; for details, see \citet{NelsonOG, nelson2005mystery}. This leads to $\mathcal{O}(1)$ variance of the gradient and, thus, allows us to achieve fast convergence with smaller batches.

\subsection{Interacting System} \label{sec:app_inter_nn_details}

In our experiments, we set $m=1$, $\hbar = 10^{-1}$, $\sigma^2=10^{-1}$. The number of time steps is $N = 1000$, and the batch size $B=100$.

\paragraph{Numerical solution} As a numerical solver, we use the \textsc{qmsolve} library \footnote{https://github.com/quantum-visualizations/qmsolve}. The solution domain is $x \in [-1.5, 1.5]$ and  $t \in [0, 1]$, where  $x$ is split into 100 points and $t$ into 1001 time steps.

\subsubsection{Architecture and training details}

Instead of a multi-layer perceptron, we follow the design choice of \citet{jiang2022embed} to use residual connection blocks. In our experiments, we used the $\tanh$ as the activation function, set the hidden dimension to 300, and used the same architecture for both DSM and PINN. Empirically, we find out that this design choice leads to faster convergence in terms of training time. The PINN inputs are $N_0 = 10000$, $N_b = 1000$ data points, and  $N_f = 1000000$ collocation points. We use Adam optimizer \citep{kingma2014adam} with a learning rate $10^{-4}$ in our experiments.  We use loss weights $w_{u}=1, w_{v}=1, w_{0} =1$.

\paragraph{Permutation invariance} Since our system comprises two identical bosons, we enforce symmetry for both the DSM and PINN models. Specifically, we sort the neural network inputs $x$ to ensure the permutation invariance of the models. While this approach guarantees adherence to the physical symmetry property, it comes with a computational overhead from the sorting operation. For higher dimensional systems, avoiding such sorting may be preferable to reduce computational costs. However, for the two interacting particle system considered here, the performance difference between regular and permutation-invariant architectures is not significant.

\paragraph{t-VMC ansatz} To enable a fair comparison between our DSM approach and t-VMC, we initialize the t-VMC trial wave function with a complex-valued multi-layer perceptron architecture identical to the one employed in our DSM method. However, even after increasing the number of samples and reducing the time step, the t-VMC method exhibits poor performance with this neural network ansatz. This result suggests that, unlike our diffusion-based DSM approach, t-VMC struggles to achieve accurate results when utilizing simple off-the-shelf neural network architectures as the ansatz representation.

As an alternative ansatz, we employ a harmonic oscillator basis expansion, expressing the wave function as a linear combination of products of basis functions. This representation scales quadratically with the number of basis functions but forms a complete basis set for the two-particle problem. Using the same number of samples and time steps, this basis expansion approach achieves significantly better performance than our initial t-VMC baseline. However, it still does not match the accuracy levels attained by our proposed DSM method.  This approach does not scale well naively to larger systems but can be adapted to form a 2-body Jastrow factor \citep{carleo2017unitary}. We expect this to perform worse for larger systems due to the lack of higher-order interactions in the ansatz. In our t-VMC experiments, we use the \textsc{NetKet} library \citep{vicentini2022netket} for many-body quantum systems simulation.

\section{Experimental Results}\label{sec:app_res}
\subsection{Singular initial conditions}\label{app:singular_ic}
As a proof of concept, we consider a case of one particle $x \in \mathbb{R}^1$ with $V(x) \equiv 0$ and $\psi_{0} = \delta_{0}$, $t \in [0, 1]$. Since $\delta$-function is a generalized function, we must take a $\delta$-sequence for the training. The most straightforward approach is to take $\widetilde{\psi_{0}} = \frac{1}{(2\pi \alpha)^{\frac{1}{4}}}e^{-\frac{x^2}{4\alpha}}$ with $\alpha \rightarrow 0_+$. In our experiments we take $\alpha = \frac{\hbar^2}{m^2}$, yielding $v_{0} (x) \equiv 0$ and $u_{0}(x) = - \frac{\hbar x}{2m\alpha}$. Since $\psi_{0}$ is singular, we must set $\nu = 1$ during sampling. The analytical solution is given as $\psi (x, t) = \frac{1}{(2\pi t)^{\frac{1}{4}}}e^{-\frac{x^2}{4t}}$. So, we expect the standard deviation of $X(t)$ to grow as $\sqrt{t}$, and the mean value of $X(t)$ to be zero.

We do not compare our approach with PINNs since it is a simple proof of concept, and the analytical solution is known. \cref{fig:singular_example}  summarizes the results of our experiment. Specifically, the left panel of the figure shows the magnitude of the density obtained with our approach alongside the true density. The right panel of \cref{fig:singular_example} shows statistics of $X_t$, such as mean and variance, and the corresponding error bars. The resulting prediction errors are calculated against the truth data for this problem and are measured at $0.008 \pm 0.007$ in the $L_2$-norm for the averaged mean and $0.011 \pm 0.007$ in the relative $L_2$-norm for the averaged variance of $X_t$. Our approach can accurately capture the behavior of the Schr\"odinger equation in the singular initial condition case.

\begin{figure}[!ht]
    \centering
    \includegraphics[width=0.9\textwidth]{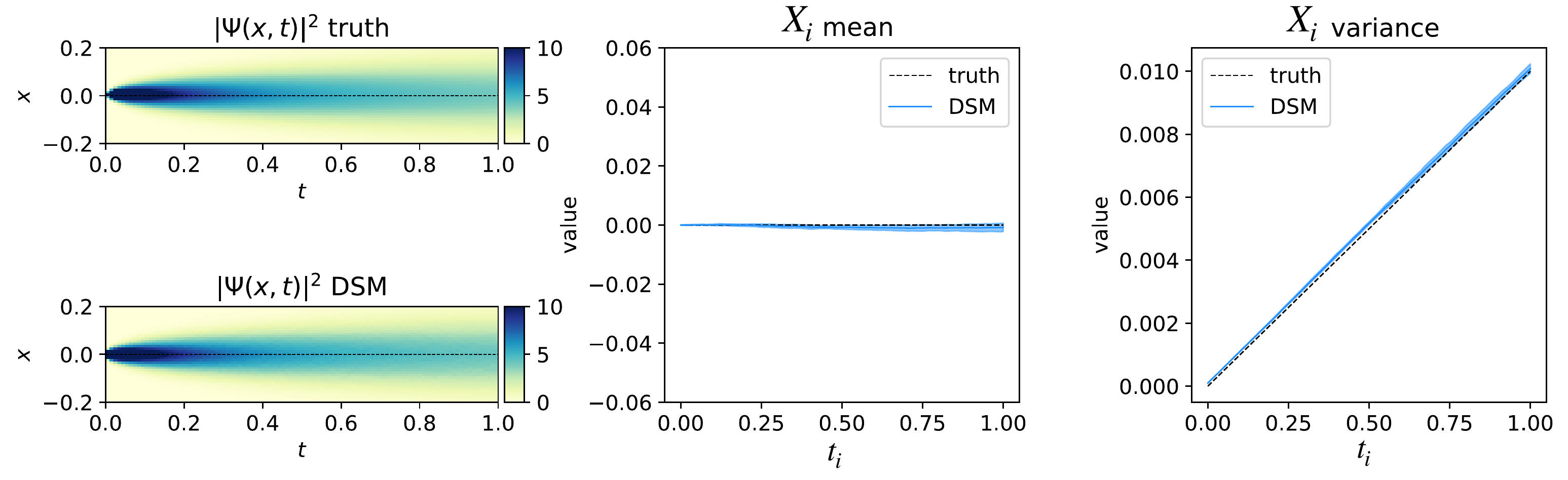}
    \caption{Results for the singular initial condition case. DSM corresponds to our method. }
    \label{fig:singular_example}
\end{figure}

\subsection{3D Harmonic Oscillator}\label{sec:app_res_3d}
We further explore our approach by considering the harmonic oscillator model with $S_0 (x) \equiv 0$ with three non-interacting particles. This setting can be viewed as a 3d problem, where the solution is a 1d solution repeated three times. Due to computational resource limitations, we are unable to execute the PINN model. The number of collocation points should grow exponentially with the problem dimension so that the PINN model converges. We have about 512 GB of memory but cannot store $60000^3$ points.   We conduct experiments comparing two versions of the proposed algorithm: the Nelsonian one
and our version.
Table \ref{table:osc_1d_comparison} provides the quantitative results of these experiments. Our version demonstrates slightly better performance compared to the Nelsonian version, although the difference is not statistically significant. Empirically, our version requires more steps to converge compared to the Nelsonian version: 7000 vs. 9000 epochs correspondingly. However, the training time of the Nelsonian approach is about 20 mins longer than our approach's time. 

\cref{fig:osc_3d_stats} demonstrates the obtained statistics with the proposed algorithm's two versions (Nelsonian and Gradient Divergence) for every dimension. \cref{fig:osc_3d_dens} compares the density function for every dimension for these two versions. \cref{table:osc_3d_comparison} summarizes the error rates per dimension. The results suggest no significant difference in the performance of these two versions of our algorithm. The Gradient Divergence version tends to require more steps to converge, but it has quadratic time complexity in contrast to the cubic complexity of the Nelsonian version.

\begin{figure}[ht]
    \centering
    \subfigure[The Nelsonian version.]{\includegraphics[width=0.44\textwidth]{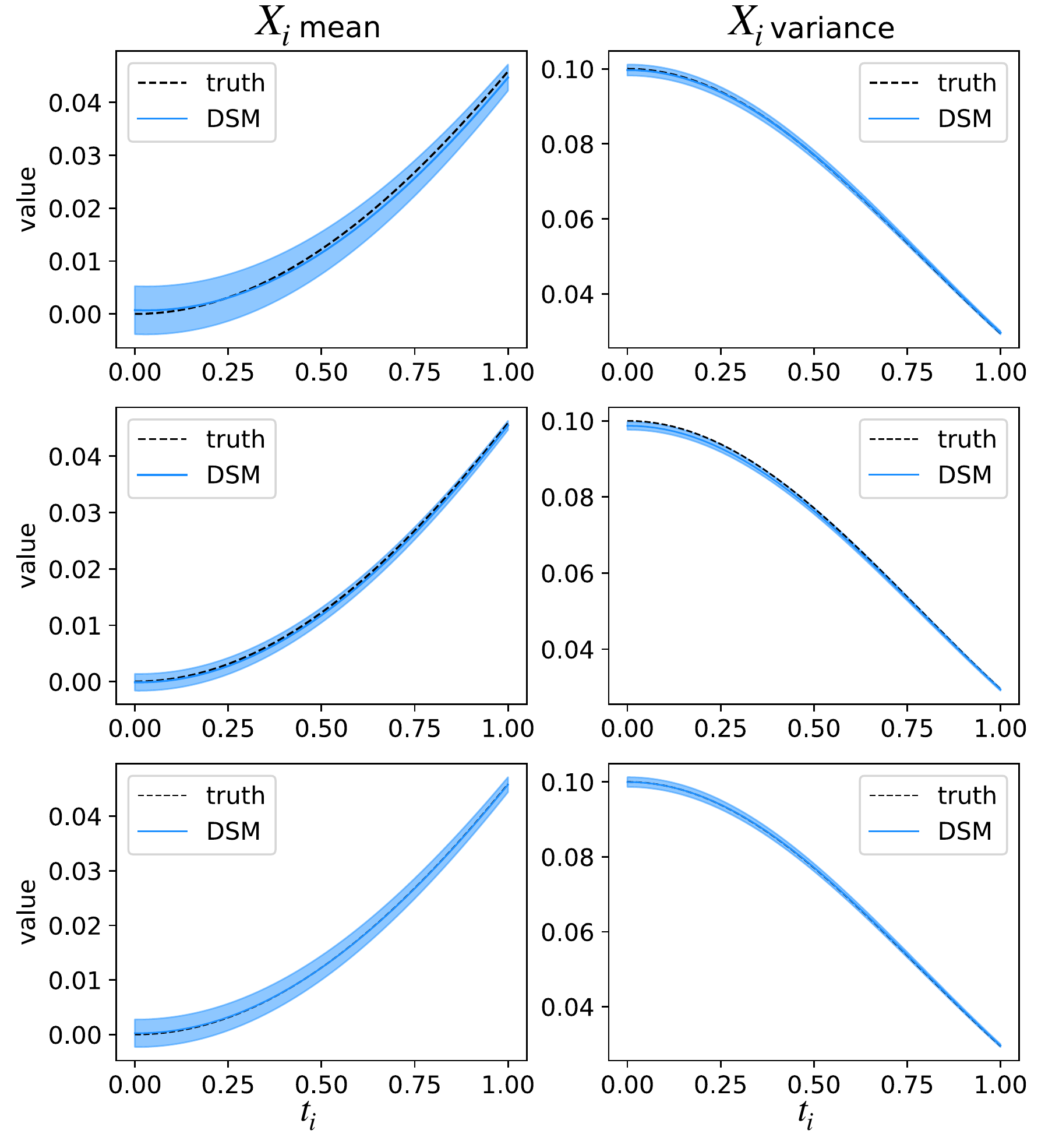}} 
    \subfigure[The Gradient Divergence version.] 
    {\includegraphics[width=0.44\textwidth]{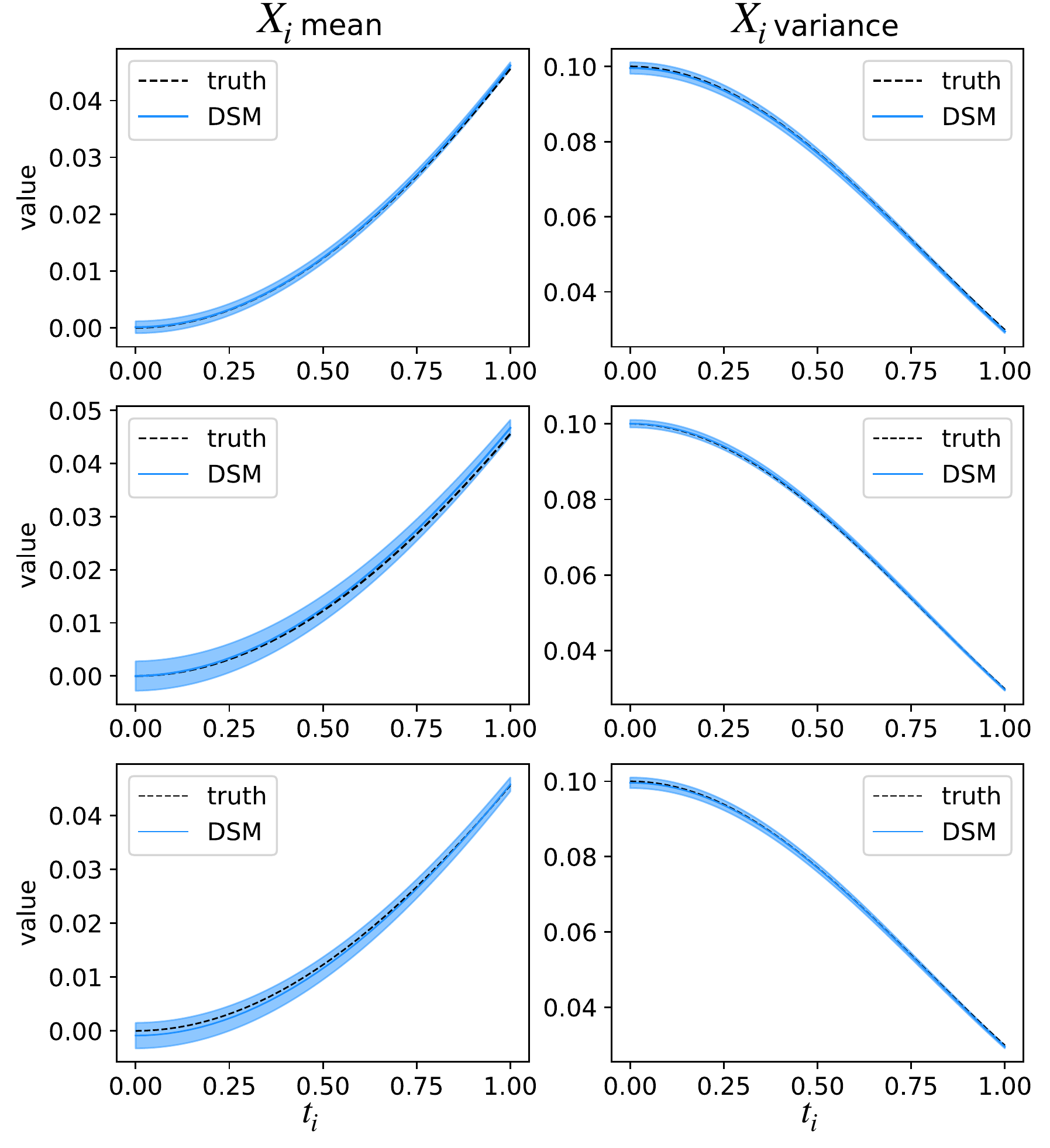}}
    \caption{The obtained statistics for 3d harmonic oscillator using two versions of the proposed approach.}
    \label{fig:osc_3d_stats}
\end{figure}

\begin{figure}[ht]
    \centering
    \subfigure[The Nelsonian version.]{\includegraphics[width=0.44\textwidth]{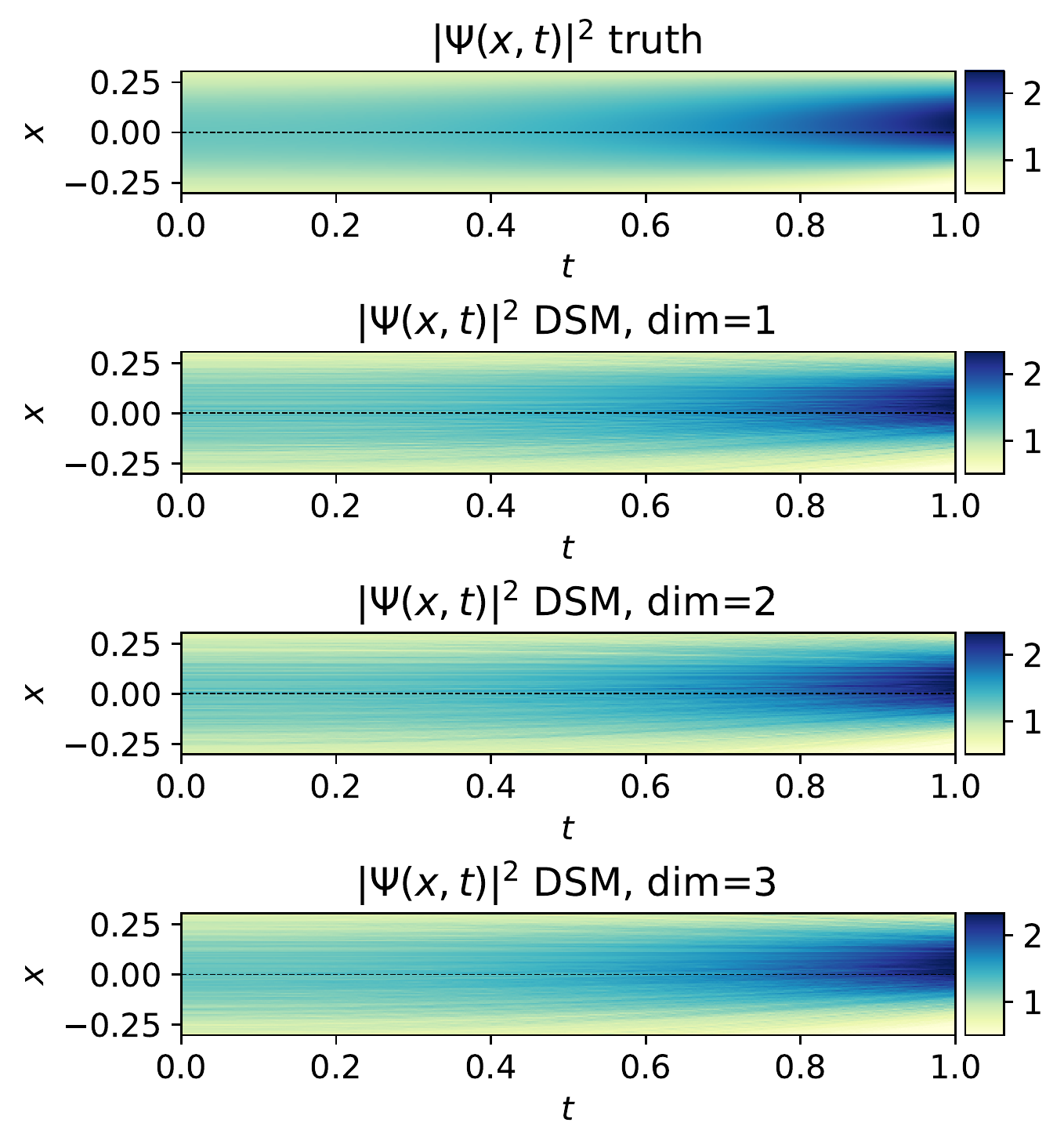}} 
    \subfigure[The Gradient Divergence version.]{\includegraphics[width=0.44\textwidth]{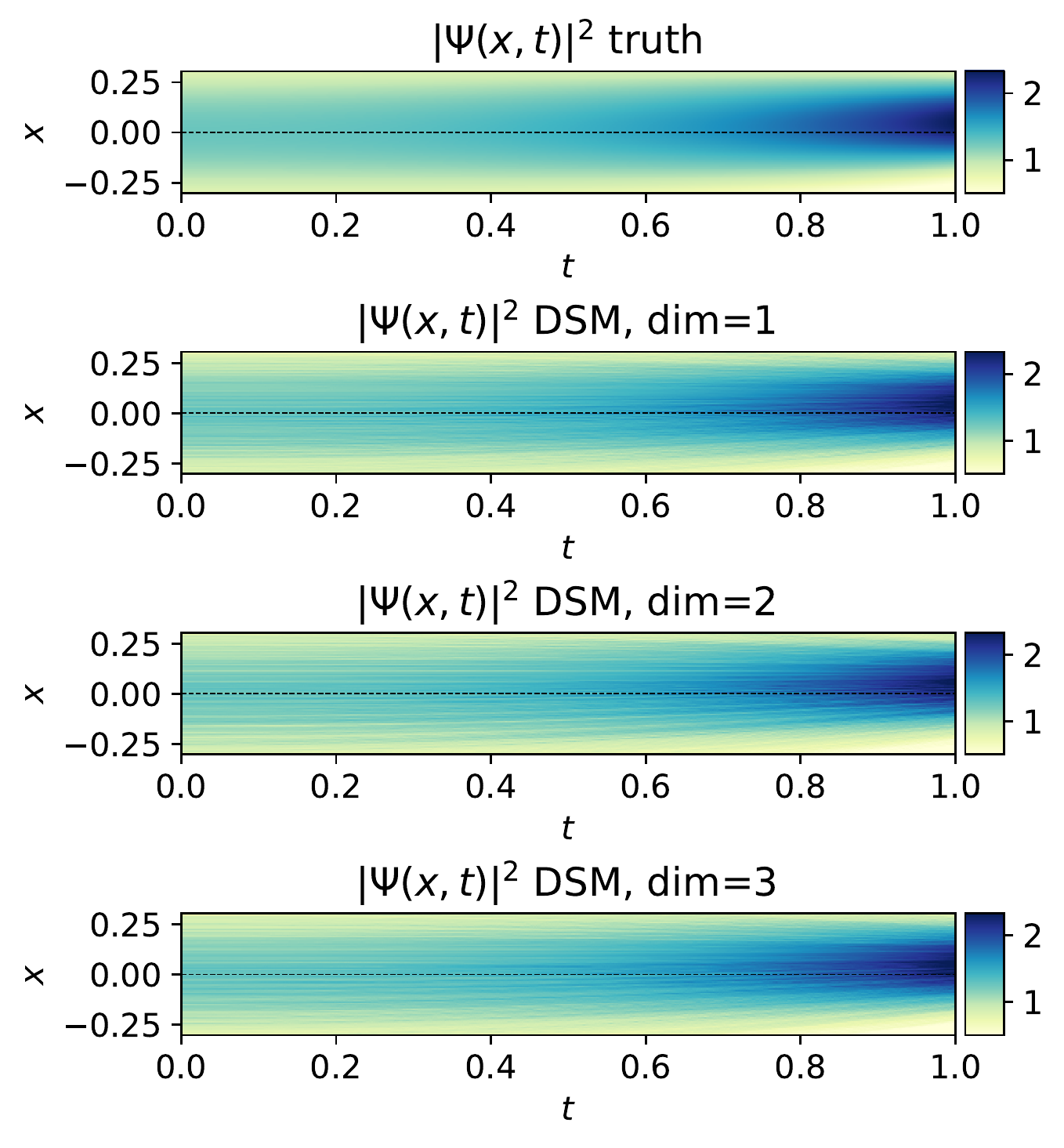}} 
    \caption{The density function for 3d harmonic oscillator using two versions of the proposed approach. }
    \label{fig:osc_3d_dens}
\end{figure}

\begin{table}[ht!]
\caption{The results for 3d harmonic oscillator with $S_0 (x) \equiv 0$ using two versions of the proposed approach: the Nelsonian one uses the Laplacian operator in the training loss, the Gradient Divergence version is our modification that replaces Laplacian with gradient of divergence.}
\label{table:osc_3d_comparison}
\centering
\small
\begin{tabular}{c c c c c} 
 \hline
 Model & $\mathcal{E}_m(X^{(1)}_i)$ $\downarrow$ & $\mathcal{E}_m (X^{(2)}_i)$ $\downarrow$ & $\mathcal{E}_m (X^{(3)}_i)$ $\downarrow$   & $\mathcal{E}_m (X_i)$ $\downarrow$ \\ [0.5ex] 
 \hline
 DSM (Nelsonian) &  0.170 $\pm$ 0.081  & 0.056 $\pm$ 0.030 & \bf 0.073 $\pm$ 0.072 &  0.100 $\pm$ 0.061 \\
 DSM (Gradient Divergence) &  \bf 0.038 $\pm$  0.023  & \bf 0.100 $\pm$ 0.060 & 0.082 $\pm$ 0.060 &  \bf 0.073 $\pm$ 0.048 \\[1ex] 
 \hline
 Model & $\mathcal{E}_v(X^{(1)}_i)$ $\downarrow$ & $\mathcal{E}_v(X^{(2)}_i)$ $\downarrow$ & $\mathcal{E}_v(X^{(3)}_i)$ $\downarrow$   & $\mathcal{E}_v(X_i)$ $\downarrow$ \\  [0.5ex] 
 \hline
 DSM (Nelsonian) & \bf 0.012 $\pm$ 0.009 &  0.012 $\pm$ 0.009 &  0.011  $\pm$ 0.008 & 0.012 $\pm$ 0.009 \\
 DSM (Gradient Divergence) & \bf 0.012 $\pm$ 0.010 &  \bf 0.009 $\pm$ 0.005 &  \bf 0.011  $\pm$ 0.010 & \bf 0.011  $\pm$ 0.008\\ [1ex] 
 \hline
 Model & $\mathcal{E}(v^{(1)})$ $\downarrow$ & $\mathcal{E}(v^{(2)})$ $\downarrow$ & $\mathcal{E}(v^{(3)})$ $\downarrow$   & $\mathcal{E}(v))$ $\downarrow$ \\ [0.5ex] 
 \hline
 DSM (Nelsonian) &  0.00013 &  0.00012 & 0.00012 & 0.00012  \\
 DSM (Gradient Divergence) & $\bf 4.346 \times 10^{-5}$ &   $\bf 4.401\times 10^{-5}$ &  $ \bf4.700 \times 10^{-5}$ & $ \bf 4.482 \times 10^{-5}$  \\[1ex] 
 \hline
 Model & $\mathcal{E}(u^{(1)})$ $\downarrow$ & $\mathcal{E}(v^{(2)})$ $\downarrow$ & $\mathcal{E}(v^{(3)})$ $\downarrow$   & $\mathcal{E}(v)$ $\downarrow$ \\ [0.5ex] 
 \hline
 DSM (Nelsonian) &  $\bf 4.441 \times 10^{-5}$ &  $ \bf2.721 \times 10^{-5}$ & $2.810 \times 10^{-5}$ &  $ \bf3.324 \times 10^{-5}$ \\ 
  DSM (Gradient Divergence) &  $6.648 \times 10^{-5}$ & $4.405 \times 10^{-5}$ & $ \bf 1.915 \times 10^{-5}$ & $4.333 \times 10^{-5}$ \\ [1ex] 
 \hline
\end{tabular}
\end{table}

\subsection{Naive Sampling}\label{app:naive_sampling}

\cref{fig:naive_sample_all} shows performance of Gaussian sampling approach applied to the harmonic oscillator and the singular initial condition setting.  \cref{table:sampling_compare} compares results of all methods. Our approach converges to the ground truth while naive sampling does not. \cref{fig:naive_sample_all} illustrates performance of Gaussian sampling.

\begin{table}[ht!]
\caption{Error rates for different problem settings using two sampling schemes: our (DSM) and Gaussian sampling. Gaussian sampling replaces all measures in the expectations with Gaussian noise in \cref{eq:L_full}.  The \textbf{best} result is in bold. These results demonstrate that our approach work better than the naïve sampling scheme.}
\label{table:sampling_compare}
\vskip 0.1in
\centering
\small
\scalebox{0.91}{
\begin{tabular}{c c c c c c c} 
 \hline
 Problem & Model & $\mathcal{E}_m(X_i)$ $\downarrow$ & $\mathcal{E}_v(X_i)$ $\downarrow$ & $\mathcal{E}(v)$ $\downarrow$ & $\mathcal{E}(u)$ $\downarrow$ \\ [0.5ex] 
 \hline
Singular IC &  Gaussian sampling & 0.043 $\pm$ 0.042 &0.146 $\pm$ 0.013 & 1.262 &   0.035 \\
& DSM & \bf 0.008 $\pm$ 0.007  & \bf 0.011 $\pm$ 0.007 &   $ \bf 0.524 $ &  $\bf 0.008$ \\ [1ex]
\hline 
\multirow{2}{*}{\begin{tabular}{@{}c@{}} Harm osc 1d, \\ $S_0(x)\equiv 0$ \end{tabular}} & Gaussian sampling &0.294 $\pm$ 0.152   & 0.488 $\pm$ 0.018 & 3.19762 & 1.18540 \\ 
 & DSM & \bf 0.077 $\pm$ 0.052  & \bf 0.011 $\pm$ 0.006 & \bf 0.00011  & $ \bf 2.811\times 10^{-5}$   \\[1ex] 
 \hline
 \multirow{2}{*}{\begin{tabular}{@{}c@{}} Harm osc 1d, \\ $S_0(x)=-5x$ \end{tabular}} & Gaussian sampling &0.836 $\pm$ 0.296  & 0.086 $\pm$ 0.007 & 77.57819 & 24.15156 \\ 
 & DSM & \bf 0.223 $\pm$ 0.207  & \bf 0.009 $\pm$ 0.008 & $  \bf 1.645\times 10^{-5}$  & $ \bf 2.168 \times 10^{-5}$  \\[1ex] 
 \hline
\multirow{2}{*}{\begin{tabular}{@{}c@{}} Harm osc 3d, \\ $S_0(x)\equiv 0$ \end{tabular}} & Gaussian sampling   & 0.459 $\pm$ 0.126  & 5.101 $\pm$ 0.201 & 13.453 & 5.063 \\
 &   DSM & \bf 0.073 $\pm$ 0.048 & \bf 0.011 $\pm$ 0.008 &   $ \bf 4.482\times 10^{-5}$ &   $\bf 4.333\times 10^{-5}$ \\ [1ex] 
 \hline
\end{tabular}}
\end{table}

\begin{figure}[!ht]
\centering
\includegraphics[width=0.9\linewidth]{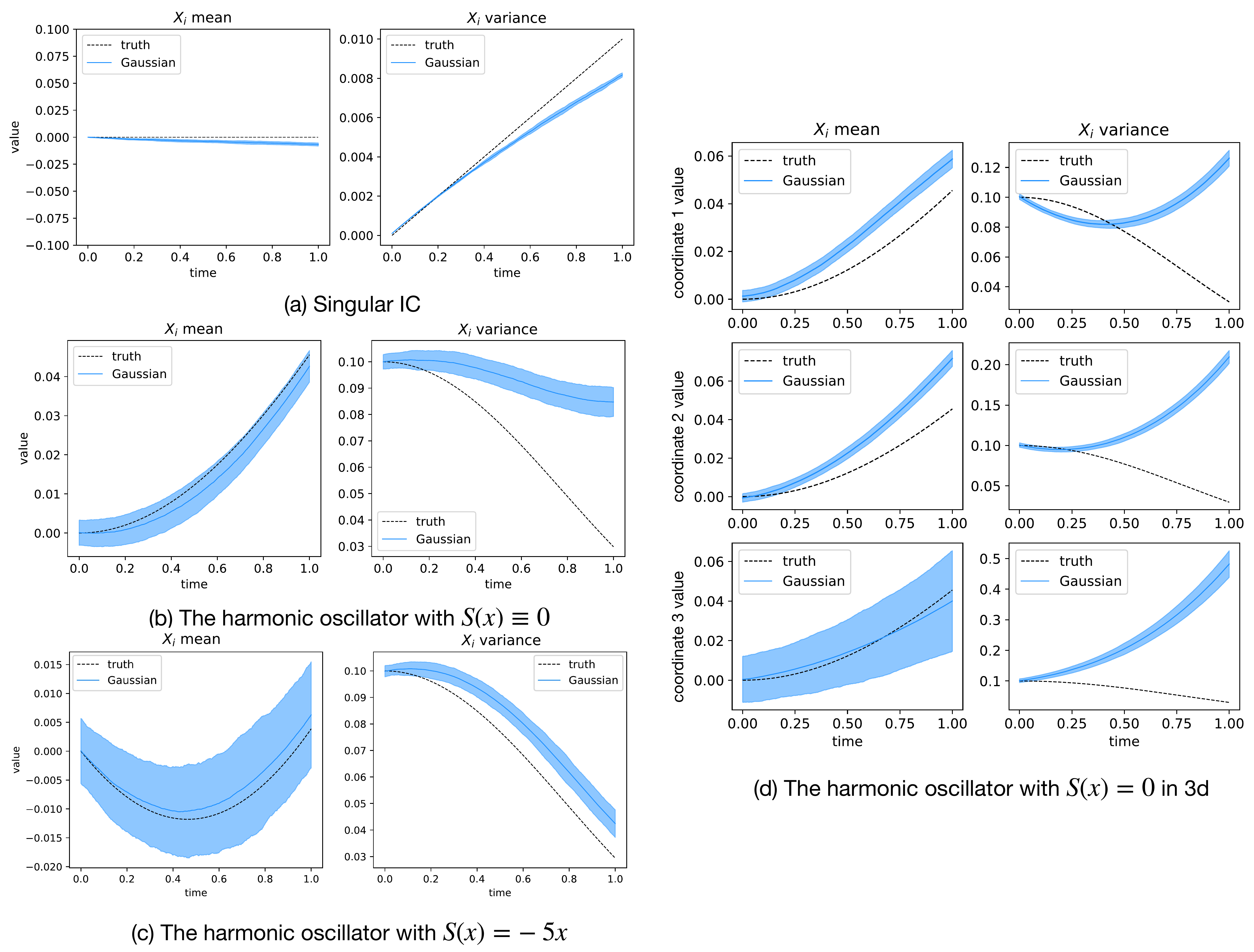}
\caption{\label{fig:naive_sample_all} An illustration of produced trajectories using the naïve Gaussian sampling scheme as a comparison with the proposed approach. The obtained trajectories do not match the solution, while the results in our paper suggest that the proposed DSM approach converges better. Compare with Figures \ref{fig:singular_example}, \ref{fig:boson_interact_2d_both_new},
\ref{fig:osc_3d_stats}.}
\end{figure}

\subsection{Scaling Experiments for Non-Interacting System} \label{app:scaling_noninteract}

We empirically estimate memory allocation on a GPU (NVIDIA A100) when training two versions of the proposed algorithm. In addition, we estimate the number of epochs until the training loss function is less than $10^{-2}$ for different problem dimensions. The results are visualized in Figure \ref{fig:complexity_memory_loss}(a) proves the memory usage of the Gradient Divergence version grows linearly with the dimension while it grows quadratically in the Nelsonian version. We also empirically access the convergence speed of two versions of our approach. Figure \ref{fig:complexity_memory_loss}(b) shows how many epochs are needed to make the training loss less than $1 \times 10^{-2}$. Usually, the Gradient Divergence version requires slightly more epochs to converge to this threshold than the Nelsonian one. The number of epochs is averaged across five runs. In both experiments, the setup is the same as we describe in Section \ref{sec:res_complexity}.

\begin{figure}[ht]
\vskip 0.1in
    \centering 
    \subfigure[GPU memory usage.]{\includegraphics[width=0.4\textwidth]{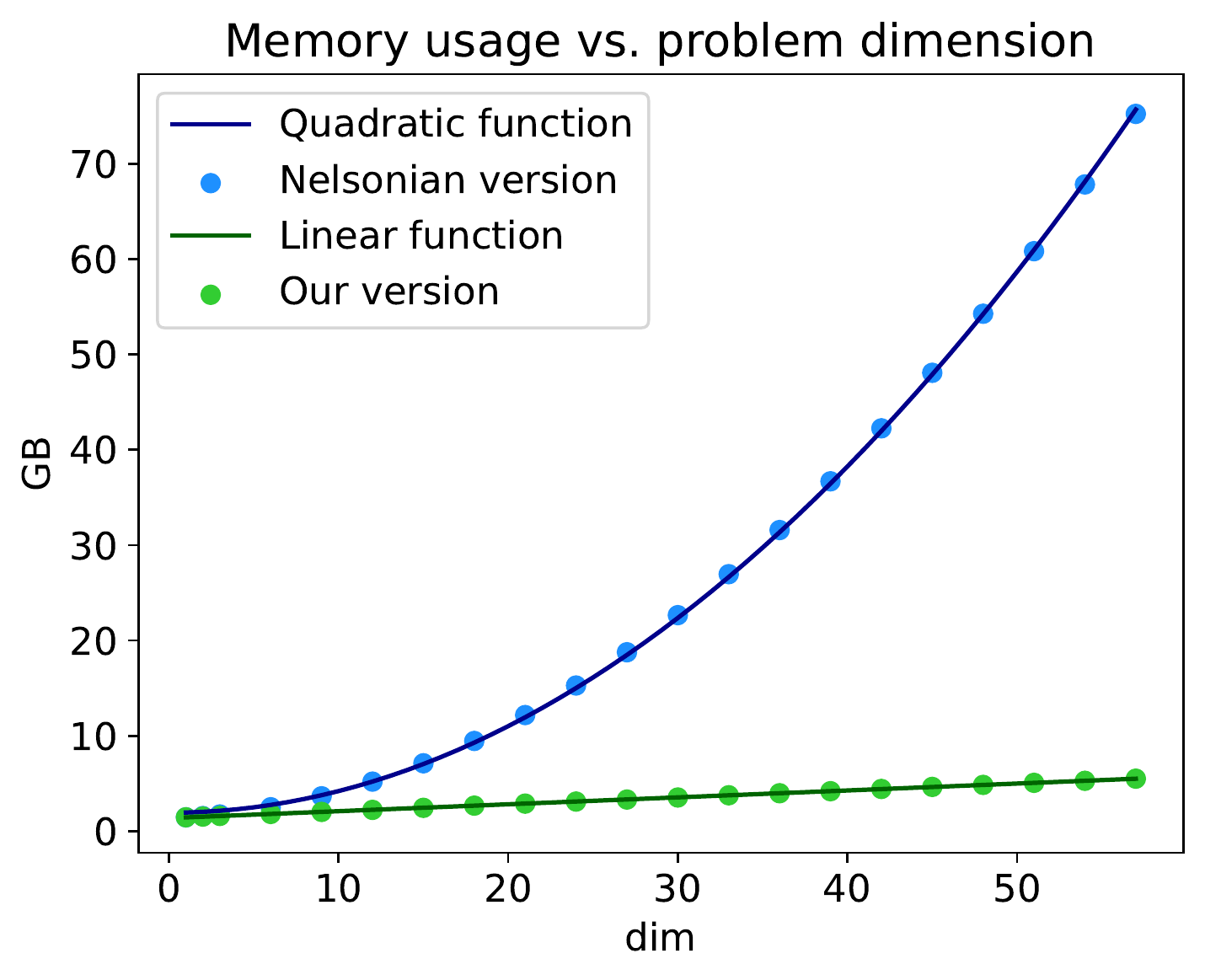}}
    \subfigure[Number of epochs until the training loss $<10^{-2}$.]{\includegraphics[width=0.4\textwidth]{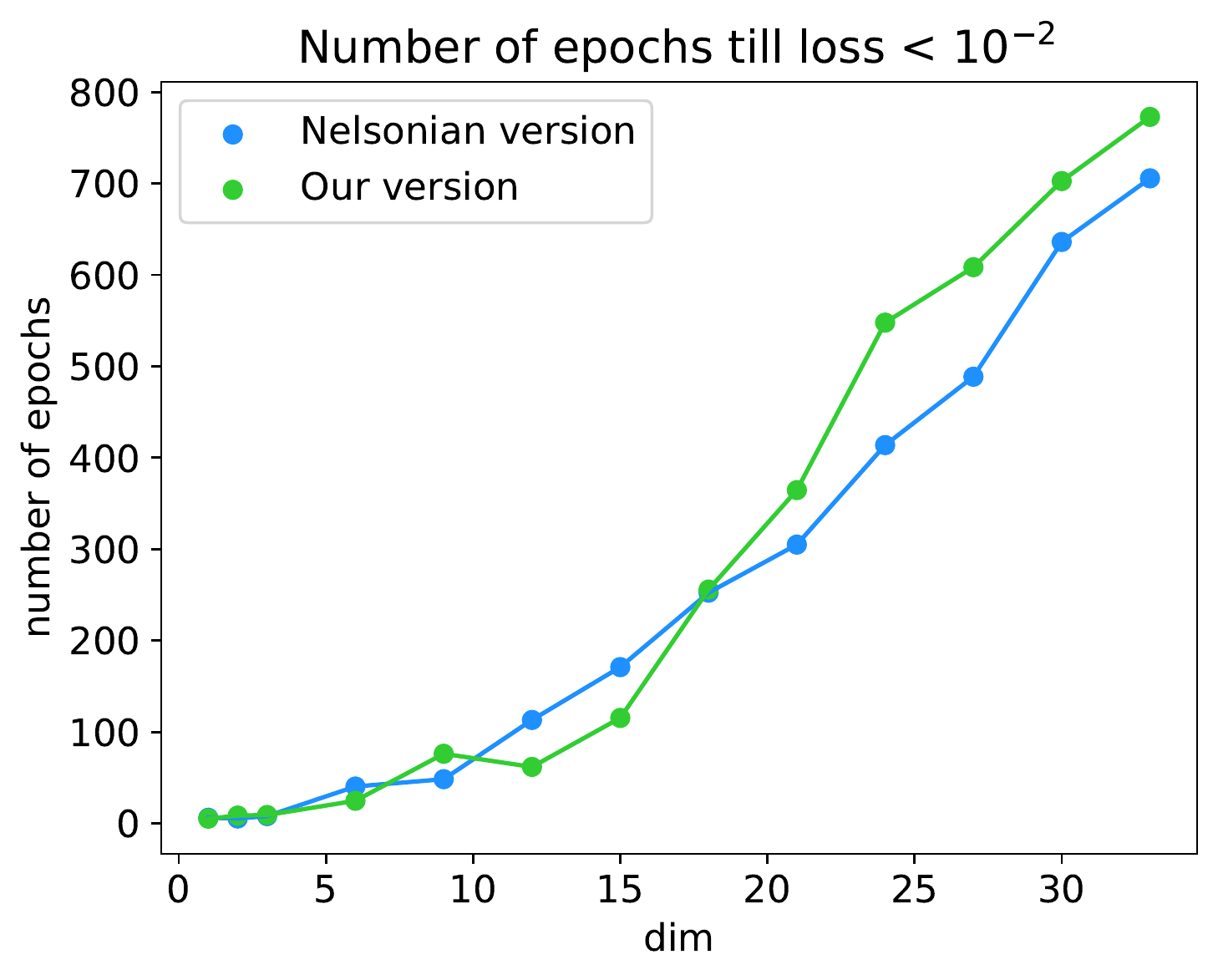}} 
    \caption{Empirical complexity evaluation of two versions of the proposed method: memory usage and the number of epochs until the loss is less than the threshold.}
    \label{fig:complexity_memory_loss}
    \vskip -0.1in
\end{figure}

Also, we provide more details on the experiment measuring the total training time per dimensions $d=1, 3, 5, 7, 9$. This experiment is described in Section \ref{sec:res_complexity}, and the training time grows linearly with the problem dimension. Table \ref{table:osc_d_train_time_err} presents the error rates and train time. The results show that the proposed approach can perform well for every dimension while the train time scales linearly with the problem dimension.

\begin{table}[ht!]
\caption{Training time and test errors for the harmonic oscillator model for different $d$.} \label{table:osc_d_train_time_err}
\vskip 0.15in
\centering
\begin{sc}
\begin{tabular}{c c c c c c} 
 \hline
 $d$ & $\mathcal{E}_m(X_i)$ $\downarrow$ & $\mathcal{E}_v(X_i)$ $\downarrow$ & $\mathcal{E}(v)$ $\downarrow$ & $\mathcal{E}(u)$ $\downarrow$ & Train time \\ [0.5ex] 
 \hline
 1 & 0.074 $\pm$ 0.052  & 0.009 $\pm$ 0.007 & 0.00012  & 2.809e-05 & 46m 20s \\ 
 3 & 0.073 $\pm$ 0.048  & 0.010 $\pm$ 0.008 & 4.479e-05  & 3.946e-05 & 2h 18m  \\
 5 & 0.081 $\pm$ 0.057  & 0.009 $\pm$ 0.008 & 4.956e-05  & 4.000e-05 & 3h 10m  \\
 7 & 0.085 $\pm$ 0.060  & 0.011 $\pm$ 0.009 & 5.877e-05  & 4.971e-05 & 3h 40m  \\
 9 & 0.096 $\pm$ 0.081  & 0.011 $\pm$ 0.009 & 7.011e-05  & 6.123e-05 & 4h 46m  \\[1ex] 
 \hline
\end{tabular}
\end{sc}
\vskip -0.1in
\end{table}

\subsection{Scaling Experiments for the Interacting System} \label{app:scaling_bosons_interact}

This section provides more details on experiments from \cref{sec:scaling_bosons_interact}, where we investigate the scaling of the DSM approach for the interacting bosons system. We compare the performance of our algorithm with a numerical solver based on the Crank–Nicolson method (we modified the \textsc{qmsolve} library to work for $d > 2$) and t-VMC method. Our method reveals favorable scaling capabilities in the problem dimension compared to the Crank–Nicolson method as shown in \cref{table:num_solver_scaling} and \cref{table:dsm_scaling}. 

\cref{fig:scaling_exps} shows generated density functions for our DSM method and t-VMC approach. The proposed DSM approach demonstrates robust performance, accurately following the ground truth and providing reasonable predictions for $d=3, 4, 5$ interacting bosons. In contrast, when utilizing the t-VMC in higher dimensions, we observe a deterioration in the quality of the results. This limitation is likely attributed to the inherent difficulty in accurately representing higher-order interactions with the ansatz employed in the t-VMC approach, as discussed in Section \ref{sec:interact_main}. Consequently, as the problem dimension grows, the lack of sufficient interaction terms in the ansatz and numerical instabilities in the solver become increasingly problematic, leading to artifacts in the density plots as time evolves. The relative error between the ground truth and predicted densities is 0.023 and 0.028 for the DSM and t-VMC approaches, respectively, in the 3d case. This trend persists in the 4d case, where the DSM's relative error is 0.073, compared to the t-VMC's higher relative error of 0.089. (when compared with a grid-based Crank-Nikolson solver with $N=60$ grid points in each dimension). While we do not have the baseline for $d=5$, we believe DSM predictions are still reasonable. Our findings indicate that the t-VMC method can perform reasonably for low-dimensional systems, but its performance degrades as the number of interacting particles increases. This highlights the need for a scalable and carefully designed ansatz representation capable of capturing the complex behavior of particles in high-dimensional quantum systems.

\begin{figure}[!ht]
\vskip 0.1in
    \centering
    \includegraphics[width=0.85\textwidth]{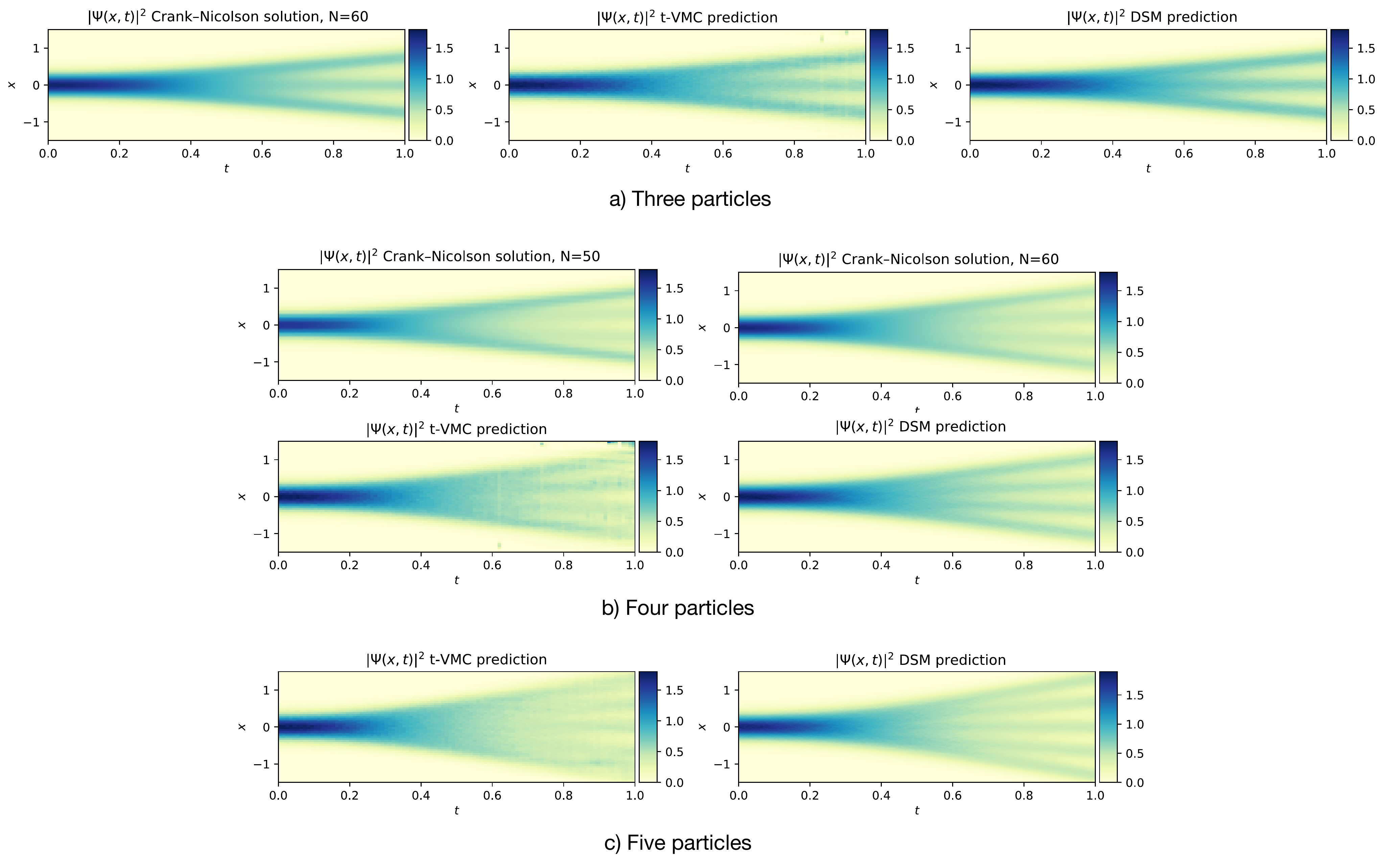}
    \caption{Probability density plots for different numbers of interacting particles $d$. For five particles, our computer system does not allow running the Crank-Nicolson solver.}
    \label{fig:scaling_exps}
    \vskip -0.1in
\end{figure}

As for the DSM implementation details, we fix hyperparameters and only change $d$: for example, the neural network size is 500, and the batch size is 100. We train our method until the average training loss becomes lower than a particular threshold (0.007). These numbers are reported for a GPU A40. The Crank-Nikolson method is run on the CPU.

\subsection{Sensitivity Analysis}

We investigate the impact of hyperparameters on the performance of our method for two systems: the 1d harmonic oscillator with $S_0(x) \equiv 0$ and two interacting bosons. Specifically, we explore different learning rates ($10^{-2}, 10^{-3}, 10^{-4}, 10^{-5}$) and hidden layer sizes (200, 300, 400, 500) for the neural network architectures detailed in Section \ref{sec:app_training}. All models are trained for an equal number of epochs across every hyper-parameter setting, and the results are presented in Figure \ref{fig:sens_exps}. For the two interacting bosons system, increasing the hidden layer size leads to lower error, although the difference between 300 and 500 neurons is marginal. In contrast, for the 1d harmonic oscillator, larger hidden dimensions result in slightly worse performance (which might be a sign of overfitting for this simple problem), but the degradation is not substantial. As for the learning rate, a higher value consistently yields poorer performance for both systems. A large learning rate can cause the weight updates to overshoot the optimal values, leading to instability and failure to converge to a good solution. Nevertheless, all models achieve reasonable performance, even with the highest learning rate of $10^{-2}$. Overall, according to the $\mathcal{E}_m(X_i)$ metric, our experiments demonstrate that our method exhibits robustness to varying hyper-parameter choices.

\begin{figure}[!ht]
\vskip 0.1in
    \centering
    \includegraphics[width=0.7\textwidth]{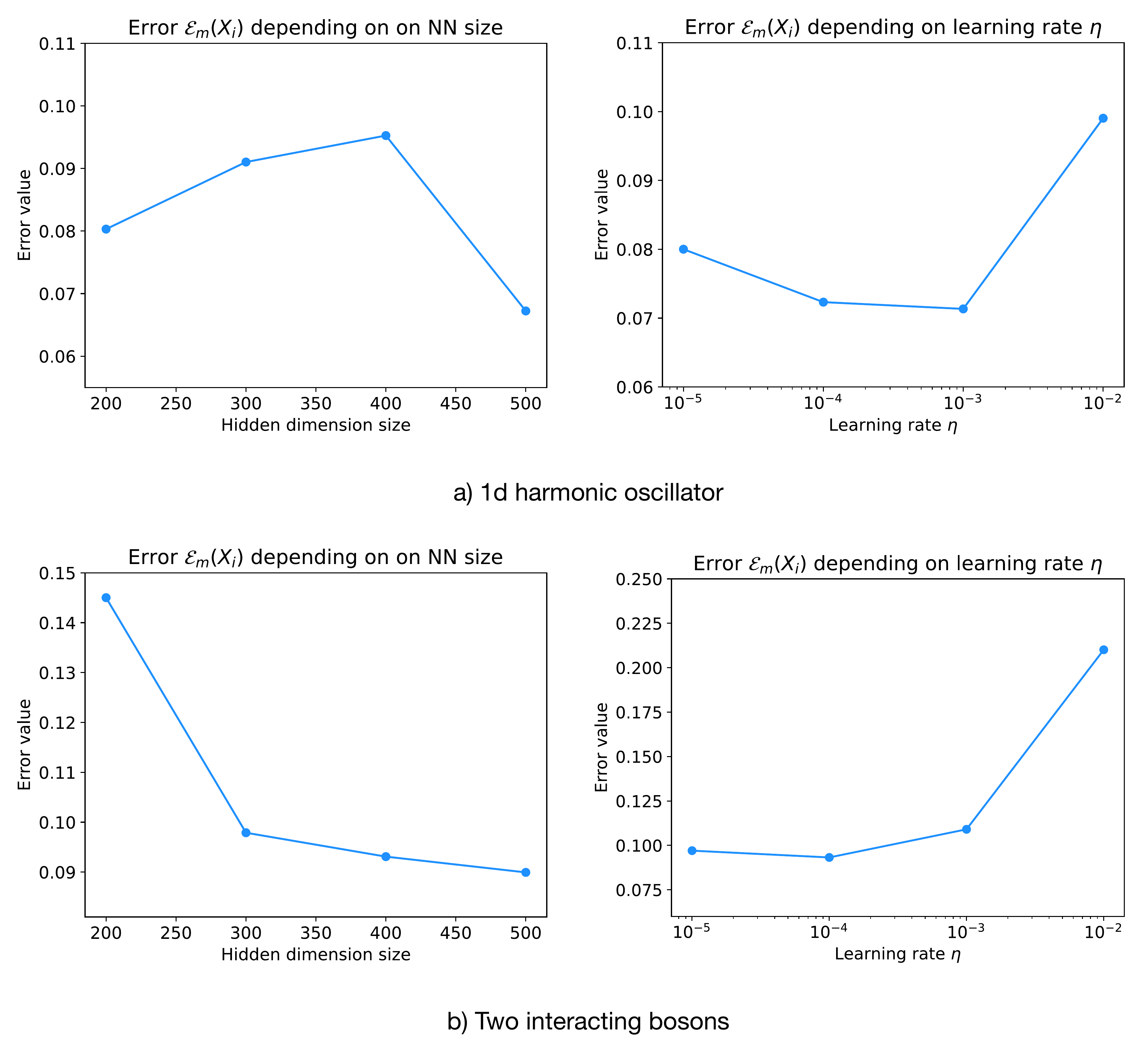}
    \caption{Sensitivity analysis of the neural network hyperparameters for the proposed method on two systems: (a) a 1D harmonic oscillator with $S_0(x) \equiv 0$, and (b) a system of two interacting bosons. The plots illustrate the impact of varying the hidden layer size and the learning rate on the model's performance, quantified by the $\mathcal{E}_m(X_i)$ error metric. }
    \label{fig:sens_exps}
    \vskip -0.1in
\end{figure}

\section{Stochastic Mechanics} \label{sec:stochastic_processes}
We show a derivation of the equations stochastic mechanics from the Schr\"odinger one. For full derivation and proof of equivalence, we refer the reader to the work of  \citet{NelsonOG}.

\subsection{Stochastic Mechanics Equations} \label{app:stochastic_mechanics}
Let's consider a polar decomposition of a wave function $\psi = \sqrt{\rho}e^{i S}$. Observe that for $\partial \in \{\partial_t, \partial_{x_i}\}$, we have
\begin{align*}
\partial \psi &= (\partial \sqrt{\rho}) e^{iS} + (i\partial S)\psi = \frac{\partial \rho}{2\sqrt{\rho}}e^{i S} + (i\partial S)\psi = \frac{1}{2}\frac{\partial \rho}{\rho} \sqrt{\rho} e^{iS} + (i\partial S)\psi = \big(\frac{1}{2}\partial \log \rho + i\partial S\big) \psi,\\
\partial^2 \psi &= \partial\Big(\big(\frac{1}{2}\partial \log \rho + i\partial S\big) \psi\Big) = \Big(\frac{1}{2}\partial^2 \log \rho + i\partial^2 S + \big(\frac{1}{2}\partial \log \rho + i\partial S\big)^2\Big) \psi
\end{align*}
\begin{align*}
\partial \psi &= (\partial \sqrt{\rho}) e^{iS} + (i\partial S)\psi = \frac{\partial \rho}{2\sqrt{\rho}}e^{i S} + (i\partial S)\psi = \frac{1}{2}\frac{\partial \rho}{\rho} \sqrt{\rho} e^{iS} + (i\partial S)\psi = \big(\frac{1}{2}\partial \log \rho + i\partial S\big) \psi,\\
\partial^2 \psi &= \partial\Big(\big(\frac{1}{2}\partial \log \rho + i\partial S\big) \psi\Big) = \Big(\frac{1}{2}\partial^2 \log \rho + i\partial^2 S + \big(\frac{1}{2}\partial \log \rho + i\partial S\big)^2\Big) \psi.
\end{align*}
Substituting it into the Schr\"odinger equation, we obtain the following:
\begin{align}
    i\hbar \big(\frac{1}{2}\partial_{t} \log \rho + i\partial_{t}S\big)\psi = -\frac{\hbar^2}{2m}\Big(\frac{1}{2}\Delta \log \rho + i\Delta S + \big\|\frac{1}{2}\nabla \log \rho + i\nabla S\big\|^2 \Big)\psi + V \psi. 
\end{align}
Dividing by $\psi$\footnote{We assume $\psi \ne 0$. Even though it may seem a restriction, we will solve the equations only for $X(t)$, which satisfy $\mathbb{P}\big(\psi(X(t), t) = 0\big) = 0$. So, we are allowed to assume this without loss of generality. The same cannot be said if we considered the PINN over a  grid to solve our equations.}, and separating real and imaginary parts, we obtain
\begin{align}
        -\hbar\partial_{t}S &= -\frac{\hbar^2}{2m}\Big(\frac{1}{2} \Delta \log \rho + \frac{1}{4}\| \log \rho\|^2 - \|\nabla S\|^2 \Big) + V, \\
        \frac{\hbar}{2} \partial_{t} \log \rho &= -\frac{\hbar^2}{2m}\big(\Delta S + \langle \log \rho,\nabla S\rangle \big).
\end{align}
Noting that $\Delta = \langle \nabla, \nabla\cdot\rangle$ and substituting $v = \frac{\hbar}{m}\nabla S, u = \frac{\hbar}{2m}  \log \rho$ to simplify, we obtain
\begin{align}\label{eq:phase_density_equations}
        m\frac{\hbar}{m}\partial_{t}S &= \frac{\hbar}{2m}\langle \nabla, u\rangle + \frac{1}{2}\|u\|^2 - \frac{1}{2}\|v\|^2 - V, \\
        \frac{\hbar}{2m} \partial_{t} \log \rho &= -\frac{\hbar}{2m}\langle \nabla, v\rangle - \langle u, v\rangle.
\end{align}
Finally, by taking $\nabla$ from both parts, noting that $\big[\nabla, \partial_t\big]= 0$ for scalar functions, and substituting $u, v$ again, we arrive at
\begin{align}\label{eq:nav-s}
    \partial_{t} v &= -\frac{1}{m} \nabla V + \langle u, \nabla\rangle u - \langle v, \nabla\rangle v + \frac{\hbar}{2m} \nabla\langle \nabla, u\rangle,\\
    \partial_{t} u &=  - \nabla \langle v, u\rangle - \frac{\hbar}{2m} \nabla \langle \nabla, v\rangle.
\end{align}
To get the initial conditions on the velocities of the process $v_{0} = v(x, 0)$ and $u_{0} = u(x, 0)$, we can refer to the equations that we used in the derivation
\begin{align}
    v(x, t) &= \frac{\hbar}{m}\nabla S(x, t), \\
    u(x, t) &= \frac{\hbar}{2m}  \nabla \log \rho(x, t)
\end{align}
So, we can get our initial conditions at $t=0$ on $v_{0}(x) = \frac{\hbar}{m}\nabla S(x, 0), u_0(x) = \nu \nabla \log \rho_{0}(x)$, where $\rho_{0}(x) = \rho(x, 0)$. 

For more detailed derivation and proof of equivalence of those two equations to the Schr\"odinger one, see \citet{NelsonOG, nelson2005mystery, guerra1995introduction}. Moreover, this equivalence holds for manifolds $\mathcal{M}$ with trivial second cohomology group as noted in \citet{about_line_bundles, inequavalency, about_single_valuedness}.

\subsection{Novel Equations of Stochastic Mechanics}\label{app:sm_differences}
We note that our equations differ from \citet{guerra1995introduction, NelsonOG}. In \citet{NelsonOG}, we see
\begin{subequations}\label{eq:nelson_eq}
    \begin{align} \label{eq:nelson_eq1}
    \partial_{t} v &= -\frac{1}{m} \nabla V + \langle u, \nabla\rangle u - \langle v, \nabla\rangle v + \frac{\hbar}{2m} \Delta u,
    \end{align}
\begin{align} \label{eq:nelson_eq2}
    \partial_{t} u &=  - \nabla \langle v, u\rangle - \frac{\hbar}{2m} \nabla \langle \nabla, v\rangle;
\end{align}
\end{subequations}
and in \citet{guerra1995introduction}, we see

\begin{subequations}\label{eq:guerra_eq}
\begin{align}\label{eq:guerra_eq1}
    \partial_{t} v &= -\frac{1}{m} \nabla V + \langle u, \nabla\rangle u - \langle v, \nabla\rangle v + \frac{\hbar}{2m} \Delta u,
\end{align}
\begin{align}\label{eq:guerra_eq2}
    \partial_{t} u &=  - \nabla \langle v, u\rangle - \frac{\hbar}{2m} \Delta v .
\end{align}
\end{subequations}
Note that our Equations \labelcref{eq:nav-s1}, \labelcref{eq:nav-s2} do not directly use the second-order Laplacian operator $\Delta$, as it appears for $u$ in \cref{eq:nelson_eq1} and $v$ in \cref{eq:guerra_eq2}. The discrepancy between Nelson's and Guerra's equations seems to occur because the work by \citet{nelson2005mystery} coversthe case of the multi-valued $S$, and thus does not assume that $\big[\Delta, \nabla\big] = 0$ to transform $\nabla \langle \nabla, v \rangle = \nabla \langle \nabla, \nabla S\rangle$ into $\Delta (\nabla S)$ to make the equations work for the case of a non-trivial cohomology group of $\mathcal{M}$. However, \citet{guerra1995introduction} does employ $\Delta (\nabla S)$ in their formulation. Naively computing the Laplacian $\Delta$ of $u$ or $v$ with autograd tools requires $\mathcal{O}(d^3)$ operations as it requires computing the full Hessian $\nabla^2$. To reduce the computational complexity, we treat $\log \rho$ as a potentially multi-valued function, aiming to achieve a lower computational time of $\mathcal{O}(d^2)$ in the dimension $d$.
Generally, we cannot swap $\Delta$ with $\nabla \langle \nabla, \cdot\rangle$ unless the solutions of the equation can be represented as full gradients of some function. This condition holds for stochastic mechanical equations but not for the Shr\"odinger one. 

We derive equations different from both works and provide insights into why there are four different equivalent sets of equations (by changing $\Delta$ with $\nabla \langle \nabla, \cdot\rangle$ in both equations independently). From a numerical perspective, it is more beneficial to avoid Laplacian calculations. However, we notice that inference using equations from \citet{NelsonOG} converges faster by iterations to the true $u,v$ compared to our version. It comes at the cost of a severe slowdown in each iteration for $d \gg 1$, which diminishes the benefit since the overall training time to get comparable results decreases significantly.

\subsection{Diffusion Processes of Stochastic Mechanics}

Let's consider an arbitrary Ito diffusion process:
\begin{align}\label{eq:Ito-diffusion}
    \mathrm{d}X(t) &= b(X(t), t)\mathrm{d}t + \sigma(X(t), t) \mathrm{d}\overset{\rightarrow}{W},\\
    X(0)&\sim \rho_{0},
\end{align}
where $W(t)\in\mathbb{R}^d$ is the standard Wiener process, $b:\mathbb{R}^d\times[0,T]\rightarrow \mathbb{R}^d$ is the drift function, and $\sigma:\mathbb{R}^d\times [0, T]\rightarrow \mathbb{R}^{d\times d}$ is a symmetric positive definite matrix-valued function called a diffusion coefficient. Essentially, $X(t)$ samples from $\rho_{X} = \mathrm{Law}(X(t))$ for each $t\in [0, T]$. Thus, we may wonder how to define $b$ and $\sigma$ to ensure $\rho_{X} = |\psi|^2$.  

There is the forward Kolmogorov equation for the density $\rho_{X}$ associated with this diffusion process:
\begin{align}
    \partial_{t} \rho_{X} = \langle \nabla, b\rho_{X}\rangle + \frac{1}{2}\mathrm{Tr}\big(\nabla^2\cdot(\sigma\sigma^T\rho_{X})\big).
\end{align}
Moreover, the diffusion process is time-reversible. This leads to the backward Kolmogorov equation:
\begin{align}
    \partial_{t} \rho_{X} = \langle \nabla, b^{*}\rho_{X}\rangle - \frac{1}{2}\mathrm{Tr}\big(\nabla^2\cdot(\sigma\sigma^T\rho_{X})\big),
\end{align}
where $b^{*}_i = b_i - \rho_{X}^{-1} \langle \nabla, \sigma\sigma^T e_{i} \rho_{X}\rangle $ with $e_{ij} = \delta_{ij}$ for $j\in \{1,\ldots, d\}$. Summing up those two equations, we obtain the following:
\begin{equation}\label{eq:continuity_diffusion}
    \partial_{t} \rho_{X} = \langle \nabla, v\rho_{X}\rangle,
\end{equation}
where $v = \displaystyle \frac{b + b^{*}}{2}$ is so called probability current. This is the continuity equation for the Ito diffusion process from \cref{eq:Ito-diffusion}. We refer to \citet{ANDERSON1982313} for details. We note that the same \cref{eq:continuity_diffusion} can be obtained with an arbitrary non-singular $\sigma(x, t)$ as long as $v=v(x, t)$ remains fixed.
\begin{proposition}
    Consider arbitrary $\nu > 0$, denote $\rho = |\psi|^2$ and consider decomposition $\psi = \sqrt{\rho}e^{iS}$. Then the following process $X(t)$:
    \begin{align}
        \mathrm{d}X(t) &= \big(\nabla S(X(t), t) + \frac{\nu \hbar}{2m} \nabla \log \rho(X(t), t)\big)\mathrm{d}t + \sqrt{\frac{\nu \hbar}{m}} \mathrm{d}\overset{\rightarrow}{W},\\
        X(0) &\sim |\psi_{0}|^2,
    \end{align}
satisfies $\mathrm{Law}(X(t)) = |\psi|^2$ for any $t > 0$.
\end{proposition}
\begin{proof}
We want to show that by choosing appropriately $b, b_*$, we can ensure that $\rho_{X} = |\psi|^2$. Let's consider the Schr\"odinger equation once again:
\begin{align}
    i \hbar \partial_{t} \psi &= (-\frac{\hbar^2}{2m}\Delta + V)\psi,\\
    \psi(\cdot, 0) &= \psi_{0}
\end{align}
where $\Delta = \mathrm{Tr}(\nabla^2) = \sum_{i=1}^d \frac{\partial^2}{\partial x_{i}^2}$ is the Laplace operator. The second cohomology is trivial in this case. So, we can assume that $\psi = \sqrt{\rho}e^{iS}$ with $S(x, t)$ is a single-valued function. 

By defining the drift $v = \displaystyle \frac{\hbar}{m}\nabla S$, we can derive quantum mechanics continuity equation on density $\rho$:
\begin{align}
    \partial_{t} \rho = \langle \nabla, v\rho\rangle,\\
    \rho(\cdot, 0) = \big|\psi_{0}\big|^2.
\end{align}
This immediately tells us what should be initial distribution $\rho_{0}$ and $\frac{b+b^*}{2}$ for the Ito diffusion process from \cref{eq:Ito-diffusion}.

For now, the only missing parts for obtaining the diffusion process from the quantum mechanics continuity equation are to identify the term $\frac{b-b^{*}}{2}$ and the diffusion coefficient $\sigma$. Both of them should be related as $(b-b^{*})_{i} = \rho^{-1} \langle \nabla, \sigma\sigma^T e_{i} \rho\rangle$. Thus, we can pick $\sigma \propto I_{d}$ to simplify the equations. Nevertheless, our results can be extended to any non-trivial diffusion coefficient. Therefore, by defining $u(x, t) = \displaystyle \frac{\hbar}{2m} \nabla \log \rho(x, t)$ and using arbitrary $\nu > 0$ we derive
\begin{equation}
    \partial_{t} \rho =  \langle\nabla, (v+\nu u)\rho\rangle + \frac{\nu \hbar}{2m} \Delta \rho.
\end{equation}
Thus, we can sample from $\rho_{X}(x, t) \equiv \rho(x, t)$ using the diffusion process with $b(x,t) = v(x,t) + \nu u(x,t)$ and $\sigma(x,t) \equiv \frac{\nu \hbar}{m} I_{d}$:
\begin{align} 
    \mathrm{d}X(t) &= (v(X(t), t)+\nu u(X(t), t))\mathrm{d}t + \sqrt{\frac{\nu \hbar}{m} }\mathrm{d}\overset{\rightarrow}{W},\\
    X(0)&\sim \big|\psi_{0}\big|^2.
\end{align}
\end{proof}
To obtain numerical samples from the diffusion, one can use any numerical integrator, for example, the Euler-Maruyama integrator \citep{kloeden1992stochastic}:
\begin{align}
    X_{i+1} &= X_{i} + (v(X_{i}, t_{i})+\nu u(X_{i}, t_{i}))\epsilon + \sqrt{\frac{\nu \hbar}{m}  \epsilon} \mathcal{N}(0, I_{d}), \\
    X_{0}&\sim \big|\psi_{0}\big|^2,
\end{align}
where $\epsilon > 0$ is a step size, $0 \le i < \frac{T}{\epsilon}$. We consider this type of integrator in our work. However, integrators of higher order, e.g., Runge-Kutta family of integrators \citep{kloeden1992stochastic}, can achieve the same integration error with larger $\epsilon > 0$; this approach is out of the scope of our work.

\subsection{Interpolation between Bohmian and Nelsonian pictures}\label{sec:app_interpolation}
We also differ from \citet{NelsonOG} since we define $u$ without $\nu$. We bring it into the picture separately as a multiplicative factor:
\begin{align}
    \mathrm{d}X(t) &= (v(X(t), t)+\nu u(X(t), t))\mathrm{d}t + \sqrt{\frac{\nu \hbar}{m} }\mathrm{d}\overset{\rightarrow}{W},\\
    X(0)&\sim \big|\psi_{0}\big|^2
\end{align}
This trick allows us to recover Nelson's diffusion when $\nu = 1$:
\begin{align}
    \mathrm{d}X(t) &= (v(X(t), t)+u(X(t), t))\mathrm{d}t + \sqrt{\frac{\hbar}{m} }\mathrm{d}\overset{\rightarrow}{W},\\
    X(0)&\sim \big|\psi_{0}\big|^2
\end{align}

For cases of $|\psi_{0}|^2 > 0$ everywhere, e.g., if the initial conditions are Gaussian but not singular like $\delta_{x_{0}}$, we can actually set $\nu = 0$ to obtain a deterministic flow:
\begin{align}
    \mathrm{d}X(t) &= v(X(t), t)\mathrm{d}t,\\
    X(0)&\sim \big|\psi_{0}\big|^2.
\end{align}
This is the guiding equation in Bohr's pilot-wave theory \citep{PhysRev.85.166}. The major drawback of using Bohr's interpretation is that $\rho_{X}$ may not equal $\rho = |\psi|^2$, a phenomenon known as quantum non-equilibrium \citep{colin2010quantum}. Though, under certain mild conditions \citep{boffi2023probability} (one of which is $|\psi_{0}|^2 > 0$ everywhere) time marginals of such deterministic process $X(t)$ satisfy $\mathrm{Law}(X(t)) = \rho$ for each $t\in [0, T]$. As with the SDE case, it is unlikely that those trajectories are ``true'' trajectories. It only matters that their time marginals coincide with true quantum mechanical densities.  

\subsection{Computational Complexity}\label{sec:app_complexity}

\begin{proposition}[\cref{prop:alg}]
The algorithmic complexity w.r.t. $d$ of computing differential operators from Equations (\ref{eq:diff_operators}), (\ref{eq:diff_operators2}) for velocities $u, v$ is $\mathcal{O}(d^2)$.
\end{proposition}
\begin{proof}
    Computing a forward pass of $u_{\theta}, v_{\theta}$ scales as $\mathcal{O}(d)$ by their design. What we need is to prove that Equations (\ref{eq:diff_operators}), (\ref{eq:diff_operators2}) can be computed in $\mathcal{O}(d^2)$. We have two kinds of operators there: $\langle \nabla \cdot, \cdot\rangle$ and $\nabla \langle \nabla, \cdot\rangle$.

    The first operator, $\langle \nabla \cdot, \cdot\rangle$, is a Jacobian-vector product. There exists an algorithm to estimate it with linear complexity, assuming the forward pass has linear complexity, as shown by \citet{doi:10.1137/1.9780898717761}.

    For the second operator, the gradient operator $\nabla$ scales linearly with the problem dimension $d$. To estimate the divergence operator $\langle \nabla, \cdot\rangle$, we need to run automatic differentiation $d$ times to obtain the full Jacobian and take its trace. This leads to a quadratic computational complexity of $\mathcal{O}(d^2)$ in the problem dimension. It is better than the naive computation of the Laplace operator $\Delta$, which has a complexity of $\mathcal{O}(d^3)$ due to computing the full Hessian for each component of $u_{\theta}$ or $v_{\theta}$.
\end{proof}

We assume that one of the dimensions when evaluating the $d$-dimensional functions involved in our method is parallelized by modern deep learning libraries. It means that empirically, we can see a linear $\mathcal{O}(d)$ scaling instead of the theoretical $\mathcal{O}(d^2)$ complexity.

\section{On Strong Convergence}\label{app:theory}
Let's consider a standard Wiener processes $\overset{\rightarrow}{W^{X}}, \overset{\rightarrow}{W^{Y}}$ in $\mathbb{R}^{d}$ and define $\overset{\rightarrow}{\mathcal{F}_{t}}$ as a filtration generated by $\Big\{\big(\overset{\rightarrow}{W^{X}}(t'), \overset{\rightarrow}{W^{Y}}(t)\big): t' \le t\Big\}$. Let $\overset{\leftarrow}{\mathcal{F}_{t}}$ be a filtration generated by all events $\Big\{\big(\overset{\rightarrow}{W^{X}}(t'), \overset{\rightarrow}{W^{Y}}(t)\big): t' \ge t\Big\}$. 

Assume that $ u, v, \widetilde{u}, \widetilde{v} \in C^{2, 1}(\mathbb{R}^d \times [0, T];\mathbb{R}^d) \cap C^{1, 0}_b(\mathbb{R}^d\times [0, T]; \mathbb{R}^d)$, where $C_{b}^{p, k}$ is a class of continuously differentiable functions with uniformly bounded $p$-th derivative in a coordinate $x$ and $k$-th continuously differentiable in $t$, $C^{p, k}$ analogously but without requiring bounded derivative. For $f:\mathbb{R}^d \times [0, T]\rightarrow \mathbb{R}^k$ define $\|f\|_{\infty} = \mathrm{ess\,sup}_{t\in [0, T], x\in \mathbb{R}^d} \|f(x, t)\|$ and $\|\nabla f\|_{\infty} = \mathrm{ess\,sup}_{t\in [0, T], x\in \mathbb{R}^d} \|\nabla f(x, t)\|_{op}$ where $\|\cdot\|_{op}$ denotes operator norm. 
\begin{align}
\mathrm{d}X(t) &= (\widetilde{v}(X(t), t) + \widetilde{u}(X(t), t)\big)\mathrm{d}t + \sqrt{\frac{\hbar}{m}}\mathrm{d}\overset{\rightarrow}{W^{X}}(t),\\
\mathrm{d}Y(t) &= (v(Y(t), t) + u(Y(t), t)\big)\mathrm{d}t + \sqrt{\frac{\hbar}{m}}\mathrm{d}\overset{\rightarrow}{W^{Y}}(t),\\
X(0)&\sim |\psi_{0}|^2,\\
Y(0)&=X(0),
\end{align}
where $u, v$ are true solutions to equations \labelcref{eq:nelson_eq}. We have that $p_{Y}(\cdot, t) = \big|\psi(\cdot, t)\big|^2$ $\forall t$ where $p_{Y}$ is density of the process $Y(t)$. We have not specified yet quadratic covariation of those two processes $\frac{\mathrm{d}\big[\overset{\rightarrow}{W^{X}}, \overset{\rightarrow}{W^{Y}}\big]_{t}}{\mathrm{d}t} = \lim_{\mathrm{d}t\rightarrow 0_+} \mathbb{E}\Big(\frac{\big(\overset{\rightarrow}{W^{X}}(t+\mathrm{d}t) - \overset{\rightarrow}{W^{X}}(t)\big)\big(\overset{\rightarrow}{W^{Y}}(t+\mathrm{d}t) - \overset{\rightarrow}{W^{Y}}(t)\big)}{\mathrm{d}t}\Big|\overset{\rightarrow}{\mathcal{F}_t}\Big)$. We will though specify it as $\mathrm{d}\big[\overset{\rightarrow}{W^{X}}, \overset{\rightarrow}{W^{Y}}\big]_{t} = I_d\mathrm{d}t$, and it allows to cancel some terms appearing in the equations. As for now, we will derive all results in the most general setting. 

Let's define our loss functions:
\begin{equation} 
L_{1} (\widetilde{v}, \widetilde{u}) =  \int_{0}^{T}\mathbb{E}^{X}\big\| \partial_{t} \widetilde{u}(X(t), t) -\mathcal{D}_{u}[\widetilde{v}, \widetilde{u}, x, t]\big\|^2\mathrm{d}t,
\end{equation}
\begin{equation}
L_{2} (\widetilde{v}, \widetilde{u}) =   \int_{0}^{T}\mathbb{E}^{X} \big\| \partial_{t} \widetilde{v}(X(t),t) -  \mathcal{D}_{v}[\widetilde{v}, \widetilde{u}, X(t), t] \big\|^2\mathrm{d}t,
\end{equation}
\begin{equation}
L_{3}(\widetilde{u}, \widetilde{v}) = \mathbb{E}^{X}\|\widetilde{u}(X(0), 0) - u(X(0), 0)\|^2 
\end{equation}
\begin{equation}
L_{4}(\widetilde{u}, \widetilde{v}) = \mathbb{E}^{X} \|\widetilde{v}(X(0), 0) - v(X(0), 0)\|^2
\end{equation}
Our goal is to show that for some constants $w_i > 0$, there is natural bound $\mathrm{sup}_{0\le t\le T}\mathbb{E}\|X(t) - Y(t)\|^2 \le \sum w_i L_i(\widetilde{v}, \widetilde{u})$.

\subsection{Stochastic Processes}
Consider a general Itô SDE defined using a drift process $F(t)$ and a covariance process $G(t)$, both predictable with respect to forward and backward flirtations $\overset{\leftarrow}{\mathcal{F}_{t}}$ and $\overset{\rightarrow}{\mathcal{F}_{t}}$:
\begin{align}\label{eq:main_sde_app}
\mathrm{d}Z(t) &= F(t)\mathrm{d}t + G(t)\mathrm{d}\overset{\rightarrow}{W},\\
Z(0)&\sim \rho_{0}. \nonumber
\end{align}
Moreover, assume $\displaystyle \big[Z(t), Z(t)\big]_{t} = \mathbb{E}\int_{0}^{t} G^{T}G(t)\mathrm{d}t < \infty$ , $ \displaystyle \mathbb{E}\int_{0}^{t} \|F(t)\|^2 \mathrm{d}t < \infty$.
We denote by $\mathbb{P}^{Z}_t = \mathbb{P}(Z(t) \in \cdot)$ a law of the process $Z(t)$. Let's define a (extended) forward generate of the process as the linear operator satisfying
\begin{align}
\overset{\rightarrow}{M^{f}}(t) = f(Z(t), t) - f(Z(0), 0) - \int_{0}^{t} \overset{\rightarrow}{\mathcal{L}^{X}} f(Z(t), t) \text{ is }\overset{\rightarrow}{\mathcal{F}_{t}}\text{-martingale}.
\end{align}
Such an operator is uniquely defined and is called a forward generator associated with the process $Z_{t}$. 
Similarly, we define a (extended) backward generator $\overset{\leftarrow}{\mathcal{L}^{X}}$ as linear operator satisfying:
\begin{align}
\overset{\leftarrow}{M^{f}}(t) = f(Z(t), t) - f(Z(0), 0) - \int_{0}^{t} \overset{\leftarrow}{\mathcal{L}^{X}} f(Z(t), t) \text{ is }\overset{\leftarrow}{\mathcal{F}_t}\text{-martingale}
\end{align}

For more information on the properties of generators, we refer to \citet{baldi2017stochastic}. 
\begin{lemma}(Itô Lemma, \citep[Theorem 8.1 and Remark 9.1]{baldi2017stochastic} ) \label{lemma:ito}
\begin{align}
\overset{\rightarrow}{\mathcal{L}^{Z}}f(x, t) = \partial_{t} f(x, t) + \langle \nabla f(x, t), F(t)\rangle + \frac{\hbar}{2m}\mathrm{Tr}\big(G^{T}(t)\nabla^2 f(x, s)G(t)\big).
\end{align}
\end{lemma}
\begin{lemma} \label{lemma:reverse_generator}
Let $p_{Z}(x, t) = \frac{\mathrm{d}\mathbb{P}^{Z}_t}{\mathrm{d} x}$ be the density of the process with respect to standard Lebesgue measure on $\mathbb{R}^d$. Then
\begin{align}
\overset{\leftarrow}{\mathcal{L}^{Z}}f(x, t) = \partial_{t} f(x, t) + \langle \nabla f(x, t), F(t) - \frac{\hbar}{m}\nabla \log p_{Z}(x, t)\rangle - \frac{1}{2}\mathrm{Tr}\big(G^T(t)\nabla^2 f(x, s)G(t)\big).
\end{align}
\end{lemma}
\begin{proof}
We have the following operator identities:
\begin{align*}
\overset{\leftarrow}{\mathcal{L}^{Z}} = \big(\overset{\rightarrow}{\mathcal{L}^{Z}}\big)^{*} = p_{Z}^{-1}\big(\overset{\rightarrow}{\mathcal{L}^{X}}\big)^{\dagger}p_{Z}
\end{align*}
where $\mathcal{A}^{*}$ is adjoint operator in $L_2(\mathbb{R}^d\times [0, T], \mathbb{P}^{Z}\otimes \mathrm{d}t )$ and $\mathcal{A}^{\dagger}$ is adjoint in $L_2(\mathbb{R}^d\times [0, T], \mathrm{d}x\otimes \mathrm{d}t)$. Using Itô lemma \ref{lemma:ito} and grouping all terms yields the statement.
\end{proof}
\begin{lemma}\label{lemma:commutative}
The following identity holds for any process $Z(t)$:
\begin{align}
\overset{\rightarrow}{\mathcal{L}^{Z}} \overset{\leftarrow}{\mathcal{L}^{Z}} x = \overset{\leftarrow}{\mathcal{L}^{Z}} \overset{\rightarrow}{\mathcal{L}^{Z}} x.
\end{align}
\end{lemma}
\begin{proof}
One needs to recognize that \cref{eq:continuity_diffusion} is the difference between two types of generators, we automatically have the following identity that holds for any process $Z$.
\end{proof}
\begin{lemma}(Nelson Lemma, \cite{nelson2020dynamical}) \label{lemma:nelson}
\begin{align}
\mathbb{E}^{Z}&\Big(f(Z(t), t)g(Z(t), t) - f(Z(0), t)g(Z(0), t)\Big) \\
&= \mathbb{E}^{Z}\int_{0}^{t}\Big(\overset{\rightarrow}{\mathcal{L}^{Z}}f(Z(s), t)g(Z(s), t) + f(Z(s), t)\overset{\leftarrow}{\mathcal{L}^{Z}}g(Z(s), s)\Big)\mathrm{ds}
\end{align}
\end{lemma}
\begin{lemma}\label{G5}
It holds that:
\begin{align}
\mathbb{E}^{Z}&\Big(\|Z(t)\|^2 - \|Z(0)\|^2\Big) \\
&= \int_{0}^{t}\mathbb{E}^{Z}\Big(2\langle\overset{\leftarrow}{\mathcal{L}^{Z}} Z(0), Z(s)\rangle + 2\int_{0}^{s} \langle\overset{\leftarrow}{\mathcal{L}^{Z}} \overset{\rightarrow}{\mathcal{L}^{Z}} Z(z), Z(s)\rangle\mathrm{d}z \Big)\mathrm{d}s + \big[Z(t), Z(t)\big]_{t}
\end{align}
\end{lemma}
\begin{proof}
By using Itô Lemma \ref{lemma:ito} for $f(x) = \|x\|^2$ and noting that $\overset{\rightarrow}{\mathcal{L}^{Z}}Z(t) = F(t)$ we immediately obtain:
\begin{align*}
\mathbb{E}^{Z}(\|Z(t)\|^2 - \|Z(0)\|^2) &= \int_{0}^{t} \mathbb{E}\Big(2\langle \overset{\rightarrow}{\mathcal{L}^{Z}} Z(s), Z(s)\rangle+\mathrm{Tr}\big(G^{T}G(t)\big)\Big)\mathrm{d}s 
\end{align*}
Let's deal with the term $\int_{0}^{t} \langle \overset{\rightarrow}{\mathcal{L}^{Z}}Z(s), Z(s)\rangle\mathrm{d}s$. We have the following observation: $\overset{\rightarrow}{M^{F}}(z) = \overset{\leftarrow}{\mathcal{L}^{Z}}Z(s) - \overset{\leftarrow}{\mathcal{L}^{Z}}Z(0) - \int_{0}^{s}\overset{\leftarrow}{\mathcal{L}^{Z}}\overset{\rightarrow}{\mathcal{L}^{Z}}Z(z)\mathrm{d}z$ is $\overset{\leftarrow}{\mathcal{F}}_{s}$-martingale, thus
\begin{align*}
\int_{0}^{t} \langle &\overset{\rightarrow}{\mathcal{L}^{Z}}Z(s), Z(s)\rangle\mathrm{d} s
=\int_{0}^{t} \langle \overset{\leftarrow}{\mathcal{L}^{Z}}Z(0) + \int_{0}^{s}\big(\overset{\leftarrow}{\mathcal{L}^{Z}}\overset{\rightarrow}{\mathcal{L}^{Z}}Z(z)
+ \ \overset{\leftarrow}{M^{F}}(z)\big)\mathrm{d}z, Z(s)\rangle\mathrm{d}s,
\end{align*}
The process $\overset{\rightarrow}{A}(s', s) = \int_{s'}^{s} \langle \overset{\leftarrow}{M^{F}}(z),Z(s)\rangle\mathrm{d}z$ is again $\mathcal{F}^{\leftarrow}_{s'}$-martingale for $s' \le s$, which implies that $\mathbb{E}^{Z}\overset{\rightarrow}{A}(0, s) = 0$.
Noting that $\mathbb{E}^{Z}\int_{0}^{t} \mathrm{Tr}\big(G^{T}(t)G(t)\big)\mathrm{d}t = \big[Z(t), Z(t)\big]_{t}$ yields the lemma.
\end{proof}
\subsection{Adjoint Processes}
Consider process $X'(t)$ defined through time-reversed SDE:
\begin{align}
    \mathrm{d}X'(t) &= (\widetilde{v}(X'(t), t) + \widetilde{u}(X'(t), t)\big)\mathrm{d}t + \sqrt{\frac{\hbar}{2m}}\mathrm{d}\overset{\leftarrow}{W^{X}}(t).
\end{align}
We call such process as adjoint to the process $X$. Lemma \ref{lemma:commutative} can be generalized to the pair of adjoint processes $(X, X')$ in the following way and will be instrumental in proving our results.
\begin{lemma}\label{lemma:magic} For any pair of processes $X(t), X'(t)$ such that the forward drift of $X$ is of form $\widetilde{v} + \widetilde{u}$ and backward drift of $X'$ is $\widetilde{v} - \widetilde{u}$:
\begin{align}
\overset{\rightarrow}{\mathcal{L}^{X}} \overset{\leftarrow}{\mathcal{L}^{X'}} x - \overset{\leftarrow}{\mathcal{L}^{X'}} \overset{\rightarrow}{\mathcal{L}^{X}} x =  \overset{\leftarrow}{\mathcal{L}^{X'}} \overset{\leftarrow}{\mathcal{L}^{X'}} x - \overset{\rightarrow}{\mathcal{L}^{X}} \overset{\rightarrow}{\mathcal{L}^{X}} x.
\end{align}
with both sides being equal to $0$ if and only if $X'$ is time reversal of $X$.
\end{lemma}
\begin{proof}
Manual substitution of explicit forms of generators and drifts yields \cref{eq:nav-s2} for both cases. This equation is zero only if $\widetilde{u} = \frac{\hbar}{2m}\nabla \log p_{X}$
\end{proof}
\begin{lemma}\label{lemma:bound-density-diff}
The following bound holds:
\begin{align}
\Big\|\big(\overset{\rightarrow}{\mathcal{L}^{X}}+\overset{\leftarrow}{\mathcal{L}^{X}}\big) (\widetilde{u} - \frac{\hbar}{2m}\nabla \log p_{X})\| \le \Big\|\overset{\rightarrow}{\mathcal{L}^{X}} \overset{\leftarrow}{\mathcal{L}^{X'}} x - \overset{\leftarrow}{\mathcal{L}^{X'}} \overset{\rightarrow}{\mathcal{L}^{X}} x\Big\| + 2 \|\nabla \widetilde{v}\|_{\infty}\big\|\widetilde{u} - \frac{\hbar}{2m} \nabla \log p_{X}\big\|.
\end{align}
\end{lemma}
\begin{proof}
First, using Lemma \ref{lemma:magic} we obtain:
\begin{align}
\overset{\rightarrow}{\mathcal{L}^{X}}& \overset{\leftarrow}{\mathcal{L}^{X}} x - \overset{\leftarrow}{\mathcal{L}^{X}}\overset{\rightarrow}{\mathcal{L}^{X}} x = 0\\
\iff \overset{\rightarrow}{\mathcal{L}^{X}}& \big(\widetilde{v}+\widetilde{u} - \frac{\hbar}{m}\nabla \log p_{X}\big) - \overset{\leftarrow}{\mathcal{L}^{X}}\big(\widetilde{v} + \widetilde{u}\big) = 0\\
\iff \overset{\rightarrow}{\mathcal{L}^{X}}& \big((\widetilde{v}-\widetilde{u}) + (2\widetilde{u} - \frac{\hbar}{m}\nabla \log p_{X})\big) - \overset{\leftarrow}{\mathcal{L}^{X}}\big(\widetilde{v} + \widetilde{u}\big) = 0\\
\iff \overset{\rightarrow}{\mathcal{L}^{X}}& \big((\widetilde{v}-\widetilde{u}) + (2\widetilde{u} - \frac{\hbar}{m}\nabla \log p_{X})\big) - \overset{\leftarrow}{\mathcal{L}^{X'}}\big(\widetilde{v} + \widetilde{u}\big) + \Big(\overset{\leftarrow}{\mathcal{L}^{X'}}\big(\widetilde{v} + \widetilde{u}\big) - \overset{\leftarrow}{\mathcal{L}^{X}}\big(\widetilde{v} + \widetilde{u}\big)\Big) = 0\\
\iff \overset{\rightarrow}{\mathcal{L}^{X}}&\big(2\widetilde{u} - \frac{\hbar}{m}\nabla \log p_{X}\big)+\overset{\rightarrow}{\mathcal{L}^{X}}(\widetilde{v}-\widetilde{u}) - \overset{\leftarrow}{\mathcal{L}^{X'}}\big(\widetilde{v} + \widetilde{u}\big) + \Big(\overset{\leftarrow}{\mathcal{L}^{X'}}\big(\widetilde{v} + \widetilde{u}\big) - \overset{\leftarrow}{\mathcal{L}^{X}}\big(\widetilde{v} + \widetilde{u}\big)\Big) = 0.
\end{align}
Then, we note that:
\begin{align}\label{eq74}
\overset{\leftarrow}{\mathcal{L}^{X'}}\big(\widetilde{v} + \widetilde{u}\big) - \overset{\leftarrow}{\mathcal{L}^{X}}\big(\widetilde{v} + \widetilde{u}\big) = \langle \frac{\hbar}{m}\nabla \log p_{X} - 2 \widetilde{u}, \nabla(\widetilde{v} + \widetilde{u})\rangle.
\end{align}
This leads us to the following identity:
\begin{align*}
\overset{\rightarrow}{\mathcal{L}^{X}}&\big(2\widetilde{u} - \frac{\hbar}{m}\nabla \log p_{X}\big)+\overset{\rightarrow}{\mathcal{L}^{X}}(\widetilde{v}-\widetilde{u}) - \overset{\leftarrow}{\mathcal{L}^{X'}}\big(\widetilde{v} + \widetilde{u}\big)  + \langle \frac{\hbar}{m}\nabla \log p_{X} - 2 \widetilde{u}, \nabla(\widetilde{v} + \widetilde{u})\rangle  = 0\\
\iff \overset{\rightarrow}{\mathcal{L}^{X}}&\big(2\widetilde{u} - \frac{\hbar}{m}\nabla \log p_{X}\big)+\overset{\rightarrow}{\mathcal{L}^{X}}\overset{\leftarrow}{\mathcal{L}^{X'}}x - \overset{\leftarrow}{\mathcal{L}^{X'}}\overset{\rightarrow}{\mathcal{L}^{X}}x  + \langle \frac{\hbar}{m}\nabla \log p_{X} - 2 \widetilde{u}, \nabla(\widetilde{v} + \widetilde{u})\rangle  = 0.
\end{align*}    
Again by using Lemma \ref{lemma:magic} to time-reversal $X'$ we obtain:
\begin{align}
\overset{\leftarrow}{\mathcal{L}^{X}}& \overset{\leftarrow}{\mathcal{L}^{X}} x - \overset{\rightarrow}{\mathcal{L}^{X}}\overset{\rightarrow}{\mathcal{L}^{X}} x = 0\\
\iff \overset{\leftarrow}{\mathcal{L}^{X}}& \big(\widetilde{v}+\widetilde{u} - \frac{\hbar}{m}\nabla \log p_{X}\big) - \overset{\rightarrow}{\mathcal{L}^{X}}\big(\widetilde{v} + \widetilde{u}\big) = 0\\
\iff \overset{\leftarrow}{\mathcal{L}^{X}}& \big((\widetilde{v}-\widetilde{u}) + (2\widetilde{u} - \frac{\hbar}{m}\nabla \log p_{X})\big) - \overset{\rightarrow}{\mathcal{L}^{X}}\big(\widetilde{v} + \widetilde{u}\big) = 0\\
\iff \overset{\leftarrow}{\mathcal{L}^{X'}}&\big(\widetilde{v} - \widetilde{u}\big) + \overset{\leftarrow}{\mathcal{L}^{X}} \big(2\widetilde{u} - \frac{\hbar}{m}\nabla \log p_{X}\big) - \overset{\rightarrow}{\mathcal{L}^{X}}\big(\widetilde{v} + \widetilde{u}\big) + \Big(\overset{\leftarrow}{\mathcal{L}^{X}}\big(\widetilde{v} - \widetilde{u}\big) - \overset{\leftarrow}{\mathcal{L}^{X'}}\big(\widetilde{v} - \widetilde{u}\big)\Big) = 0\\
\iff \overset{\leftarrow}{\mathcal{L}^{X}}& \big(2\widetilde{u} - \frac{\hbar}{m}\nabla \log p_{X}\big) + \overset{\leftarrow}{\mathcal{L}^{X'}}\big(\widetilde{v} - \widetilde{u}\big) - \overset{\rightarrow}{\mathcal{L}^{X}}\big(\widetilde{v} + \widetilde{u}\big) - \langle \frac{\hbar}{m}\nabla \log p_{X} - 2 \widetilde{u}, \nabla(\widetilde{v} - \widetilde{u})\rangle = 0\\
\iff \overset{\leftarrow}{\mathcal{L}^{X}}& \big(2\widetilde{u} - \frac{\hbar}{m}\nabla \log p_{X}\big) + \overset{\leftarrow}{\mathcal{L}^{X'}}\overset{\leftarrow}{\mathcal{L}^{X'}}x - \overset{\rightarrow}{\mathcal{L}^{X}}\overset{\rightarrow}{\mathcal{L}^{X}} x - \langle \frac{\hbar}{m}\nabla \log p_{X} - 2 \widetilde{u}, \nabla(\widetilde{v} - \widetilde{u})\rangle = 0.
\end{align}
By using Lemma \ref{lemma:magic} we thus derive:
\begin{align}
\overset{\leftarrow}{\mathcal{L}^{X}}& \big(2\widetilde{u} - \frac{\hbar}{m}\nabla \log p_{X}\big) + \overset{\rightarrow}{\mathcal{L}^{X}}\overset{\leftarrow}{\mathcal{L}^{X'}}x - \overset{\leftarrow}{\mathcal{L}^{X'}}\overset{\rightarrow}{\mathcal{L}^{X}}x  - \langle \frac{\hbar}{m}\nabla \log p_{X} - 2 \widetilde{u}, \nabla(\widetilde{v} - \widetilde{u})\rangle = 0.
\end{align}
Summing up both identities, therefore, yields:    
\begin{align}
\Big(\overset{\leftarrow}{\mathcal{L}^{X}}+\overset{\rightarrow}{\mathcal{L}^{X}}\Big)& \big(\widetilde{u} - \frac{\hbar}{2m}\nabla \log p_{X}\big) + \overset{\rightarrow}{\mathcal{L}^{X}}\overset{\leftarrow}{\mathcal{L}^{X'}}x - \overset{\leftarrow}{\mathcal{L}^{X'}}\overset{\rightarrow}{\mathcal{L}^{X}}x  + 2\langle \widetilde{u} - \frac{\hbar}{2m}\nabla \log p_{X} , \nabla \widetilde{v}\rangle = 0.
\end{align}
\end{proof}
\begin{theorem}\label{thmG8}
The following bound holds:
\begin{align} \sup_{0\le t\le T}\mathbb{E}^{X}\big\|\widetilde{u}(X(t), t) - \frac{\hbar}{2m}\nabla \log p_{X}(X(t), t)\big\|^2 \le e^{\big(\frac{1}{2}+4\|\nabla \widetilde{v}\|_{\infty}\big)T}\big(L_{3}(\widetilde{v}, \widetilde{u}) + L_{2}(\widetilde{v}, \widetilde{u})\big).
\end{align}
\end{theorem}
\begin{proof}
We consider process $Z(t) = \widetilde{u}u(X(t), t) - \frac{\hbar}{2m}\nabla \log p_{X}(X(t), t)$. From Nelson's lemma \ref{lemma:nelson}, we have the following identity:
\begin{align}
\mathbb{E}^{X}&\|\widetilde{u}(X(t), t) - \frac{\hbar}{2m}\nabla \log p_{X}(X(t), t)\|^2 - \mathbb{E}^{X}\|\widetilde{u}(X(0), 0) - \frac{\hbar}{2m}\nabla \log p_{X}(X(0), 0)\|^2 \\
=& \mathbb{E}^{X} \int_{0}^{t}\langle u(X(s), s) - \frac{\hbar}{2m}\nabla \log p_{X}(X(s), s), \\
&\qquad \big(\overset{\rightarrow}{\mathcal{L}^{X}}+\overset{\leftarrow}{\mathcal{L}^{X}}\big)\big(u(X(s), s) - \frac{\hbar}{2m}\nabla \log p_{X}(X(s), s)\big) \rangle \mathrm{d}s.
\end{align}    
Note that $u\equiv \frac{\hbar}{2m}\nabla \log p_{X}(X(t), t)$. Thus, $\mathbb{E}^{X}\|\widetilde{u}(X(0), 0) - \frac{\hbar}{2m}\nabla \log p_{X}(X(0), 0)\|^2 = L_{3}(\widetilde{v}, \widetilde{u})$. Using inequality $\langle a, b\rangle \le \frac{1}{2} \big(\|a\|^2 + \|b\|^2\big)$ we obtain:
\begin{align}
\mathbb{E}^{X}&\|u(X(t), t) - \frac{\hbar}{2m}\nabla \log p_{X}(X(t), t)\|^2 -L_{3}(\widetilde{v}, \widetilde{u})\\
\le& \int_{0}^{t}\Big(\frac{1}{2}\mathbb{E}^{X} \|u(X(s), s) - \frac{\hbar}{2m}\nabla \log p_{X}(X(s), s)\|^2 \\
&+
\frac{1}{2}\mathbb{E}^{X}\Big\| \big(\overset{\rightarrow}{\mathcal{L}^{X}}+\overset{\leftarrow}{\mathcal{L}^{X}}\big)\big(u(X(s), s) - \frac{\hbar}{2m}\nabla \log p_{X}(X(s), s)\big)\Big\|^2\Big)\mathrm{d}s
\end{align}
Using Lemma \ref{lemma:bound-density-diff}, we obtain:
\begin{align}
\mathbb{E}^{X}&\|u(X(t), t) - \frac{\hbar}{2m}\nabla \log p_{X}(X(t), t)\|^2 - L_{3}(\widetilde{v},\widetilde{u})\\
\le& \int_{0}^{t}\Big(\frac{1}{2}\mathbb{E}^{X} \|u(X(s), s) - \frac{\hbar}{2m}\nabla \log p_{X}(X(s), s)\|^2 \\
&+ \Big\|\overset{\rightarrow}{\mathcal{L}^{X}} \overset{\leftarrow}{\mathcal{L}^{X'}} x - \overset{\leftarrow}{\mathcal{L}^{X'}} \overset{\rightarrow}{\mathcal{L}^{X}} x\Big\|^2 + 4 \|\nabla \widetilde{v}\|_{\infty}^2\big\|\widetilde{u} - \frac{\hbar}{2m} \nabla \log p_{X}\big\|^2\Big)\mathrm{d}s
\end{align}
Observe that $\int_{0}^{t} \mathbb{E}^{X} \Big\|\overset{\rightarrow}{\mathcal{L}^{X}} \overset{\leftarrow}{\mathcal{L}^{X'}} x - \overset{\leftarrow}{\mathcal{L}^{X'}} \overset{\rightarrow}{\mathcal{L}^{X}} x\Big\|^2\mathrm{d}t \le L_{2}(\widetilde{v}, \widetilde{u})$, in fact, at $t=T$ it is equality as this is the definition of the loss $L_{2}$. Thus, we have:    
\begin{align}
\mathbb{E}^{X}&\|u(X(t), t) - \frac{\hbar}{2m}\nabla \log p_{X}(X(t), t)\|^2\\
&\le L_{3}(\widetilde{v}, \widetilde{u}) + L_{2}(\widetilde{v}, \widetilde{u}) + \int_{0}^{t}\big(\frac{1}{2}+4\|\nabla \widetilde{v}\|_{\infty}\big)\mathbb{E}^{X} \|u(X(s), s) - \frac{\hbar}{2m}\nabla \log p_{X}(X(s), s)\|^2\mathrm{d}s.
\end{align}
Using integral Gr\"onwall's inequality \citep{gronwall}  yields the bound: $\mathbb{E}^{X}\|u(X(t), t) - \frac{\hbar}{2m}\nabla \log p_{X}(X(t), t)\|^2 \le e^{\big(\frac{1}{2}+4\|\nabla \widetilde{v}\|_{\infty}\big)t}\big(L_{3}(\widetilde{v}, \widetilde{u}) + L_{2}(\widetilde{v}, \widetilde{u})\big).$
\end{proof}
\subsection{Nelsonian Processes}
Considering those two operators, we can rewrite the equations \labelcref{eq:nelson_eq} alternatively as:
\begin{align}\label{eq:OG}
\frac{1}{2}\Big(\overset{\rightarrow}{\mathcal{L}^{Y}} \overset{\leftarrow}{\mathcal{L}^{Y}} x + \overset{\leftarrow}{\mathcal{L}^{Y}} \overset{\rightarrow}{\mathcal{L}^{Y}} x\Big) &= -\frac{1}{m} \nabla V(x),\\
\frac{1}{2}\Big(\overset{\rightarrow}{\mathcal{L}^{Y}} \overset{\leftarrow}{\mathcal{L}^{Y}} x - \overset{\leftarrow}{\mathcal{L}^{Y}}\overset{\rightarrow}{\mathcal{L}^{Y}} x\Big) &= 0.
\end{align}
This leads us to the identity:
\begin{align}\label{eq:grad}
\overset{\rightarrow}{\mathcal{L}^{Y}} \overset{\leftarrow}{\mathcal{L}^{Y}} x &= -\frac{1}{m} \nabla V(x).
\end{align}
\begin{lemma}
We have the following bound:
\begin{align*}
\int_{0}^{t}\mathbb{E}^{X}\Big\|\overset{\rightarrow}{\mathcal{L}^{X'}} \overset{\leftarrow}{\mathcal{L}^{X}} X(t) + \frac{1}{m} \nabla V(X(t))\Big\|^2\mathrm{d}t \le 2 L_1(\widetilde{v}, \widetilde{u}) + 2 L_2 (\widetilde{v}, \widetilde{u}).
\end{align*}
\end{lemma}
\begin{proof}
Consider rewriting losses as:
\begin{align}
L_{1}(\widetilde{v}, \widetilde{u}) &= \int_{0}^{t}\mathbb{E}_{t\sim U[0, T]}\mathbb{E}^{X}\Big\|\frac{1}{2}\big(\overset{\rightarrow}{\mathcal{L}^{X}} \overset{\leftarrow}{\mathcal{L}^{X'}} X(t) + \overset{\rightarrow}{\mathcal{L}^{X}} \overset{\leftarrow}{\mathcal{L}^{X'}} X(t)\big) + \frac{1}{m} \nabla V(X(t))\Big\|^2\mathrm{d}t,\\
L_{2}(\widetilde{v}, \widetilde{u}) &=  \frac{1}{4}\int_{0}^{t}\mathbb{E}_{t\sim U[0, T]}\mathbb{E}^{X}\Big\|\overset{\rightarrow}{\mathcal{L}^{X}} \overset{\leftarrow}{\mathcal{L}^{X'}} X(t) - \overset{\rightarrow}{\mathcal{L}^{X'}} \overset{\leftarrow}{\mathcal{L}^{X}} X(t)\Big\|^2\mathrm{d}t.
\end{align}
Using the triangle inequality yields the statement.
\end{proof}
\begin{lemma}\label{lemma:G10}
We have the following bound: 
\begin{align*}
\int_{0}^{t}&\mathbb{E}^{X}\Big\|\overset{\leftarrow}{\mathcal{L}^{X}} \overset{\rightarrow}{\mathcal{L}^{X}} X(t) + \frac{1}{m} \nabla V(X(t))\Big\|^2\mathrm{d}t\\
&\le 2 T\big(\|\nabla \widetilde{u}\|_{\infty}+\|\nabla\widetilde{v}\|_{\infty}\big)^2 e^{\big(\frac{1}{2}+4\|\nabla \widetilde{v}\|_{\infty}\big)T}\big(L_{3}(\widetilde{v}, \widetilde{u}) + L_{2}(\widetilde{v}, \widetilde{u})\big) +  4 L_1(\widetilde{v}, \widetilde{u}) + 4 L_2 (\widetilde{v}, \widetilde{u}).
\end{align*}
\end{lemma}
\begin{proof}
From \eqref{eq74} we have:
\begin{align}
\overset{\leftarrow}{\mathcal{L}^{X}} \overset{\rightarrow}{\mathcal{L}^{X}} X(t) = \overset{\leftarrow}{\mathcal{L}^{X'}} \overset{\rightarrow}{\mathcal{L}^{X}} X(t) + \langle \frac{\hbar}{m}\nabla \log p_{X} - 2 \widetilde{u}, \nabla(\widetilde{v} + \widetilde{u})\rangle .
\end{align}
Noting that $\langle \frac{\hbar}{m}\nabla \log p_{X} - 2 \widetilde{u}, \nabla(\widetilde{v} + \widetilde{u})\rangle  \le \big(\|\nabla\widetilde{u}\|_{\infty}+\|\nabla\widetilde{v}\|_{\infty}\big)\Big\|\frac{\hbar}{m}\nabla \log p_{X} - 2 \widetilde{u}\Big\| $ and using triangle inequality we obtain the bound:
\begin{align}
\int_{0}^{t}&\mathbb{E}^{X}\Big\|\overset{\leftarrow}{\mathcal{L}^{X}} \overset{\rightarrow}{\mathcal{L}^{X}} X(t) + \frac{1}{m} \nabla V(X(t))\Big\|^2\mathrm{d}t\\
&\le 2\big(\|\widetilde{u}\|_{\infty}+\|\widetilde{v}\|_{\infty}\big)^2\int_{0}^{t}\mathbb{E}^{X}\Big\|u(X(t), t) - \frac{\hbar}{2m}\log p_{X}(X(t), t)\Big\|^2\mathrm{d}t +  4 L_1(\widetilde{v}, \widetilde{u}) + 4 L_2 (\widetilde{v}, \widetilde{u}).
\end{align}
Using Theorem \ref{thmG8} concludes the proof.
\end{proof}
\begin{lemma}
Denote $Z(t) = (X(t), Y(t))$ as compound process. For functions $h(x, y, t) = f(x, t) + g(y, t)$ we have the following identity:
\begin{align}
\overset{\rightarrow}{\mathcal{L}^{Z}} h = \overset{\rightarrow}{\mathcal{L}^{X}} f + \overset{\rightarrow}{\mathcal{L}^{Y}} g 
\end{align}
\end{lemma}
\begin{proof}
A generator is a linear operator by very definition. Thus, it remains to prove only
\begin{align}
\overset{\rightarrow}{\mathcal{L}^{Z}} f = \overset{\rightarrow}{\mathcal{L}^{X}} f 
\end{align}
Since the definition of $\overset{\rightarrow}{\mathcal{F}_t}$ already contains all past events for both processes $X(t), Y(t)$, we see that this is a tautology. 
\end{proof}
As a direct application of this Lemma, we obtain the following Corollary (by applying it twice):
\begin{corollary}\label{cor:G12}
We have the following identity:
\begin{align*}
\overset{\leftarrow}{\mathcal{L}^{Z}} \overset{\rightarrow}{\mathcal{L}^{Z}} \big(X(t)-Y(t)\big) = \overset{\leftarrow}{\mathcal{L}^{X}} \overset{\rightarrow}{\mathcal{L}^{X}} X(t) - \overset{\leftarrow}{\mathcal{L}^{Y}} \overset{\rightarrow}{\mathcal{L}^{Y}} Y(t).
\end{align*}
\end{corollary}
\begin{theorem}\label{thm:main_theorem} (Strong Convergence) Let the loss be defined as $\mathcal{L}(\widetilde{v}, \widetilde{u}) = \sum_{i=1}^{4} w_i L_{i}(\widetilde{v}, \widetilde{u})$ for some arbitrary constants $w_i > 0$. Then we have the following bound between processes $X$ and $Y$:
\begin{align}
\sup_{t\le T} \mathbb{E}\|X(t) - Y(t)\|^2 \le C_T \mathcal{L}(\widetilde{v}, \widetilde{u})
\end{align}
where $C_T = \max_{i} \frac{w_i'}{w_i}$, $w_1' = 4 e^{T(T+1)}$, $w_{2}' = e^{T(T+1)}\Big(2 T\big(\|\nabla \widetilde{u}\|_{\infty}+\|\nabla\widetilde{v}\|_{\infty}\big)^2 e^{\big(\frac{1}{2}+4\|\nabla \widetilde{v}\|_{\infty}\big)T} + 4\Big)$, $w_3' = 2T e^{T(T+1)} \Big(1 + \big(\|\nabla \widetilde{u}\|_{\infty}+\|\nabla\widetilde{v}\|_{\infty}\big)^2 e^{\big(\frac{1}{2}+4\|\nabla \widetilde{v}\|_{\infty}\big)T}\Big)$, $w_{4}' = 2 T e^{T(T+1)}$.
\end{theorem}
\begin{proof}
We are going to prove the bound:
\begin{align}
\sup_{t\le T} \mathbb{E}\|X(t) - Y(t)\|^2 \le \sum_{i=1}^4 w_i' L_i(\widetilde{v}, \widetilde{u}) 
\end{align}
for constants that we obtain from the Lemmas above. Then we will use the following trick to get the bound with arbitrary weights:
\begin{align}
\sum_{i=1}^4 w_i' L_i(\widetilde{v}, \widetilde{u}) \le \sum_{i=1}^4 \frac{w_i'}{w_i} w_i L_i(\widetilde{v}, \widetilde{u}) \le \big(\max_i \frac{w_i'}{w_i}\big) \sum_{i=1}^4  w_i L_i(\widetilde{v}, \widetilde{u}) = C_T \mathcal{L}(\widetilde{v}, \widetilde{u})
\end{align}
First, we apply Lemma \ref{G5} to $Z = X - Y$ by noting that $\big[X(t)-Y(t), X(t)-Y(t)\big]_{t} \equiv 0$ and $\|X(0)-Y(0)\|^2 = 0$ almost surely:
\begin{align}
\mathbb{E}^{Z}&\|X(t)-Y(t)\|^2 \\
=& \int_{0}^{t}\mathbb{E}^{Z}\Big(2\langle\overset{\leftarrow}{\mathcal{L}^{Z}} (X(0)-Y(0)), X(s)-Y(s)\rangle\\
&+ 2\int_{0}^{s} \langle\overset{\leftarrow}{\mathcal{L}^{Z}} \overset{\rightarrow}{\mathcal{L}^{Z}} (X(z)-Y(z)), X(s)-Y(s)\rangle\mathrm{d}z \Big)\mathrm{d}s \\
\le & \int_{0}^{t}\mathbb{E}^{Z}\Big(\big\|\overset{\leftarrow}{\mathcal{L}^{Z}} (X(0)-Y(0))\big\|^2 + \|X(s)-Y(s)\|^2 \\
&+ \int_{0}^{s} \Big(\big\|\overset{\leftarrow}{\mathcal{L}^{Z}} \overset{\rightarrow}{\mathcal{L}^{Z}} (X(z)-Y(z))\big\|^2 + \|X(s)-Y(s)\|^2 \mathrm{d}z \Big)\Big)\mathrm{d}s \\
\le & \int_{0}^{t}\mathbb{E}^{Z}\Big(\big\|\overset{\leftarrow}{\mathcal{L}^{Z}} (X(0)-Y(0))\big\|^2 + (1+T)\|X(s)-Y(s)\|^2 \\
&+ \int_{0}^{s} \big\|\overset{\leftarrow}{\mathcal{L}^{Z}} \overset{\rightarrow}{\mathcal{L}^{Z}} (X(z)-Y(z))\big\|^2\mathrm{d}z \Big)\mathrm{d}s.
\end{align}
Then, using Corollary \ref{cor:G12}, \eqref{eq:grad} and then Lemma \ref{lemma:G10} we obtain that 
\begin{align} 
\int_{0}^{s}& \big\|\overset{\leftarrow}{\mathcal{L}^{Z}} \overset{\rightarrow}{\mathcal{L}^{Z}} (X(z)-Y(z))\big\|^2\mathrm{d}z = \int_{0}^{s} \big\|\overset{\leftarrow}{\mathcal{L}^{Z}} \overset{\rightarrow}{\mathcal{L}^{Z}} X(z)+\frac{1}{m}\nabla V(X(z))\big\|^2\mathrm{d}z\\
&\le 2 T\big(\|\nabla \widetilde{u}\|_{\infty}+\|\nabla\widetilde{v}\|_{\infty}\big)^2 e^{\big(\frac{1}{2}+4\|\nabla \widetilde{v}\|_{\infty}\big)T}\big(L_{3}(\widetilde{v}, \widetilde{u}) + L_{2}(\widetilde{v}, \widetilde{u})\big) +  4 L_1(\widetilde{v}, \widetilde{u}) + 4 L_2 (\widetilde{v}, \widetilde{u}). 
\end{align}
To deal with the remaining term involving $X(0) - Y(0)$ we observe that:
\begin{align}
\int_{0}^{t}\mathbb{E}^{Z}\Big(\big\|\overset{\leftarrow}{\mathcal{L}^{Z}} (X(0)-Y(0))\big\|^2 \le 2T L_{3}(\widetilde{v}, \widetilde{u}) + 2 T L_{4}(\widetilde{v}, \widetilde{u}),
\end{align}
where we used triangle inequality. Combining obtained bounds yields:
\begin{align}
\mathbb{E}^{Z}&\|X(t)-Y(t)\|^2 \\
\le & \int_{0}^{t} (1+T)\|X(s)-Y(s)\|^2\mathrm{d}s \\
& + 2T L_{3}(\widetilde{v}, \widetilde{u}) + 2 T L_{4}(\widetilde{v}, \widetilde{u}) \\
& + 2 T\big(\|\nabla \widetilde{u}\|_{\infty}+\|\nabla\widetilde{v}\|_{\infty}\big)^2 e^{\big(\frac{1}{2}+4\|\nabla \widetilde{v}\|_{\infty}\big)T}\big(L_{3}(\widetilde{v}, \widetilde{u}) + L_{2}(\widetilde{v}, \widetilde{u})\big) \\
& +  4 L_1(\widetilde{v}, \widetilde{u}) + 4 L_2 (\widetilde{v}, \widetilde{u})\\
= & \int_{0}^{t} (1+T)\|X(s)-Y(s)\|^2\mathrm{d}s \\
& + 4 L_1(\widetilde{v}, \widetilde{u})+ \Big(2 T\big(\|\nabla \widetilde{u}\|_{\infty}+\|\nabla\widetilde{v}\|_{\infty}\big)^2 e^{\big(\frac{1}{2}+4\|\nabla \widetilde{v}\|_{\infty}\big)T} + 4\Big) L_{2}(\widetilde{v}, \widetilde{u})\\ 
& + 2T \Big(1 + \big(\|\nabla \widetilde{u}\|_{\infty}+\|\nabla\widetilde{v}\|_{\infty}\big)^2 e^{\big(\frac{1}{2}+4\|\nabla \widetilde{v}\|_{\infty}\big)T}\Big) L_{3}(\widetilde{v}, \widetilde{u}) + 2 T L_{4}(\widetilde{v}, \widetilde{u}).
\end{align}
Finally, using integral Gr\"onwall's inequality \citet{gronwall}, we have:
\begin{align}
\mathbb{E}^{Z}&\|X(t)-Y(t)\|^2 \\
&\le 4 e^{T(T+1)} L_1(\widetilde{v}, \widetilde{u})+ e^{T(T+1)}\Big(2 T\big(\|\nabla \widetilde{u}\|_{\infty}+\|\nabla\widetilde{v}\|_{\infty}\big)^2 e^{\big(\frac{1}{2}+4\|\nabla \widetilde{v}\|_{\infty}\big)T} + 4\Big) L_{2}(\widetilde{v}, \widetilde{u})\\ 
& + 2T e^{T(T+1)} \Big(1 + \big(\|\nabla \widetilde{u}\|_{\infty}+\|\nabla\widetilde{v}\|_{\infty}\big)^2 e^{\big(\frac{1}{2}+4\|\nabla \widetilde{v}\|_{\infty}\big)T}\Big) L_{3}(\widetilde{v}, \widetilde{u}) + 2 T e^{T(T+1)} L_{4}(\widetilde{v}, \widetilde{u}) 
\end{align}
\end{proof}

\section{Applications}\label{sec:applications_full}

\subsection{Bounded Domain \texorpdfstring{$\mathcal{M}$}{M}}
Our approach assumes that the manifold $\mathcal{M}$ is flat or curved. For bounded domains $\mathcal{M}$, e.g., like it is assumed in PINN or any other grid-based methods, our approach can be applied if we embed $\mathcal{M} \subset \mathbb{R}^d$ and define a new family of smooth non-singular potentials $V_{\alpha}$ on entire $\mathbb{R}^d$ such that $V_{\alpha} \rightarrow V$ when restricted to $\mathcal{M}$ and $V_{\alpha}\rightarrow +\infty$ on $\partial(\mathcal{M}, \mathbb{R}^d)$ (boundary of the manifold in embedded space) as $\alpha\rightarrow 0_+$. 

\subsection{Singular Initial Conditions}\label{sec:singular_potential}
It is possible to apply \cref{alg:train} to $\psi_{0} = \delta_{x_{0}}e^{iS_{0} (x)}$ for some $x_{0}\in\mathcal{M}$. We need to augment the initial conditions with a parameter $\alpha > 0$ as $\psi_{0} = \sqrt{\frac{1}{\sqrt{2\pi \alpha^2}}e^{-\frac{(x-x_{0})^2}{2\alpha^2}}}$ for small enough $\alpha > 0$. In that case, $u_{0}(x) = -\frac{\hbar}{2m}\frac{(x-x_{0})}{\alpha}$. We must be careful with choosing $\alpha$ to avoid numerical instability. It makes sense to try $\alpha \propto \frac{\hbar^2}{m^2}$ as $\frac{X(0) - x_{0}}{\alpha} = \mathcal{O}(\sqrt{\alpha})$. We evaluated such a setup in \cref{app:singular_ic}.

\subsection{Singular Potential}
We must augment the potential to apply our method for simulations of the atomic nucleus with Bohr-Oppenheimer approximation \citep{woolley1977molecular}. A potential arising in this case has components of form $\frac{a_{ij}}{\|x_{i}-x_{j}\|}$. Basically, it has singularities when $x_{i} = x_{j}$. In case when $x_{j}$ is fixed, our manifold is $\mathcal{M}\backslash \{x_{j}\}$, which has a non-trivial cohomology group.

When such potential arises we suggest to augment the potential $V_{\alpha}$ (e.g., replace all $\frac{a_{ij}}{\|x_{i}-x_{j}\|}$ with $\frac{a_{ij}}{\sqrt{\|x_{i}-x_{j}\|^2 + \alpha}}$) so that $V_{\alpha}$ is smooth and non-singular everywhere on $\mathcal{M}$. In that case we have that $V_{\alpha}\rightarrow V$ as $\alpha\rightarrow 0$. With the augmented potential $V_{\alpha}$, we can apply stochastic mechanics to obtain an equivalent to quantum mechanics theory. Of course, augmentation will produce bias, but it will be asymptotically negligent as $\alpha \rightarrow 0$.

\subsection{Measurement} \label{sec:measurement}
Even though we have entire trajectories and know positions for each moment, we should carefully interpret them. This is because they are not the result of the measurement process. Instead, they represent hidden variables (and $u,v$ represent global hidden variables -- what saves us from the Bells inequalities as stochastic mechanics is non-local \citep{NelsonOG}).

For a fixed $t\in [0, T]$, the distribution of $X(t)$ coincides with the distribution $\mathbf{X}(t)$ for $\mathbf{X}$ being position operator in quantum mechanics. Unfortunately, a compound distribution $(X(t), X(t'))$ for $t\ne t'$ may not correspond to the compound distribution of $(\mathbf{X}(t), \mathbf{X}(t'))$; for details see \citet{nelson2005mystery}. This is because each $\mathbf{X}(t)$ is a result of the \textit{measurement process}, which causes the wave function to collapse \citep{derakhshani2022multitime}. 

Trajectories $X_i$ are as if we could measure $\mathbf{X}(t)$ without causing the collapse of the wave function. To use this approach for predicting some experimental results involving multiple measurements, we need to re-run our method after each measurement process with the measured state as the new initial condition. This issue is not novel for stochastic mechanics. There is the same problem in classical quantum mechanics.

This ``contradiction'' is resolved once we realize that $\mathbf{X}(t)$ involves measurement, and thus, if we want to calculate correlations of $(\mathbf{X}(t), \mathbf{X}(t'))$ for $t < t'$ we need to do the following:
\begin{itemize}
    \item Run Algorithm \ref{alg:train} with $\psi_{0}, V(x,t)$ and $T = t$ to get $\widetilde{u}, \widetilde{v}$.
    \item Run Algorithm \ref{alg:generate} with $\widetilde{u},\widetilde{v}$, $\psi_{0}$ to get $\{X_{N j}\}_{j=1}^{B}$ -- $B$ last steps from trajectories $X_i$ of length $N$. 
    \item For each $X_{N j}$
    in the batch we need to run Algorithm \ref{alg:train} with $\psi_{0} = \delta_{X_{N j}}, V'(x, t') = V(x, t'+t)$ (assuming that $u_{0} = 0, v_{0} = 0$) and $T = t' - t$ to get $\widetilde{u}_{j}, \widetilde{v}_{j}$.
    \item For each $X_{N j}$ run Algorithm \ref{alg:generate} with batch size $B = 1$, $\psi_{0} = \delta_{X_{N j}}$, $\widetilde{u}_j, \widetilde{v}_{j}$ to get $X_{N j}'$.
    \item Output pairs $\big\{(X_{N, j}, X_{N, j}')\big\}_{j=1}^{B}$.
\end{itemize}
Then the distribution of $(X_{N, j}, X_{N, j}')$ will correspond to the distribution of $(\mathbf{X}(t), \mathbf{X}(t'))$. This is well described and proven in \citet{derakhshani2022multitime}. Therefore, it is possible to simulate the right correlations in time using our approach, though, it may require learning $2(B+1)$ models. The promising direction of future research is to consider $X_{0}$ as a feature for the third step here and, thus, learn only $2+2$ models.

\subsection{Observables}
To estimate any scalar observable of form $\mathbf{Y}(t) = y(\mathbf{X}(t))$ in classic quantum mechanics one needs to calculate:
$$\langle \mathbf{Y}\rangle_{t} = \int_{\mathcal{M}} \overline{\psi(x, t)} y(x) \psi(x, t)\mathrm{d}x.$$
In our setup, we can calculate this using the samples $X_{\big[\frac{Nt}{T}\big]}\approx X(t) \sim \big|\psi(\cdot, t)\big|^2$:
$$\langle \mathbf{Y}\rangle_{t} \approx \frac{1}{B}\sum_{j}^{B} y(X_{\big[\frac{Nt}{T}\big] j}),$$
where $B \ge 1$ is the batch size, $N$ is the time discretization size. The estimation error has magnitude $\mathcal{O}(\frac{1}{\sqrt{B}} + \epsilon + \varepsilon)$, where $\epsilon = \frac{T}{N}$ and $\varepsilon$ is the $L_{2}$ error of recovering true $u,v$. In our paper, we have not bounded $\varepsilon$ but provide estimates for it in our experiments against the finite difference solution.\footnote{If we are able to reach $\mathcal{L}(\theta) = 0$ then essentially $\varepsilon = 0$. We leave bounding $\varepsilon$ by $\mathcal{L}(\theta_{\tau})$ for future work.}

\subsection{Wave Function}
Recovering the wave function from $u,v$ is possible using a relatively slow procedure. Our experiments do not cover this because our approach's main idea is to avoid calculating wave function. But for the record, it is possible. Assume we solved equations for $u, v$. We can get the phase and density by integrating \cref{eq:phase_density_equations}:
\begin{align}\label{eq:phase_density_equations1}
        S(x, t) &= S(x, 0) + \int_{0}^{t}\Big(\frac{1}{2m}\langle \nabla, u(x, t)\rangle + \frac{1}{2\hbar}\big\|u(x,t)\big\|^2 - \frac{1}{2\hbar}\big\|v(x,t)\big\|^2 - \frac{1}{\hbar}V(x, t)\Big)\mathrm{d}t, \\
        \rho(x, t) &= \rho_{0}(x) \exp\Big(\int_{0}^{t}\big(-\langle \nabla, v(x, t)\rangle - \frac{2 m}{\hbar}\langle u(x, t), v(x, t)\rangle\big)\Big)\mathrm{d}t
\end{align}
This allows us to define $\psi = \sqrt{\rho (x, t)}e^{i S (x, t)}$, which satisfies the Schr\"odinger equation \labelcref{eq:Schrodinger}. 
Suppose we want to estimate it over a grid with $N$ time intervals and $\big[\sqrt{N}\big]$ intervals for each coordinate (a typical recommendation for \cref{eq:Schrodinger} is to have a grid satisfying $\mathrm{d}x^2\approx \mathrm{d}t$). It leads to a sample complexity of $\mathcal{O}(N^{\frac{d}{2}+1})$, which is as slow as other grid-based methods for quantum mechanics. The error in that case will  also be $\mathcal{O}(\sqrt{\epsilon} + \varepsilon)$ \citep{smith1985numerical}.

\section{On criticism of Stochastic Mechanics}
Three major concerns arise regarding stochastic mechanics developed by \citet{NelsonOG, guerra1995introduction}:
\begin{itemize}
    \item The proof of the equivalence of stochastic mechanics to classic quantum mechanics relies on an implicit assumption of the phase $S(x, t)$ being single-valued \citep{inequavalency}.
    \item If there is an underlying stochastic process of quantum mechanics, it should be non-Markovian \citep{nelson2005mystery}.
    \item For a quantum observable, e.g., a position operator $\mathbf{X}(t)$, a compound distribution of positions at two different timestamps $t, t'$ does not match distribution of $(\mathbf{X}(t), \mathbf{X}(t'))$ \citep{nelson2005mystery}. 
\end{itemize}
Appendix \ref{sec:measurement} discusses why a mismatch of the distributions is not a problem and how we can adopt stochastic mechanics with our approach to get correct compound distributions by incorporating the measurement process into the stochastic mechanical picture.

\subsection{On ``Inequivalence'' to Schr\"odinger equation}
This problem is explored in the paper by \citet{inequavalency}. Firstly, the authors argue that proofs of the equivalency in \citet{NelsonOG, guerra1995introduction} are based on the assumption that the wave function phase $S$ is single-valued. In the general case of a multi-valued phase, the wave functions are identified with sections of complex line bundles over $\mathcal{M}$. In the case of a trivial line bundle, the space of sections can be formed from single-valued functions, see \citet{about_line_bundles}. The equivalence class of line bundles over a manifold $\mathcal{M}$ is called Picard group, and for smooth manifolds, $\mathcal{M}$ is isomorphic to $H^2(\mathcal{M}, \mathbb{Z})$, so-called second cohomology group over $\mathbb{Z}$, see \citet{about_single_valuedness} for details. Elements in this group give rise to non-equivalent quantizations with irremovable gauge symmetry phase factor. 

Therefore, \textit{in this paper, we assume that} $H^2(\mathcal{M}, \mathbb{Z}) = 0$, which allows us to eliminate all criticism about non-equivalence. Under this assumption, stochastic mechanics is \textit{equivalent} indeed. This condition holds when $\mathcal{M} = \mathbb{R}^d$. Though, if a potential $V$ has singularities, e.g., $\frac{a}{\|x-x_{*}\|}$, then we should exclude $x_{*}$ from $\mathbb{R}^d$ which leads to $\mathcal{M} = \mathbb{R}^d \backslash \{x_{*}\}$ and this manifold satisfies $H^2(\mathcal{M}, \mathbb{Z}) \cong \mathbb{Z}$ \citep{may1999concise}, which essentially leads to "counterexample" provided in \citet{inequavalency}. We suggest a solution to this issue in Appendix \ref{sec:singular_potential}.

\subsection{On ``Superluminal'' Propagation of Signals}

We want to clarify why this work should not be judged from perspectives of \textit{physical realism}, \textit{correspondence to reality} and \textit{interpretations} of quantum mechanics. 
This tool gives the exact predictions as classical quantum mechanics at a moment of measurement. Thus, we do not care about a superluminal change in the drifts of entangled particles and other problems of the Markovian version of stochastic mechanics.

\subsection{Non-Markovianity}
Nelson believes that an underlying stochastic process of reality should be non-Markovian to avoid issues with the Markovian processes like superluminal propagation of signals \citep{nelson2005mystery}. Even if such a process were proposed in the future, it would not affect our approach. In stochastic calculus, there is a beautiful theorem from \citet{gyongy1986mimicking}:

\begin{theorem}
Assume $X(t), F(t), G(t)$ are adapted to Wiener process $W(t)$ and satisfy:
$$\mathrm{d}X(t) = F(t)\mathrm{d}t + G(t)\mathrm{d}\overset{\rightarrow}{W}.$$
Then there exist a Markovian process $Y(t)$ satisfying
$$\mathrm{d}Y(t) = f(Y(t), t)\mathrm{d}t + g(Y(t), t)\mathrm{d}\overset{\rightarrow}{W}$$
where $f(x, t) = \mathbb{E}(F(t)\|X(t)=x)$,$g(x, t) = \sqrt{\mathbb{E}(G(t)G(t)^{T}\|X(t)=x)}$ and such that $\forall t$ holds $\mathrm{Law}(X(t)) = \mathrm{Law}(Y(t))$.
\end{theorem}
This theorem tells us that we already know how to build a process $Y(t)$ without knowing $X(t)$; it is stochastic mechanics by Nelson \citep{guerra1995introduction, NelsonOG} that we know. From a numerical perspective, we better stick with $Y(t)$ as it is easier to simulate, and as we explained, we do not care about correspondence to reality as long as it gives the same final results.

 \subsection{Ground State}
 Unfortunately, our approach is unsuited for the ground state estimation or any other stationary state. FermiNet \citep{pfau2020ferminet} does a fantastic job already. The main focus of our work is time evolution. It is possible to estimate some observable $\mathbf{Y}$ for the ground state if its energy level is unique and significantly lower than others. In that case, the following value approximately equals the group state observable for $T\gg 1$:
 $$\langle\mathbf{Y}\rangle_{ground} \approx \frac{1}{T}\int_{0}^{T} \langle\mathbf{Y}\rangle_{t}\mathrm{d}t \approx \frac{1}{NB}\sum_{i=1}^{N}\sum_{j=1}^{B} y(X_{ij})$$
 This works only if the ground state is unique, and the initial conditions satisfy $\int_{\mathcal{M}} \overline{\psi_{0}}\psi_{ground}\mathrm{d}x \ne 0$, and its energy is well separated from other energy levels. In that scenario, oscillations will cancel each other out.

\section{Future Work}\label{app:future_work}
This section discusses possible directions for future research. Our method is a promising direction for fast quantum mechanics simulations, but we consider the most straightforward setup in our work. 
Possible future advances include:
\begin{itemize}
    \item In our work, we consider the simplest integrator of SDE (Euler-Maruyama), which may require setting $N\gg 1$ to achieve the desired accuracy. However, a higher-order integrator \citep{smith1985numerical} or an adaptive integrator \citep{ilie2015adaptive} should achieve the desired accuracy with much lower $ N $.
    \item Exploring the applicability of our method to fermionic systems is a promising avenue for future investigation. Successful extensions in this direction would not only broaden the scope of our approach but also have implications for designing novel materials, optimizing catalytic processes, and advancing quantum computing technologies. 
    
    \item It should be possible to extend our approach to a wide variety of other quantum mechanical equations, including Dirac and Klein-Gordon equations used to account for special relativity \citep{serva1988relativistic, blanchard2005stochastic}, a non-linear Schr\"odinger \cref{eq:Schrodinger} used in condensed matter physics \citep{serkin2000novel} by using McKean-Vlasov SDEs and the mean-field limit  \citep{10.1214/15-AOP1076, dos2022simulation}, and the Shr\"odinger equation with a spin component \citep{dankel1970mechanics,de1991imaginary}.
    \item We consider a rather simple, fully connected architecture of neural networks with $\tanh$ activation and three layers. It might be more beneficial to consider specialized architectures for quantum mechanical simulations, e.g., \citet{pfau2020ferminet}. Permutation invariance can be ensured using a self-attention mechanism \citep{vaswani2017attention}, which could potentially offer significant enhancements to model performance. Additionally, incorporating gradient flow techniques as suggested by \citet{neklyudov2024wasserstein} can help to accelerate our algorithm. 
    \item Many practical tasks require knowledge of the error magnitude. Thus, providing explicit bounds on $\varepsilon$ in terms of $\mathcal{L}(\theta_{M})$ is critical.
\end{itemize}

\newpage

\newpage

\end{document}